\documentclass{article}

\PassOptionsToPackage{numbers}{natbib}
\usepackage[final,main]{neurips_2026}
\usepackage[utf8]{inputenc}
\usepackage[T1]{fontenc}
\usepackage{hyperref}
\usepackage{url}
\usepackage{booktabs}
\usepackage{amsfonts}
\usepackage{amsmath}
\usepackage{amsthm}
\usepackage{algorithm}
\usepackage{algpseudocode}
\newtheorem{proposition}{Proposition}
\newtheorem{corollary}{Corollary}
\newtheorem{remark}{Remark}
\usepackage{enumitem}
\usepackage{nicefrac}
\usepackage{microtype}
\usepackage{xcolor}
\usepackage{colortbl}
\usepackage{wrapfig}
\usepackage{listings}
\definecolor{codebg}{HTML}{F7F7F8}
\definecolor{codegreen}{HTML}{1B5E20}
\definecolor{codepurple}{HTML}{4A148C}
\definecolor{codegray}{HTML}{757575}
\lstdefinestyle{pythonstyle}{
  language=Python,
  backgroundcolor=\color{codebg},
  basicstyle=\ttfamily\scriptsize,
  keywordstyle=\color{codepurple}\bfseries,
  commentstyle=\color{codegray}\itshape,
  stringstyle=\color{codegreen},
  numbers=left,
  numberstyle=\tiny\color{codegray},
  numbersep=6pt,
  frame=single,
  rulecolor=\color{codegray!30},
  breaklines=true,
  showstringspaces=false,
  tabsize=4,
  xleftmargin=1.5em,
  framexleftmargin=1em,
}
\usepackage{graphicx}
\usepackage{pgfplots}
\usetikzlibrary{calc}
\usetikzlibrary{arrows.meta,positioning,fit,shapes.geometric,backgrounds}
\usepgfplotslibrary{groupplots}
\pgfplotsset{compat=1.18}
\usepackage{caption}
\usepackage{float}
\usepackage{placeins}
\usepackage[most]{tcolorbox}
\usepackage{pifont}
\usepackage{fontawesome5}
\definecolor{mydarkblue}{HTML}{1A4F8B}
\newsavebox{\tablefitbox}
\newcommand{\fitwidth}[1]{%
  \sbox{\tablefitbox}{#1}%
  \ifdim\wd\tablefitbox>\linewidth
    \resizebox{\linewidth}{!}{\usebox{\tablefitbox}}%
  \else
    \makebox[\linewidth][c]{\usebox{\tablefitbox}}%
  \fi
}
\hypersetup{
  colorlinks=true,
  linkcolor=mydarkblue,
  citecolor=mydarkblue,
  urlcolor=mydarkblue,
  filecolor=mydarkblue
}
\newtcolorbox{problembox}{
  enhanced,
  colback=mydarkblue!4,
  colframe=mydarkblue!22,
  boxrule=0.4pt,
  arc=1.2mm,
  left=1.2mm,
  right=1.2mm,
  top=0.9mm,
  bottom=0.9mm,
  borderline west={1.4pt}{0pt}{mydarkblue!55}
}

\newtcolorbox{contribbox}{
  enhanced,
  breakable,
  colback=mydarkblue!2,
  colframe=mydarkblue!24,
  boxrule=0.4pt,
  arc=1.1mm,
  left=1.5mm,
  right=1.5mm,
  top=1.0mm,
  bottom=0.9mm,
  borderline west={1.6pt}{0pt}{mydarkblue!60}
}

\newcommand{\method}{\textsc{RoboAlign-R1}}
\title{\method{}: Distilled Multimodal Reward Alignment for Robot Video World Models}

\author{
  {\small \textbf{Hao Wu}$^{1,*}$\hspace{2.5pt}\textbf{Yuqi Li}$^{2,*}$\hspace{2.5pt}\textbf{Yuan Gao}$^{1,*}$\hspace{2.5pt}\textbf{Fan Xu}$^{5,\dagger}$\hspace{2.5pt}\textbf{Fan Zhang}$^{3}$\hspace{2.5pt}\textbf{Kun Wang}$^{4,\dagger}$\hspace{2.5pt}\textbf{Penghao Zhao}$^{5}$} \\[3pt]
  {\small \textbf{Qiufeng Wang}$^{5}$\hspace{2.5pt}\textbf{Yizhou Zhao}$^{6}$\hspace{2.5pt}\textbf{Weiyan Wang}$^{5}$\hspace{2.5pt}\textbf{Yingli Tian}$^{2}$\hspace{2.5pt}\textbf{Xian Wu}$^{5,\dagger}$\hspace{2.5pt}\textbf{Xiaomeng Huang}$^{1}$} \\[6pt]
  {\normalfont\small $^{1}$THU \quad $^{2}$CUNY \quad $^{3}$CUHK
  \quad $^{4}$NTU \quad $^{5}$Tencent \quad $^{6}$CMU} \\[3pt]
  {\normalfont\small $^{*}$Equal contribution. \quad $^{\dagger}$Corresponding authors.} \\[3pt]
  {\normalfont\small \href{mailto:wu-h25@mails.tsinghua.edu.cn}{\faIcon[regular]{envelope}\hspace{0.45em}\texttt{wu-h25@mails.tsinghua.edu.cn}}}
}

\begin{document}

\maketitle

\begin{abstract}
Existing robot video world models are typically trained with low-level objectives such as reconstruction and perceptual similarity, which are poorly aligned with the capabilities that matter most for robot decision making, including instruction following, manipulation success, and physical plausibility. They also suffer from error accumulation in long-horizon autoregressive prediction. We present \method{}, a framework that combines reward-aligned post-training with stabilized long-horizon inference for robot video world models. We construct \textsc{RobotWorldBench}, a benchmark of 10{,}000 annotated video--instruction pairs collected from four robot data sources, and train a multimodal teacher judge, \textsc{RoboAlign-Judge}, to provide fine-grained six-dimensional evaluation of generated videos. We then distill the teacher into a lightweight student reward model for efficient reinforcement-learning-based post-training. To reduce long-horizon rollout drift, we further introduce \textit{Sliding Window Re-encoding} (SWR), a training-free inference strategy that periodically refreshes the generation context. Under our in-domain evaluation protocol, \method{} improves the aggregate six-dimension score by $10.1\%$ over the strongest baseline, including gains of $7.5\%$ on Manipulation Accuracy and $4.6\%$ on Instruction Following; these ranking improvements are further supported by an external VLM-based cross-check and a blinded human study. Meanwhile, SWR improves long-horizon prediction quality with only about $1\%$ additional latency, yielding a $2.8\%$ gain in SSIM and a $9.8\%$ reduction in LPIPS. Together, these results show that reward-aligned post-training and stabilized long-horizon decoding improve task consistency, physical realism, and long-horizon prediction quality in robot video world models. \href{https://alexander-wu.github.io/roboalign-r1-pages/}{Codes are available.}
\end{abstract}

\vspace{-0.6em}

\section{Introduction}
Robot video world models are becoming an important component of embodied intelligence because they predict how scenes may evolve under candidate actions before executing them in the real world \cite{ha2018world,ni2025maskgwm,brooks2024video}. As learnable internal simulators, they support planning, policy evaluation, and long-horizon reasoning \cite{tan2026towards,hansen2023td,li2024evaluating}. Unlike generic video generation, robot world models must capture action-conditioned dynamics, contacts, and scene evolution so that predicted futures remain consistent with manipulation outcomes and are useful for control \cite{li2025worldmodelbench,wu2025rlvr}. This requirement makes it challenging to achieve visual fidelity, task consistency, and physical plausibility simultaneously~\cite{barcellona2024dream,pang2025learning}.

However, most existing robot video world models~\cite{wu2025rlvr,huang2025vid2world,mao2025robot,ding2025understanding} are trained with maximum-likelihood, reconstruction, or perceptual objectives. These losses optimize statistical similarity rather than decision-making utility, leading to a fundamental mismatch: a prediction can be visually accurate yet task-incorrect or physically implausible. This gap becomes more pronounced in long-horizon generation, where small errors accumulate and task-level correctness dominates frame-level similarity.

Recent efforts to address this issue via reinforcement learning still rely largely on low-level rewards such as MSE, LPIPS, or SSIM. While efficient, these metrics fail to capture high-level properties critical for robot reasoning, including instruction following, manipulation success, and physically consistent interactions \cite{wu2025rlvr,xiong2026phycritic}. As a result, current training objectives remain poorly aligned with the capabilities that matter most for downstream control.

\begin{figure}[t]
  \centering
  \includegraphics[width=1\linewidth]{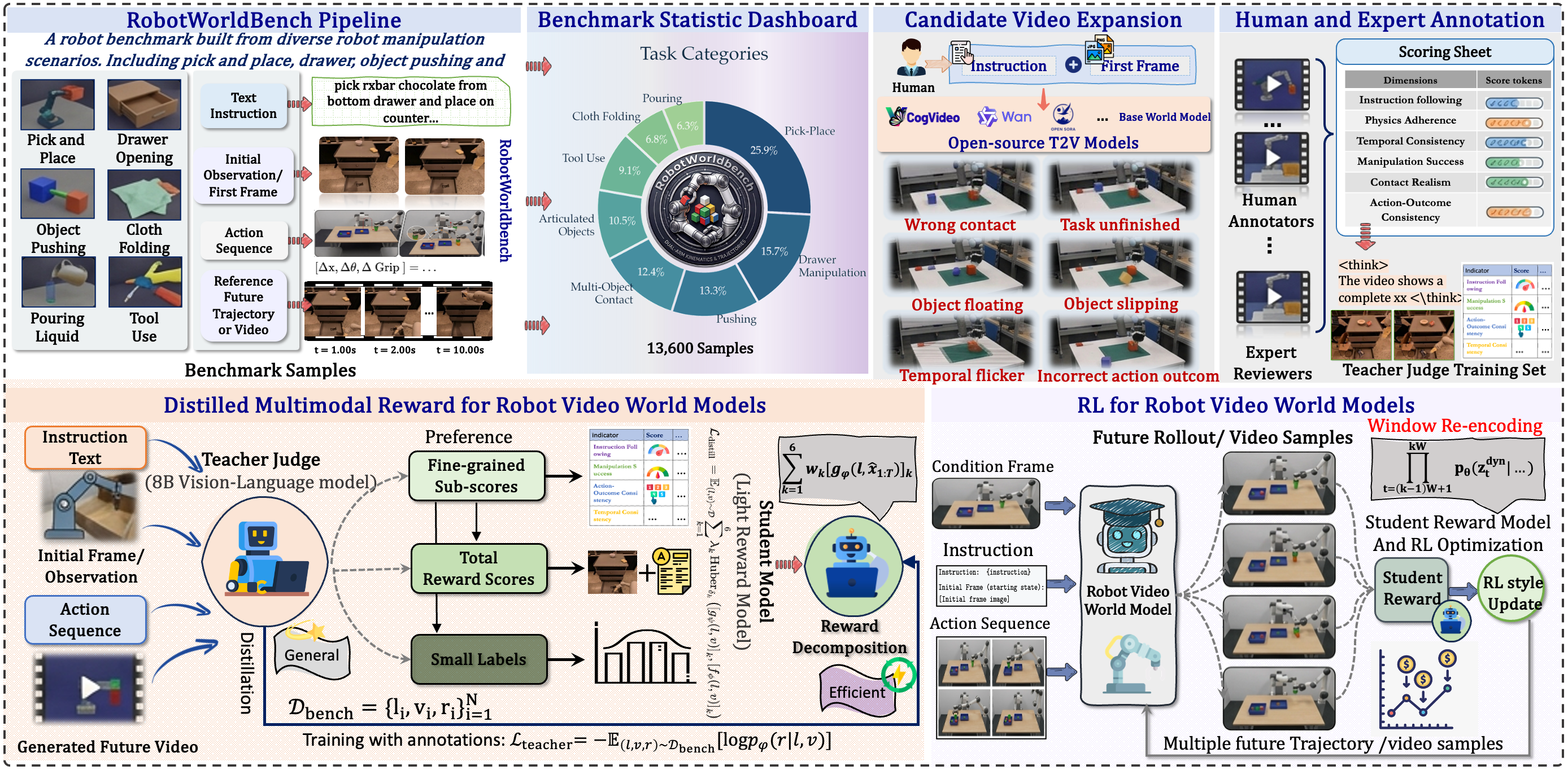}
  \caption{\textbf{Overview of \method{}.} \textsc{RobotWorldBench} provides robot-centric benchmark statistics and fine-grained annotations for training a multimodal teacher judge. The teacher is distilled into a lightweight student reward model for efficient reinforcement-learning-based post-training of robot video world models. In parallel, a sliding-window re-encoding strategy stabilizes long-horizon autoregressive rollouts by periodically refreshing the visual context during inference.}
  \vspace{-10pt}
  \label{fig:overview}
\end{figure}

These limitations highlight two key challenges. 
1) \textbf{\textit{reward misalignment}}: low-level metrics are scalable but weak, while stronger multimodal evaluators are computationally prohibitive for online optimization. 2) \textbf{\textit{long-horizon drift}}: autoregressive generation accumulates errors over time, degrading physical realism and task consistency \cite{xiao2023efficient}. To address these challenges, we propose \method{}, a unified framework for improving both training-time alignment and inference-time stability in robot video world models (Figure~\ref{fig:overview}). Our approach introduces a distilled multimodal reward that aligns model optimization with task-level objectives, and a sliding window re-encoding (SWR) strategy that mitigates long-horizon error accumulation during inference.

Specifically, we construct \textsc{RobotWorldBench}, a large-scale benchmark of $10K$ annotated video–instruction pairs curated from four robot datasets. Using this data, we train a multimodal teacher model (\textsc{RoboAlign-Judge}) that provides fine-grained, multi-dimensional evaluation of generated videos. We then distill this teacher into a lightweight student reward model, enabling efficient reinforcement-learning-based post-training. Complementing this, SWR periodically refreshes the autoregressive context without retraining, improving long-horizon prediction quality with minimal overhead. Our contributions are threefold:

\textbf{\ding{182} \textit{Reward-aligned training framework}.} We introduce \method{}, which replaces weak low-level online rewards with a distilled multimodal reward aligned with task success. This improves aggregate in-domain evaluation performance  by \textbf{$+10.1$\%} over the strongest baseline while reducing online reward cost by over $10{\times}$,  with gains validated by both VLM-based evaluation and blinded human studies. 

\textbf{\ding{183} \textit{Benchmark and scalable reward distillation}.} We build \textsc{RobotWorldBench}, a dataset of $10K$ annotated video–instruction pairs from four robot datasets, and a teacher–student distillation pipeline that compresses an $8B$ multimodal judge into a $98M$ reward model running at ${\sim}$$50$ videos/sec. 

\textbf{\ding{184} \textit{Stable long-horizon inference}.} We propose sliding window re-encoding, a training-free decoding strategy that mitigates autoregressive drift by refreshing context during rollout. SWR improves long-horizon quality with minimal overhead, yielding \textbf{$+2.8$\%} SSIM, \textbf{$+0.62$}\,dB PSNR, and \textbf{$-9.8$\%} LPIPS at about \textbf{$1$\%} additional latency, with a 12.2\% reduction in ROI-LPIPS.

\section{Related Work}
\vspace{-5pt}
The relevant technical background is as follows; the full stack context can be found in Appendix \ref{app:related_work}.
\vspace{-5pt}
\paragraph{\textit{Robot Video World Models and RL Post-Training}.}
World models act as internal simulators for planning and long-horizon decision making by predicting future observations under candidate actions~\cite{ha2018world,hansen2023td}. Recent advances in video generation and embodied learning have led to growing interest in robot video world models for modeling manipulation dynamics, interaction outcomes, and future observations~\cite{brooks2024video,bruce2024genie,yang2023learning,zhou2024robodreamer,huang2025vid2world,mao2025robot,zhang2025step,ding2025understanding}. Most existing methods, however, still focus on generation or representation quality and are trained with maximum-likelihood, reconstruction, or perceptual objectives~\cite{zhang2025step,ding2025understanding,mathieu2015deep,zhang2018unreasonable,wang2004image}. Recent work has explored RL-based post-training for world models~\cite{wu2025rlvr,prabhudesai2024video,sfp2026}, but current rewards remain largely low-level or verifiable, such as reconstruction and perceptual similarity, which are weak proxies for instruction following, physical plausibility, and long-horizon outcome correctness~\cite{wu2025rlvr,wu2024beamvq,wang2025beamvq,sfp2026}. Similar issues appear in physical spatiotemporal forecasting, where reconstruction-based training often produces oversmoothed predictions and misses rare but decision-relevant events~\cite{ravuri2021skilful,finn2016unsupervised,wang2025beamvq}. In addition, autoregressive token-based world models suffer from long-horizon error accumulation at inference time~\cite{yan2021videogpt,xiao2023efficient,wu2024ivideogpt,wang2022predrnn}. Our work studies reward-aligned post-training together with a simple strategy for stabilizing long-horizon rollouts.

\paragraph{\textit{Multimodal Judges and Rewards}.}
Recent work has explored multimodal evaluators and reward models for video understanding and generation~\cite{huang2024vbench,he2024videoscore,xiong2025llava}. Compared with low-level visual metrics, these judges provide supervision that better reflects instruction following, physical plausibility, and higher-level video consistency~\cite{li2025worldmodelbench,xiong2026phycritic,huang2024vbench,he2024videoscore}. However, their cost and latency make them difficult to use directly as online rewards in reinforcement learning~\cite{li2025worldmodelbench,he2024videoscore,xiong2026phycritic}. \method{} follows this direction, but targets robot video world models and distills multimodal judgment into a lightweight reward model for online reinforcement learning.

\vspace{-5pt}
\section{Method}
\vspace{-5pt}
\method{} addresses two key challenges in robot video world models: \emph{reward misalignment} during training and \emph{error accumulation} during inference. \method{} unifies reward-aligned post-training and stabilized long-horizon decoding within a single framework.

\vspace{-5pt}
\subsection{Token-Based Robot Video World Model}
\vspace{-5pt}

Our backbone (Figure \ref{fig:model}) follows the \emph{tokenize--predict--decode} paradigm. A dual-branch visual tokenizer $\mathcal{T}_{\mathrm{vis}}=(E_c,E_d)$ maps a short clip to discrete tokens: $E_c$ encodes a conditioning frame into $N_c$ \emph{context tokens} $z^{\mathrm{ctx}}$, while $E_d$, cross-attended to $E_c$'s feature pyramid, encodes each observation frame into $N_d$ \emph{dynamics tokens} $z^{\mathrm{dyn}}_t$ carrying only residual change. Robot actions $a_t\!\in\!\mathbb{R}^{d_a}$ are discretized into action-token blocks $\hat{a}_t$ and interleaved with visual tokens, forming a unified sequence $s\!=\![z^{\mathrm{ctx}},\, z^{\mathrm{dyn}}_1,\, \hat{a}_1,\, z^{\mathrm{dyn}}_2,\, \hat{a}_2,\, \dots]$ (sizes $N_c{=}1280$, $N_d{=}80$, vocabulary $K{=}4375$; see Appendix~\ref{app:vwm_training}). A causal Transformer then models this sequence via standard next-token prediction:
\begin{equation}\label{eq:ar}
\mathcal{L}_{\mathrm{AR}}
\;=\;
-\!\sum_{t=1}^{T}
\log\, p_\theta\!\bigl(
z^{\mathrm{dyn}}_t \mid z^{\mathrm{ctx}},\, z^{\mathrm{dyn}}_{<t},\, \hat{a}_{\le t}
\bigr).
\end{equation}
At inference time the predicted tokens $\hat{z}^{\mathrm{dyn}}_{1:T}$ are decoded back to pixels via $\hat{x}_{1:T}\!=\!\mathcal{D}_{\mathrm{vis}}(z^{\mathrm{ctx}},\hat{z}^{\mathrm{dyn}}_{1:T})$. The resulting rollout is used for both reward-aligned post-training (\S\ref{sec:reward}) and sliding-window re-encoding (\S\ref{sec:swr}).
\begin{figure}[t]
  \centering
  \includegraphics[width=1\linewidth]{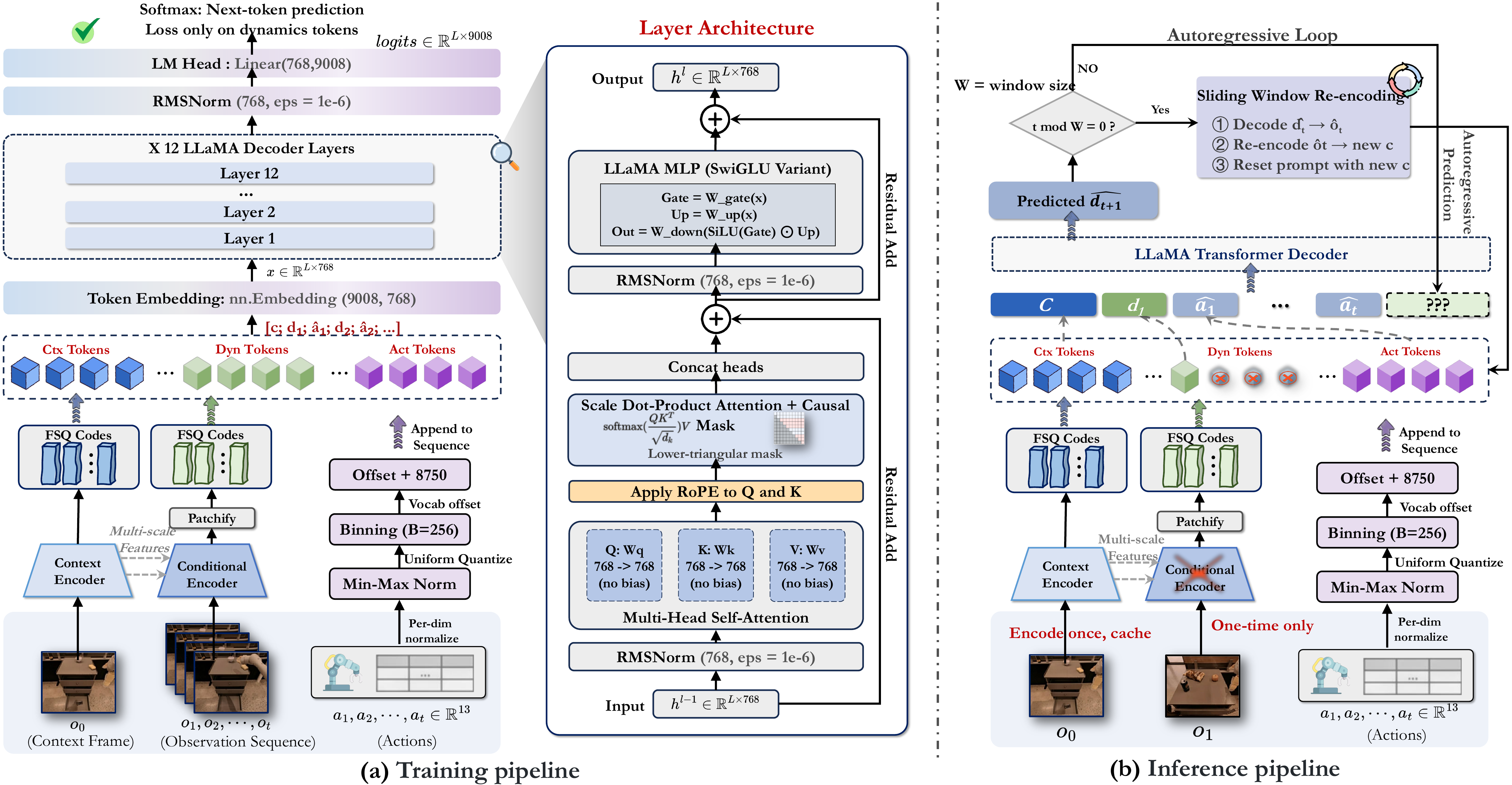}
  \caption{\textbf{Token-based robot video world model.}
  \textbf{(a)}~Training: a dual-branch FSQ tokenizer produces context tokens $c$ and dynamics tokens $d_t$; discretized action tokens are interleaved and modeled by a 12-layer LLaMA Transformer with loss on dynamics tokens only.
  \textbf{(b)}~Inference: context tokens are encoded once and cached; the model autoregressively predicts $\hat{d}_{t+1}$ and triggers sliding-window re-encoding every $W$ steps.}
  \label{fig:model}
\end{figure}
In our implementation, a dual-branch FSQ tokenizer~\cite{mentzer2024finite,wu2024ivideogpt} encodes the conditioning frame into context tokens and each subsequent observation into dynamics tokens, while continuous actions are normalized, uniformly binned, and assigned a disjoint vocabulary range. A $12$-layer LLaMA decoder models the interleaved sequence under a causal mask with loss applied only to dynamics tokens. During inference, context tokens are encoded once and cached, and sliding-window re-encoding is triggered every $W$ steps. More technical details can be found in Appendix~\ref{app:vwm_training}.

\vspace{-5pt}
\subsection{Reward-Aligned Post-Training}\label{sec:reward}
\vspace{-5pt}
The pre-training objective (Eq.~\ref{eq:ar}) captures surface-level statistical regularity but is agnostic to instruction following, physical plausibility, and action--outcome consistency.
We close this gap via RL post-training with a structured multimodal reward, obtained in three stages: benchmark construction, teacher training, and reward distillation.
\vspace{-10pt}
\paragraph{RobotWorldBench.}
Let $l$ denote an instruction and $v^{+}$ a ground-truth manipulation video from established datasets. RobotWorldBench begins from a candidate pool drawn from four robot datasets and combines two annotation sources: rule-based degradations of ground-truth episodes and generated videos from open-source image-to-video or world-model baselines. From this pool we curate 10{,}000 annotated video--instruction pairs for judge training and reward distillation. Each retained pair $(l,\,v)$ is annotated with a raw score vector $\mathbf{r}\!=\!(r_1,\dots,r_6)$ along six dimensions \emph{instruction following}, \emph{manipulation success}, \emph{action--outcome consistency}, \emph{temporal consistency}, \emph{contact realism}, and \emph{physics adherence} using the original rubric ranges $[3,2,1,1,1,2]$, yielding
\begin{equation}\label{eq:bench}
\mathcal{D}_{\mathrm{bench}}
= \bigl\{\,(l_i,\; v_i,\; \mathbf{r}_i)\bigr\}_{i=1}^{N}.
\end{equation}

\paragraph{Multimodal teacher judge.}
We fine-tune Qwen$3$\--VL\--$8$B\--Thinking~\cite{bai2025qwen25vl} as the teacher judge $f_\phi$.
Given $(l,\,v)$, the teacher produces structured raw scores $\hat{\mathbf{r}}\!=\!f_\phi(l,v)$ in the same six-dimensional rubric space as $\mathbf{r}$ via supervised fine-tuning on $\mathcal{D}_{\mathrm{bench}}$:
\begin{equation}\label{eq:teacher}
\mathcal{L}_{\mathrm{teacher}}
\;=\;
-\,\mathbb{E}_{(l,v,\mathbf{r})\,\sim\,\mathcal{D}_{\mathrm{bench}}}
\!\bigl[\log\, p_\phi(\mathbf{r}\mid l,\,v)\bigr].
\end{equation}

\paragraph{Student reward distillation.}
The teacher's autoregressive decoding cost precludes its use as an online RL reward.
We distill its judgments into a lightweight student $g_\psi$---a compact visual--text encoder followed by a sigmoid-activated head emitting $g_\psi(l,v)\!\in\![0,1]^6$. The distillation set $\mathcal{D}_{\mathrm{distill}}$ mixes annotated benchmark videos with teacher-scored rollouts from baseline/current world models. Teacher raw scores $f_\phi(l,v)$ are normalized dimension-wise via $\tilde{s}_k \!=\! [f_\phi(l,v)]_k / r_k^{\max}$ with $r_k^{\max}\!\in\!\{3,2,1,1,1,2\}$ so both sides of the regression share the same $[0,1]$ scale:
\begin{equation}\label{eq:distill}
\mathcal{L}_{\mathrm{distill}}
\;=\;
\mathbb{E}_{(l,v)\,\sim\,\mathcal{D}_{\mathrm{distill}}}
\sum_{k=1}^{6} \lambda_k\,\mathrm{Huber}_{\delta_h}\!\bigl(
[g_\psi(l,\,v)]_k,\; \tilde{s}_k(l,v)
\bigr),
\end{equation}
with Huber threshold $\delta_h\!=\!0.5$ (distinct from the quantization $\delta_q$ of Appendix~\ref{app:proof_swr}) and per-dimension balancing weights $\{\lambda_k\}$ (Appendix~\ref{app:student_reward}); these differ from the reward-aggregation weights $\{w_k\}$ in Eq.~\ref{eq:reward}. To counter reward hacking from distributional shift, we use \emph{online iterative distillation}: every $K$ policy updates, fresh rollouts are re-scored by the teacher to refresh the student.

\paragraph{RL post-training with GRPO.}
The distilled student provides a composite reward over the six dimensions:
\begin{equation}\label{eq:reward}
R(\hat{x}_{1:T})
\;=\;
\sum_{k=1}^{6} w_k \;\bigl[g_\psi(l,\,\hat{x}_{1:T})\bigr]_k,
\end{equation}
where $[g_\psi]_k\!\in\![0,1]$ are the student's normalized scores (same scale as in Eq.~\ref{eq:distill}) and $\{w_k\}$ are importance weights.
We post-train with GRPO~\cite{shao2024deepseekmath,wu2025rlvr}: given a group of $G$ rollouts $\{\hat{x}^{(j)}\}_{j=1}^{G}\!\sim\!p_\theta$, the group-normalized advantage is
\begin{equation}\label{eq:advantage}
A^{(j)}
\;=\;
\frac{R\!\bigl(\hat{x}^{(j)}\bigr) - \operatorname{mean}_{j}\, R}
     {\operatorname{std}_{j}\, R},
\end{equation}
and the clipped policy-gradient objective reads
\begin{equation}\label{eq:grpo}
\mathcal{L}_{\mathrm{GRPO}}
\;=\;
-\,\mathbb{E}\!\Biggl[
\frac{1}{G}\sum_{j=1}^{G}
\min\!\Bigl(\rho^{(j)} A^{(j)},\;
\operatorname{clip}\!\bigl(\rho^{(j)},\, 1{-}\epsilon,\, 1{+}\epsilon\bigr)\, A^{(j)}\Bigr)
\Biggr]
+ \beta\, D_{\mathrm{KL}}\!\bigl(p_\theta \,\big\|\, p_{\theta_0}\bigr),
\end{equation}
where $\rho^{(j)}\!=\!p_\theta(\hat{x}^{(j)})/p_{\theta_{\mathrm{old}}}(\hat{x}^{(j)})$ with $\theta_{\mathrm{old}}$ synchronized to $\theta$ at each GRPO iteration, and $p_{\theta_0}$ is the frozen pre-trained reference used for KL regularization.
The student scores each rollout in one forward pass, reducing reward cost by over $10{\times}$ relative to the teacher while preserving high-level judgment fidelity.

\subsection{Sliding Window Re-encoding}\label{sec:swr}
\begin{figure}[t]
  \centering
  \includegraphics[width=1\linewidth]{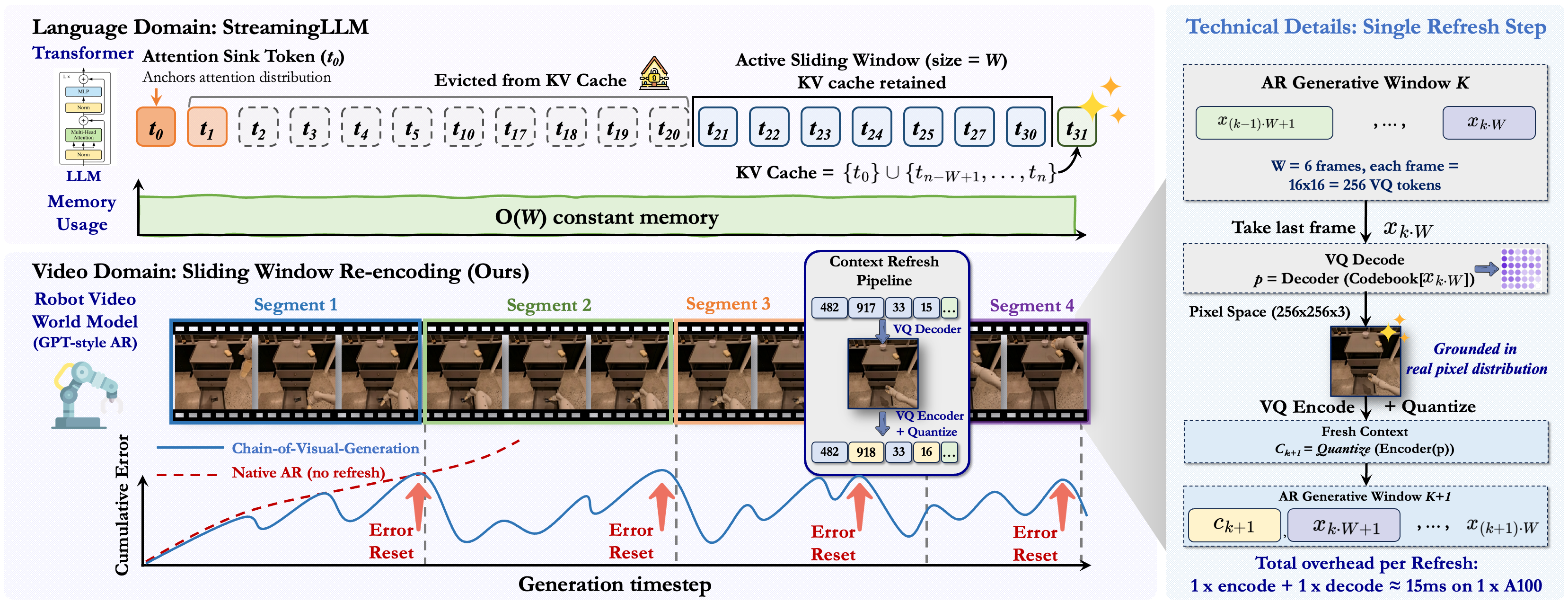}
  \caption{\textbf{Sliding window re-encoding.}
  \textit{Top:} analogy to StreamingLLM~\cite{xiao2023efficient}, which retains an attention sink and a sliding KV-cache window in the language domain.
  \textit{Bottom:} our approach periodically decodes the last predicted frame to pixel space, re-encodes it as fresh context tokens, and resets the autoregressive prompt, empirically limiting long-horizon drift while keeping the active KV-cache bounded by $O(W)$.
  \textit{Right:} technical details of a single refresh step.}
  \label{fig:Inference}
\end{figure}
Under standard autoregressive decoding, each dynamics-token block $\hat{z}^{\mathrm{dyn}}_t$ is conditioned on all previously \emph{predicted} tokens, so per-step errors compound and progressively degrade long-horizon rollouts.
Inspired by the attention-sink mechanism of StreamingLLM~\cite{xiao2023efficient} in the language domain, we introduce a training-free decoding strategy that periodically decodes recent predictions to pixel space and re-encodes them as fresh context, which can limit the carry-over of token-level drift across segments (Figure~\ref{fig:Inference}).

\paragraph{Segmented generation.}
A rollout of $T$ frames is partitioned into $K\!=\!\lceil T/W\rceil$ segments of window size $W$ (Figure~\ref{fig:Inference}, right).
Within segment $k$, the model generates $W$ dynamics-token blocks from context $z^{\mathrm{ctx}}_k$:
\begin{equation}\label{eq:sw_seg}
\hat{z}^{\mathrm{dyn}}_{(k-1)W+1:\,kW}
\;\sim\;
\prod_{t=(k-1)W+1}^{kW}
p_\theta\!\bigl(z^{\mathrm{dyn}}_t \mid z^{\mathrm{ctx}}_k,\;\hat{z}^{\mathrm{dyn}}_{<t},\;\hat{a}_{\le t}\bigr),
\end{equation}
where $z^{\mathrm{ctx}}_1$ is the original context encoding of the first conditioning frame.

\paragraph{Context refresh.}
At each segment boundary, the last predicted frame is decoded to pixel space and re-encoded as fresh context for the next segment (Figure~\ref{fig:Inference}, right):
\begin{equation}
\hat{x}_{kW} = \bigl[\mathcal{D}_{\mathrm{vis}}\!\bigl(z^{\mathrm{ctx}}_k,\;\hat{z}^{\mathrm{dyn}}_{(k-1)W+1:\,kW}\bigr)\bigr]_{t=kW}, \qquad
\bigl(z^{\mathrm{ctx}}_{k+1},\, z^{\mathrm{dyn}}_0\bigr) = \mathcal{T}_{\mathrm{vis}}\!\bigl(\hat{x}_{kW}\bigr), \label{eq:decode}
\end{equation}
where $\hat{x}_{kW}$ serves as both the new conditioning frame for $E_c$ and a length-one dynamic clip for $E_d$; $z^{\mathrm{dyn}}_0$ seeds the next segment's prompt (Appendix~\ref{app:vwm_stage1}). This \emph{decode--re-encode} cycle makes subsequent predictions depend on the current decoded frame rather than the entire raw token history.

\paragraph{Stylized stability analysis.}
We provide a simplified local error analysis that explains the qualitative window-size trade-off, rather than predicting exact gains.
Let $\varepsilon$ upper-bound the per-step token-space prediction error, $\delta_q$ the decode--re-encode quantization error, and $\alpha\!\in\![0,1)$ a local contraction factor on within-segment context-error amplification (precise assumptions in Appendix~\ref{app:proof_swr}).

\begin{proposition}\label{prop:swr}
Under the stylized model above, sliding-window re-encoding with window size $W$ yields the bound
\begin{equation}\label{eq:error_bound}
\mathcal{E}_{\mathrm{SWR}}(T)
\;\le\;
W\varepsilon + \frac{W\varepsilon + \delta_q}{1 - \alpha^{W}},
\qquad \forall\; T,
\end{equation}
which does not grow explicitly with $T$; the $\alpha\!\to\!0$ limit simplifies to $2W\varepsilon + \delta_q$. By contrast, vanilla AR without refresh admits a worst-case bound of order $\varepsilon/(1-\alpha)$, which blows up as $\alpha\!\to\!1$ and degrades to $\mathcal{O}(T\varepsilon)$ in the non-contractive regime.
\end{proposition}

\begin{proof}[Proof sketch]
Within segment $k$ the context is fixed, so errors accumulate for $\le W$ steps, giving $\le W\varepsilon + \alpha^W \eta_k$ where $\eta_k$ is the context error. The boundary cycle (Eq.~\ref{eq:decode}) refreshes context at a one-time cost $\delta_q$, so $\eta_{k+1}\!\le\! W\varepsilon + \alpha^W \eta_k + \delta_q$ converges to $\eta^{*}\!=\!(W\varepsilon+\delta_q)/(1-\alpha^W)$; see Appendix~\ref{app:proof_swr}.
\end{proof}

\begin{remark}[\textbf{Role of $\alpha$ and choice of $W$}]\label{rem:alpha}
Smaller $W$ tightens the within-segment term $W\varepsilon$ but pays more $\delta_q$ per refresh; larger $W$ leaves more room for within-segment drift. An asymptotic analysis of Eq.~\ref{eq:error_bound} (Corollary~\ref{cor:wstar} in Appendix~\ref{app:proof_swr}) predicts that the optimal window grows logarithmically in $\delta_q/\varepsilon$ and in $1/(-\ln\alpha)$, i.e.\ more local drift or cheaper refresh both push $W^{*}$ larger. Our ablation (Table~\ref{tab:swr_quality_efficiency}) shows $W\!=\!6$ hits the best balance, consistent with this prediction for our operating regime.
\end{remark}


\definecolor{tblheadbg}{HTML}{E8EBF3}
\definecolor{tblstripe}{HTML}{F7F8FB}
\definecolor{rankfirstbg}{HTML}{FCE9E7}
\definecolor{ranksecondbg}{HTML}{FEF1E5}
\definecolor{rankthirdbg}{HTML}{E8F7FB}
\newcommand{\scorestd}[2]{#1{\scriptsize$_{\pm#2}$}}
\newcommand{\bestscore}[2]{\cellcolor{rankfirstbg}\textbf{\scorestd{#1}{#2}}\textsuperscript{\textcolor{red!70!black}{\scriptsize 1}}}
\newcommand{\secondscore}[2]{\cellcolor{ranksecondbg}\textbf{\scorestd{#1}{#2}}\textsuperscript{\textcolor{orange!85!black}{\scriptsize 2}}}
\newcommand{\thirdscore}[2]{\cellcolor{rankthirdbg}\textbf{\scorestd{#1}{#2}}\textsuperscript{\textcolor{cyan!60!black}{\scriptsize 3}}}
\newcommand{\plainscore}[2]{\scorestd{#1}{#2}}
\newcommand{\tbd}{\textcolor{codegray}{\textsc{tbd}}}
\newcommand{\placeholderfigure}[2]{
  \fbox{%
    \begin{minipage}[c][#1][c]{0.96\linewidth}
      \centering
      \textbf{Placeholder figure}\\[0.35em]
      \footnotesize #2
    \end{minipage}}
}

\section{Experiments}


\textbf{Datasets.} We evaluate \method{} on RT-1~\cite{brohan2023rt1} and BridgeData~V2~\cite{walke2023bridgedata} in the main rollout experiments. RobotWorldBench draws candidate tasks and videos from four robot datasets, including RT-1, BridgeData~V2, CALVIN~\cite{mees2022calvin}, and LIBERO~\cite{liu2024libero}, and uses 10{,}000 annotated video--instruction pairs for judge training and reward distillation.

\textbf{Baselines.} We compare \method{} with three groups of approaches: closed-source video generation models, open-source video generation models, and embodied world models. Additional details are deferred to the Table~\ref{tab:judge_i2v_models}.

\textbf{Evaluation Metrics.} We report both semantic/physical alignment metrics and pixel-level reconstruction metrics. The former are scored by \textsc{RoboAlign-Judge} over six dimensions and should be interpreted as an in-domain automated proxy rather than a substitute for large-scale human evaluation. We additionally report an external VLM-based cross-check and a small-scale blinded human study on held-out subsets to test whether the main ranking is preserved beyond the in-domain judge (Appendix~\ref{app:judge_vlm_comparison} and Appendix~\ref{app:human_eval_template}). The pixel-level metrics include standard global measures and motion-mask-based ROI metrics; details are provided in Appendix~\ref{app:eval_metrics}.

\textbf{Implementation Details.} \method{} is implemented in PyTorch, and all training and evaluation are conducted on NVIDIA A100 GPUs. In the RL stage, we adopt GRPO and distill the fine-tuned RoboAlign-Judge into a lightweight reward model.

\subsection[RQ1]{RQ1: Does RoboAlign-R1 Achieve State-of-the-Art World-Model Quality?}

\paragraph{Quantitative results.}
Tables~\ref{tab:robotworldbench_main}--\ref{tab:robotworldbench_lowlevel} show that \method{} achieves the best overall performance under our evaluation protocol. It reaches an aggregate \textsc{RoboAlign-Judge} score of $\mathbf{8.52}_{\pm 0.15}$, exceeding the strongest baseline iVideoGPT ($7.74_{\pm 0.62}$) by $+10.1\%$ with $75.8\%$ lower standard deviation, while improving all six judged dimensions and surpassing all closed-source commercial models (e.g., Kling~2.6: $6.84$) despite far fewer parameters.
On low-level metrics, \method{} is also strongest on both datasets, reducing LPIPS by $4.9\%$ on RT-1 and MSE by $8.7\%$ on BridgeData~V2 relative to the respective runners-up. The main ranking is preserved by the external VLM-based validation (Appendix~\ref{app:judge_vlm_comparison}) and the blinded human study (Appendix~\ref{app:human_eval_template}).

\paragraph{Qualitative analysis.}
Figures~\ref{fig:case}--\ref{fig:case_rt_bridge} show representative comparisons. \method{} produces more coherent manipulation sequences with accurate grasping and stable contact geometry. In contrast, iVideoGPT often produces blurry details, while Wan2.2-TI2V-5B (LoRA) struggles with temporal coherence and occasionally hallucinates object deformation. Across RT-1 and BridgeData~V2, \method{} better preserves texture, shadows, and background stability, consistent with the quantitative gains above.

\begin{table*}[!h]
  \centering
  \tiny
  \renewcommand{\arraystretch}{0.92}
  \setlength{\tabcolsep}{2.4pt}
  \vspace{-5pt}
  \caption{Robot world-model performance on \textsc{RobotWorldBench} (mean $\pm$ std.).}
  \label{tab:robotworldbench_main}
  \resizebox{\textwidth}{!}{
  \begin{tabular}{lccccccc}
    \toprule
    \rowcolor{tblheadbg}
    Method & \multicolumn{3}{c}{Task Alignment} & \multicolumn{3}{c}{Physical Realism} & Total \\
    \cmidrule(lr){2-4} \cmidrule(lr){5-7}
    \rowcolor{tblheadbg}
     & Instr. $\uparrow$ & Manip. $\uparrow$ & Act.-Out. $\uparrow$ & Temp. $\uparrow$ & Contact $\uparrow$ & Phys. $\uparrow$ & $\uparrow$ \\
    \midrule
    \rowcolor{tblstripe}
    \textit{Real videos}             & 3.0 & 2.0 & 1.0 & 1.0 & 1.0 & 2.0 & 10.0 \\
    \midrule
    \multicolumn{8}{l}{\textcolor{red!70!black}{\textbf{\textit{Closed video models}}}} \\
    \rowcolor{tblstripe}
    Kling 2.6             & \thirdscore{2.42}{0.09} & \thirdscore{1.38}{0.10} & \plainscore{0.46}{0.08} & \thirdscore{0.58}{0.07} & \plainscore{0.82}{0.06} & \plainscore{1.18}{0.11} & \thirdscore{6.84}{0.22} \\
    Runway Gen-4.5        & \plainscore{2.34}{0.10} & \plainscore{1.29}{0.09} & \plainscore{0.43}{0.07} & \plainscore{0.55}{0.08} & \plainscore{0.79}{0.06} & \plainscore{1.12}{0.10} & \plainscore{6.52}{0.24} \\
    \rowcolor{tblstripe}
    MiniMax Hailuo 02     & \plainscore{2.18}{0.11} & \plainscore{1.17}{0.10} & \plainscore{0.38}{0.07} & \plainscore{0.47}{0.09} & \plainscore{0.71}{0.08} & \plainscore{1.01}{0.12} & \plainscore{5.92}{0.28} \\
    Luma Dream Machine    & \plainscore{2.03}{0.12} & \plainscore{1.08}{0.11} & \plainscore{0.35}{0.08} & \plainscore{0.44}{0.09} & \plainscore{0.69}{0.08} & \plainscore{0.92}{0.12} & \plainscore{5.51}{0.30} \\
    \midrule
    \multicolumn{8}{l}{\textcolor{codegreen!60!black}{\textbf{\textit{Open video models}}}} \\
    \rowcolor{tblstripe}
    HunyuanVideo-I2V      & \plainscore{1.20}{0.28} & \plainscore{0.40}{0.22} & \plainscore{0.24}{0.08} & \plainscore{0.56}{0.15} & \thirdscore{0.96}{0.08} & \secondscore{1.56}{0.15} & \plainscore{4.92}{0.63} \\
    LTX-Video             & \plainscore{2.28}{0.07} & \plainscore{1.32}{0.07} & \thirdscore{0.52}{0.07} & \plainscore{0.26}{0.05} & \plainscore{0.28}{0.10} & \plainscore{1.00}{0.18} & \plainscore{5.66}{0.36} \\
    \rowcolor{tblstripe}
    Stable Video Diffusion XT & \plainscore{1.84}{0.15} & \plainscore{0.98}{0.13} & \plainscore{0.46}{0.08} & \plainscore{0.06}{0.08} & \plainscore{0.36}{0.20} & \plainscore{0.58}{0.20} & \plainscore{4.28}{0.74} \\
    Mochi-1               & \plainscore{1.83}{0.12} & \plainscore{0.91}{0.10} & \plainscore{0.31}{0.07} & \plainscore{0.39}{0.08} & \plainscore{0.61}{0.07} & \plainscore{0.82}{0.11} & \plainscore{4.87}{0.29} \\
    \rowcolor{tblstripe}
    I2VGen-XL             & \plainscore{1.56}{0.15} & \plainscore{0.66}{0.14} & \plainscore{0.32}{0.04} & \plainscore{0.00}{0.00} & \plainscore{0.44}{0.19} & \plainscore{0.56}{0.16} & \plainscore{3.54}{0.57} \\
    CogVideoX-I2V         & \plainscore{1.94}{0.21} & \plainscore{1.00}{0.14} & \plainscore{0.40}{0.18} & \plainscore{0.42}{0.16} & \plainscore{0.92}{0.07} & \plainscore{1.24}{0.19} & \plainscore{5.92}{0.75} \\
    \rowcolor{tblstripe}
    OpenSora-I2V          & \plainscore{1.58}{0.13} & \plainscore{0.82}{0.11} & \plainscore{0.27}{0.06} & \plainscore{0.31}{0.07} & \plainscore{0.53}{0.08} & \plainscore{0.71}{0.10} & \plainscore{4.22}{0.30} \\
    OpenSora-Plan-I2V     & \plainscore{1.79}{0.12} & \plainscore{0.89}{0.10} & \plainscore{0.30}{0.07} & \plainscore{0.37}{0.08} & \plainscore{0.60}{0.08} & \plainscore{0.81}{0.11} & \plainscore{4.76}{0.29} \\
    \midrule
    \multicolumn{8}{l}{\textcolor{codepurple}{\textbf{\textit{Embodied / interactive world-model baselines}}}} \\
    \rowcolor{tblstripe}
    RLVR-World            & \plainscore{2.29}{0.10} & \plainscore{1.31}{0.09} & \plainscore{0.44}{0.07} & \plainscore{0.43}{0.08} & \plainscore{0.84}{0.06} & \plainscore{1.19}{0.10} & \plainscore{6.54}{0.23} \\
    iVideoGPT             & \secondscore{2.60}{0.11} & \secondscore{1.60}{0.11} & \secondscore{0.70}{0.11} & \secondscore{0.74}{0.14} & \plainscore{0.56}{0.10} & \thirdscore{1.54}{0.15} & \secondscore{7.74}{0.62} \\
    \rowcolor{tblstripe}
    RoboDreamer           & \plainscore{2.02}{0.16} & \plainscore{1.02}{0.16} & \plainscore{0.34}{0.05} & \plainscore{0.32}{0.12} & \plainscore{0.04}{0.05} & \plainscore{0.70}{0.17} & \plainscore{4.44}{0.56} \\
    Vid2World             & \plainscore{2.18}{0.10} & \plainscore{1.08}{0.04} & \plainscore{0.40}{0.09} & \plainscore{0.22}{0.07} & \secondscore{0.98}{0.04} & \plainscore{1.04}{0.14} & \plainscore{5.90}{0.30} \\
    \midrule
    \multicolumn{8}{l}{\textcolor{codepurple}{\textbf{\textit{Additional training baselines}}}} \\
    \rowcolor{tblstripe}
    Wan2.2-TI2V-5B (LoRA) & \plainscore{2.40}{0.09} & \plainscore{1.02}{0.11} & \plainscore{0.41}{0.08} & \plainscore{0.24}{0.10} & \secondscore{0.98}{0.05} & \plainscore{0.96}{0.12} & \plainscore{6.01}{0.24} \\
    \midrule
    \rowcolor{mydarkblue!6}
    \textcolor{mydarkblue}{\textbf{\method{} (ours)}} & \bestscore{2.72}{0.06} & \bestscore{1.72}{0.07} & \bestscore{0.72}{0.05} & \bestscore{0.78}{0.05} & \bestscore{1.00}{0.04} & \bestscore{1.58}{0.07} & \bestscore{8.52}{0.15} \\
    \bottomrule
  \end{tabular}}
\vspace{-10pt}
\end{table*}

\begin{table*}[!h]
  \centering
  \tiny
  \renewcommand{\arraystretch}{0.93}
  \setlength{\tabcolsep}{2.0pt}
  \caption{Low-level metrics on full variable-length rollouts over RT-1 and BridgeData V2. We report mean values and, when available, standard deviations across repeated evaluations over the full generated rollout.}
  \label{tab:robotworldbench_lowlevel}
  \resizebox{\textwidth}{!}{
  \begin{tabular}{lcccccccc}
    \toprule
    \rowcolor{tblheadbg}
    Method & \multicolumn{4}{c}{RT-1} & \multicolumn{4}{c}{BridgeData V2} \\
    \cmidrule(lr){2-5} \cmidrule(lr){6-9}
    \rowcolor{tblheadbg}
     & MSE $\downarrow$ & PSNR $\uparrow$ & SSIM $\uparrow$
     & LPIPS $\downarrow$
     & MSE $\downarrow$ & PSNR $\uparrow$ & SSIM $\uparrow$
     & LPIPS $\downarrow$ \\
    \midrule
    \rowcolor{tblstripe}
    \textit{Real videos}      & 0.000 & --- & 1.000 & 0.000 & 0.000 & --- & 1.000 & 0.000 \\
    \midrule
    \multicolumn{9}{l}{\textcolor{codepurple}{\textbf{\textit{Embodied / interactive world-model baselines}}}} \\
    \rowcolor{tblstripe}
    RLVR-World                & \secondscore{0.0125}{0.0017} & \secondscore{19.47}{1.30} & \secondscore{0.736}{0.035} & \secondscore{0.182}{0.014} & \secondscore{0.0139}{0.0021} & \secondscore{19.06}{1.38} & \secondscore{0.721}{0.040} & \secondscore{0.178}{0.014} \\
    iVideoGPT                 & \thirdscore{0.0128}{0.0018} & \thirdscore{19.36}{1.32} & \thirdscore{0.732}{0.036} & \thirdscore{0.188}{0.015} & \plainscore{0.0438}{0.0025} & \plainscore{15.98}{1.52} & \plainscore{0.558}{0.042} & \plainscore{0.435}{0.021} \\
    \rowcolor{tblstripe}
    RoboDreamer               & \plainscore{0.0171}{0.0024} & \plainscore{18.12}{1.56} & \plainscore{0.671}{0.046} & \plainscore{0.231}{0.019} & \plainscore{0.0184}{0.0027} & \plainscore{17.86}{1.69} & \plainscore{0.649}{0.051} & \plainscore{0.224}{0.018} \\
    Vid2World                 & \plainscore{0.0138}{0.0019} & \plainscore{19.03}{1.37} & \plainscore{0.721}{0.038} & \plainscore{0.196}{0.016} & \plainscore{0.0145}{0.0022} & \plainscore{18.92}{1.43} & \plainscore{0.711}{0.041} & \plainscore{0.188}{0.015} \\
    \midrule
    \multicolumn{9}{l}{\textcolor{codepurple}{\textbf{\textit{Additional training baselines}}}} \\
    \rowcolor{tblstripe}
    Wan2.2-TI2V-5B (LoRA)     & \plainscore{0.0131}{0.0018} & \plainscore{19.31}{1.29} & \plainscore{0.729}{0.037} & \plainscore{0.184}{0.015} & \thirdscore{0.0141}{0.0021} & \thirdscore{19.02}{1.39} & \thirdscore{0.718}{0.040} & \thirdscore{0.182}{0.014} \\
    \midrule
    \rowcolor{mydarkblue!6}
    \textcolor{mydarkblue}{\textbf{\method{} (ours)}} & \bestscore{0.0121}{0.0016} & \bestscore{19.72}{1.26} & \bestscore{0.745}{0.032} & \bestscore{0.173}{0.013} & \bestscore{0.0127}{0.0019} & \bestscore{20.00}{1.31} & \bestscore{0.731}{0.037} & \bestscore{0.175}{0.013} \\
    \bottomrule
  \end{tabular}}
  \vspace{-5pt}
\end{table*}

\begin{figure*}[!h]
\centering
\includegraphics[width=1\textwidth]{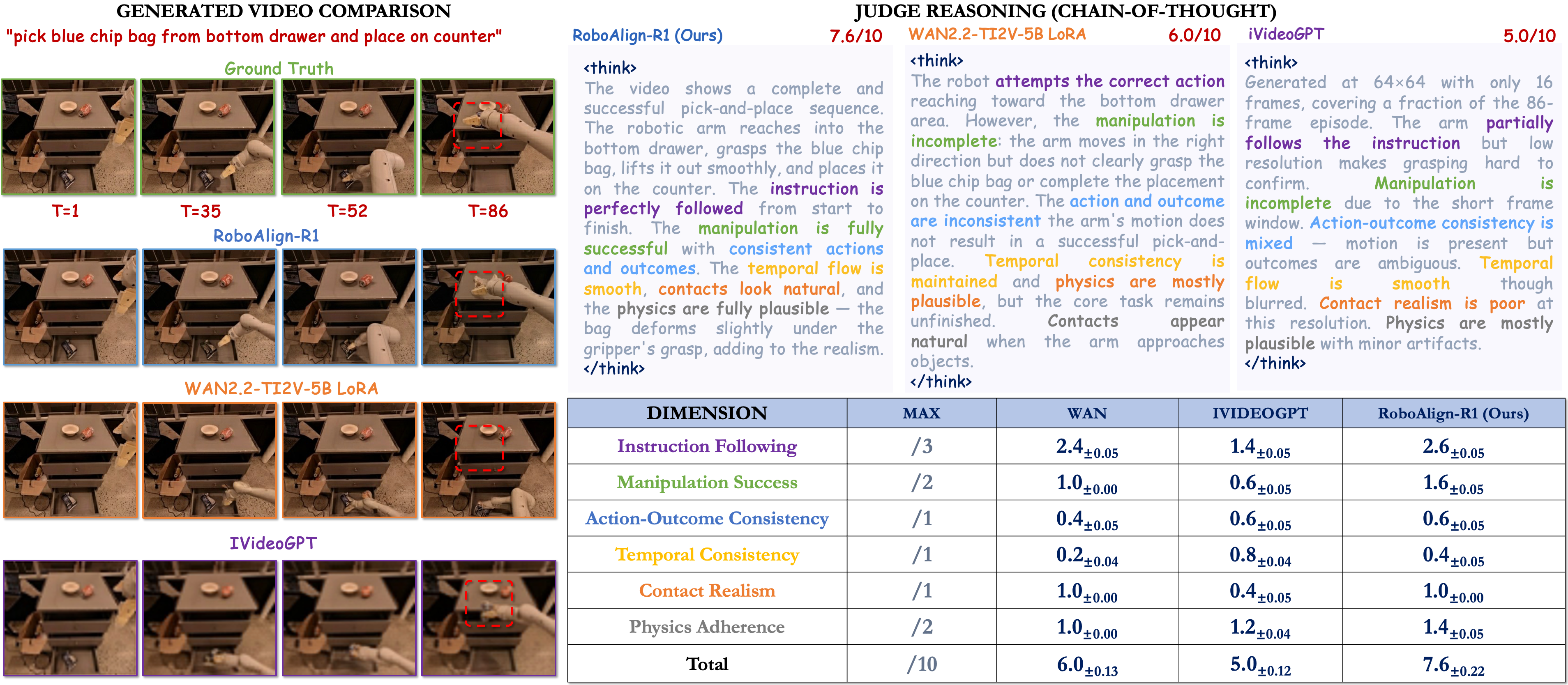}
\caption{\textbf{Qualitative comparison on a representative manipulation case.} \method{} generates physically coherent sequences with accurate grasping.}
\vspace{-10pt}
\label{fig:case}
\end{figure*}

\begin{figure}[t]
\centering
\includegraphics[width=1\linewidth]{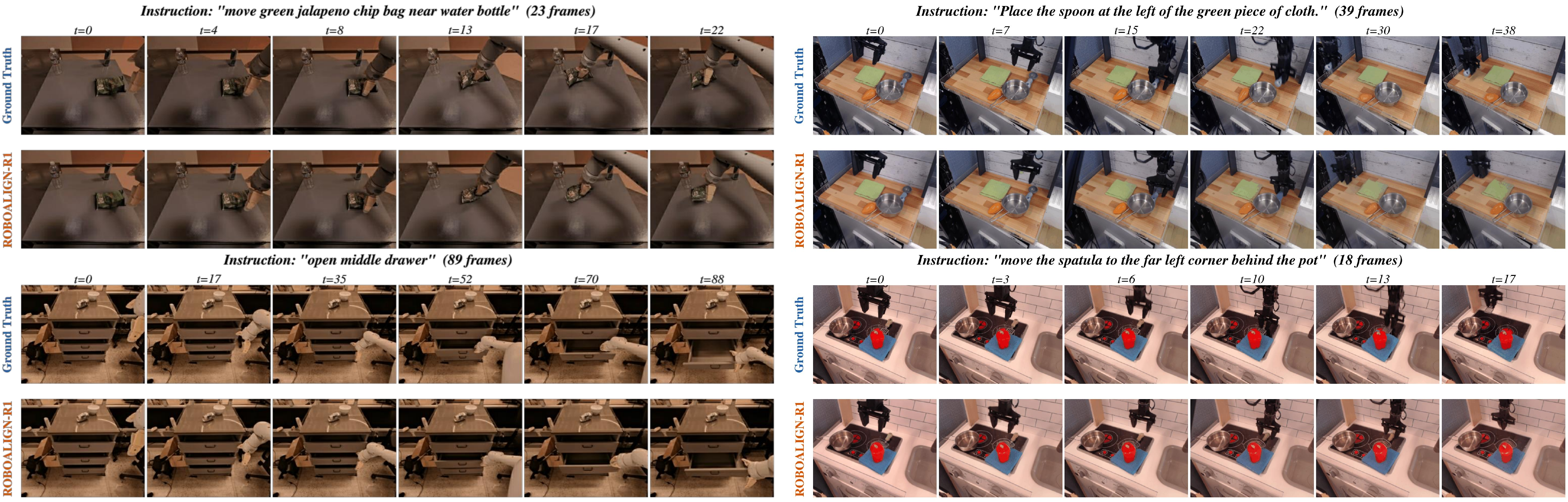}
\caption{\textbf{Qualitative case study of \method{} on RT-1 and BridgeData V2}, showing improved texture, shadow consistency, and background stability.}
\label{fig:case_rt_bridge}
\vspace{-10pt}
\end{figure}

\begin{figure*}[h!]
  \centering
  \includegraphics[width=\linewidth]{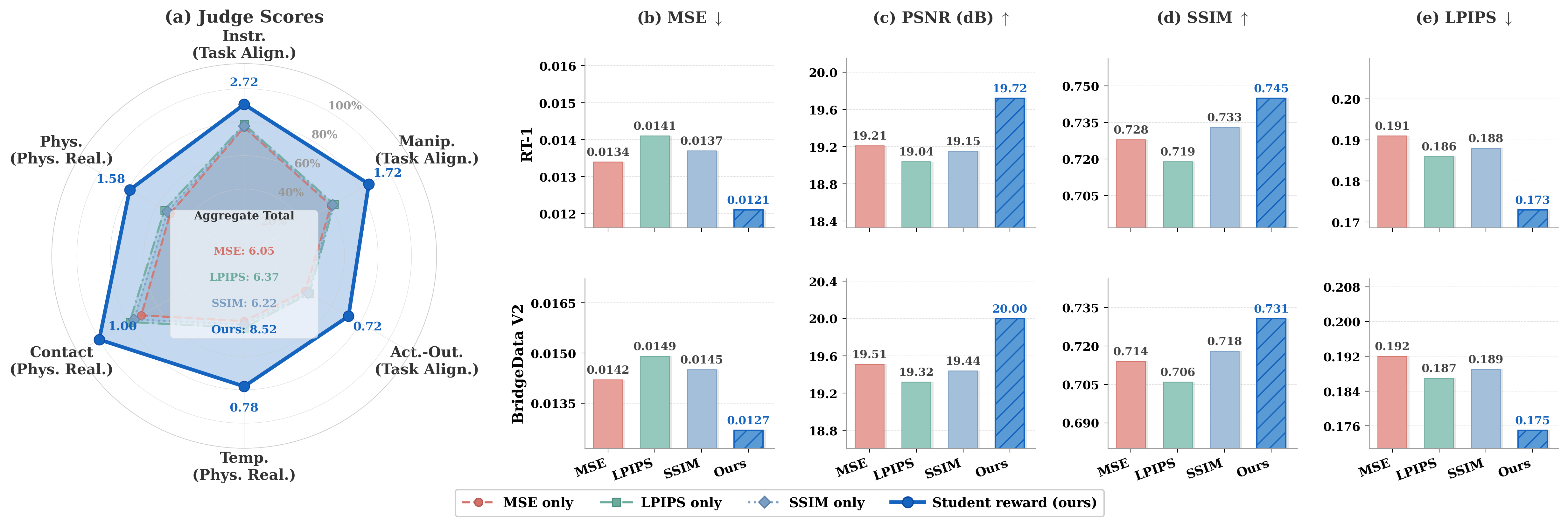}
  \caption{\textbf{{Reward-type ablation for RL post-training} on \textsc{RobotWorldBench}}. Panel (a) shows that the distilled student reward yields the strongest judge-aligned improvement across semantic and physical dimensions. Panels (b)--(e) show that it also delivers the best full-rollout low-level metrics on RT-1 and BridgeData~V2.}
  \label{fig:rq2_reward_ablation_judge}
\vspace{-18pt}
\end{figure*}

\subsection[RQ2]{RQ2: Do Distilled Multimodal Rewards Outperform Low-Level Alternatives?}
We compare the distilled multimodal student reward with low-level alternatives under matched RL post-training settings. Figure~\ref{fig:rq2_reward_ablation_judge}(a) shows that the student reward performs best on all six semantic and physical dimensions on \textsc{RobotWorldBench}, reaching the highest aggregate score ($8.52$, $+33.8\%$ over the best single-metric reward baseline, LPIPS). Figure~\ref{fig:rq2_reward_ablation_judge}(b)--(e) further shows that it also yields the best full-rollout low-level metrics on both RT-1 and BridgeData~V2. Overall, the distilled reward is better aligned with our in-domain evaluator and remains beneficial to pixel fidelity.

\vspace{-5pt}
\subsection[RQ3]{RQ3: Does SWR Improve Long-Horizon Quality Without Heavy Overhead?}
\vspace{-5pt}

We compare SWR with default autoregressive (AR) decoding on RT-1. Table~\ref{tab:swr_quality_efficiency}(a) shows that SWR with $W{=}6$ improves long-horizon generation quality, yielding $+2.8\%$ SSIM, $+0.62$\,dB PSNR, and $-9.8\%$ LPIPS. The gains are larger in the region of interest, where ROI-LPIPS drops by $-12.2\%$. These results are consistent with SWR reducing accumulated drift by periodically refreshing the context (Figure~\ref{fig:rq3}), and $W{=}6$ provides the best quality--efficiency trade-off among the tested window sizes. These gains come with only modest overhead. Table~\ref{tab:swr_quality_efficiency}(b) shows that SWR keeps total inference time ($5.709$\,s vs.\ $5.646$\,s) and throughput ($5.26$\,FPS vs.\ $5.31$\,FPS) close to AR, with about $1.1\%$ extra wall-clock time. SWR also bounds KV-cache growth to $O(W)$, reducing maximum sequence length by $54.8\%$ and peak memory by $4.2\%$. As a result, SWR maintains a near-flat per-frame latency profile, whereas AR shows a $+6.8\%$ latency drift by frame 29.

\begin{table*}[t]
  \centering
  \footnotesize
  \setlength{\tabcolsep}{2pt}
  \renewcommand{\arraystretch}{1.04}
  \caption{\textbf{SWR empirical summary} on \textsc{RT-1}. \textbf{(a)}~Default autoregressive decoding vs.\ SWR with $W{=}6$ (mean $\pm$ std.; $\Delta$ vs.\ Default AR). \textbf{(b)}~Total wall-clock time vs.\ $W$ and per-frame latency on a keyframe subsample.}
  \label{tab:swr_quality_efficiency}
  \begin{minipage}[t]{0.51\textwidth}
    \centering
    \vspace{0em}
    \begingroup
    \scriptsize
    \setlength{\aboverulesep}{0.15ex}
    \setlength{\belowrulesep}{0.15ex}
    \setlength{\tabcolsep}{1.7pt}\renewcommand{\arraystretch}{0.88}%
    \resizebox{0.975\linewidth}{!}{%
    \begin{tabular}{@{}lccc@{}}
      \toprule
      \rowcolor{tblheadbg}
      Metric & Default AR & SWR ($W{=}6$) & $\Delta$ \\
      \midrule
      \multicolumn{4}{@{}l}{\textit{Quality}} \\
      SSIM $\uparrow$ & \plainscore{0.7526}{0.084} & \plainscore{0.7735}{0.072} & $+2.8\%$ \\
      \rowcolor{tblstripe}
      PSNR (dB) $\uparrow$ & \plainscore{20.49}{3.22} & \plainscore{21.11}{3.23} & $+0.62$\,dB \\
      LPIPS $\downarrow$ & \plainscore{0.2078}{0.071} & \plainscore{0.1875}{0.064} & $-9.8\%$ \\
      \rowcolor{tblstripe}
      ROI-PSNR $\uparrow$ & \plainscore{16.86}{3.49} & \plainscore{17.48}{3.53} & $+0.62$\,dB \\
      ROI-LPIPS $\downarrow$ & \plainscore{0.1027}{0.042} & \plainscore{0.0902}{0.039} & $-12.2\%$ \\
      \midrule
      \multicolumn{4}{@{}l}{\textit{Efficiency}} \\
      \rowcolor{tblstripe}
      Infer.\ (s) & \plainscore{5.646}{0.01} & \plainscore{5.709}{0.07} & $+1.1\%$ \\
      Throughput (FPS) & $5.31$ & $5.26$ & $-0.9\%$ \\
      \rowcolor{tblstripe}
      Peak mem.\ (GB) & $33.4$ & $32.0$ & $-4.2\%$ \\
      Max seq.\ len. & $4{,}070$ & $1{,}838$ & $-54.8\%$ \\
      \rowcolor{tblstripe}
      KV-cache growth & $O(T)$ & $O(W)$ & bounded \\
      Re-enc.\ overhead & --- & $73.8$\,ms & $1.3\%$ \\
      \bottomrule
    \end{tabular}}
    \endgroup
    \setlength{\tabcolsep}{2pt}\renewcommand{\arraystretch}{1.04}
  \end{minipage}\hspace{0.4em}\begin{minipage}[t]{0.465\textwidth}
    \centering
    \vspace{0em}
    \begin{tikzpicture}
      \pgfplotsset{every axis/.append style={font=\tiny}}
      \begin{groupplot}[
        group style={
          group size=1 by 2,
          vertical sep=3pt,
          horizontal sep=0pt,
        },
        width=0.80\linewidth,
        height=1.78cm,
        scale only axis,
      ]
      \nextgroupplot[
        ybar,
        bar width=5.2pt,
        ymin=5.58,
        ymax=5.92,
        ylabel={s},
        symbolic x coords={4,6,8,10,15,AR},
        xticklabels={$4$,$6$,$8$,$10$,$15$,AR},
        xtick=data,
        enlarge x limits={0.14},
        ymajorgrids=true,
        major grid style={dashed,gray!45},
        tick align=inside,
      ]
        \addplot+[
          fill=mydarkblue!55,
          draw=mydarkblue!85,
          line width=0.9pt,
        ] coordinates {
          (4,5.863) (6,5.709) (8,5.858) (10,5.740) (15,5.732) (AR,5.646)
        };
        \draw[dashed, mydarkblue!65, line width=1.15pt]
          (axis cs:4,5.646) -- (axis cs:AR,5.646);
      \nextgroupplot[
        ylabel={ms},
        xmin=-1.5,
        xmax=30.5,
        ymin=168,
        ymax=222,
        xtick={0,5,10,15,20,25,29},
        legend style={
          font=\tiny,
          legend columns=3,
          /tikz/every even column/.append style={column sep=0.3em},
          at={(0.98,0.98)},
          anchor=north east,
        },
        legend cell align=left,
        ymajorgrids=true,
        major grid style={dashed,gray!40},
      ]
        \addplot+[black, line width=1.15pt, mark=*, mark size=1.35pt] coordinates {
          (0,177.8) (5,174.9) (10,181.1) (15,184.2) (20,185.7) (25,187.1) (29,189.9)
        };
        \addplot+[mydarkblue, line width=1.1pt, mark=square*, mark size=1.2pt] coordinates {
          (0,180.9) (5,173.8) (10,174.5) (15,176.4) (20,187.1) (25,173.7) (29,176.1)
        };
        \addplot+[gray!70, dashed, line width=1.05pt, mark=triangle*, mark size=1.2pt] coordinates {
          (0,198.6) (5,172.6) (10,173.3) (15,173.8) (20,174.7) (25,173.5) (29,172.5)
        };
        \legend{AR, $W{=}6$, $W{=}4$}
      \end{groupplot}
      \node[
        font=\bfseries\tiny,
        text=mydarkblue,
        fill=white,
        fill opacity=0.85,
        text opacity=1,
        inner sep=1.1pt,
      ] at ($ (group c1r1.north west) ! 0.5 ! (group c1r1.south east) $) {Total Time};
      \node[
        font=\bfseries\tiny,
        text=codepurple,
        fill=white,
        fill opacity=0.85,
        text opacity=1,
        inner sep=1.1pt,
      ] at ($ (group c1r2.north west) ! 0.5 ! (group c1r2.south east) $) {Per-Frame Latency};
    \end{tikzpicture}\\[-0.1em]
    {\centering\color{red!60!black}\bfseries\tiny Setup: 12L, $H{=}768$, 12 heads, vLLM, \mbox{300 episodes} $\times$ 30 frames.\par}
  \end{minipage}
\end{table*}

\begin{figure*}[t]
  \centering
  \includegraphics[width=\linewidth]{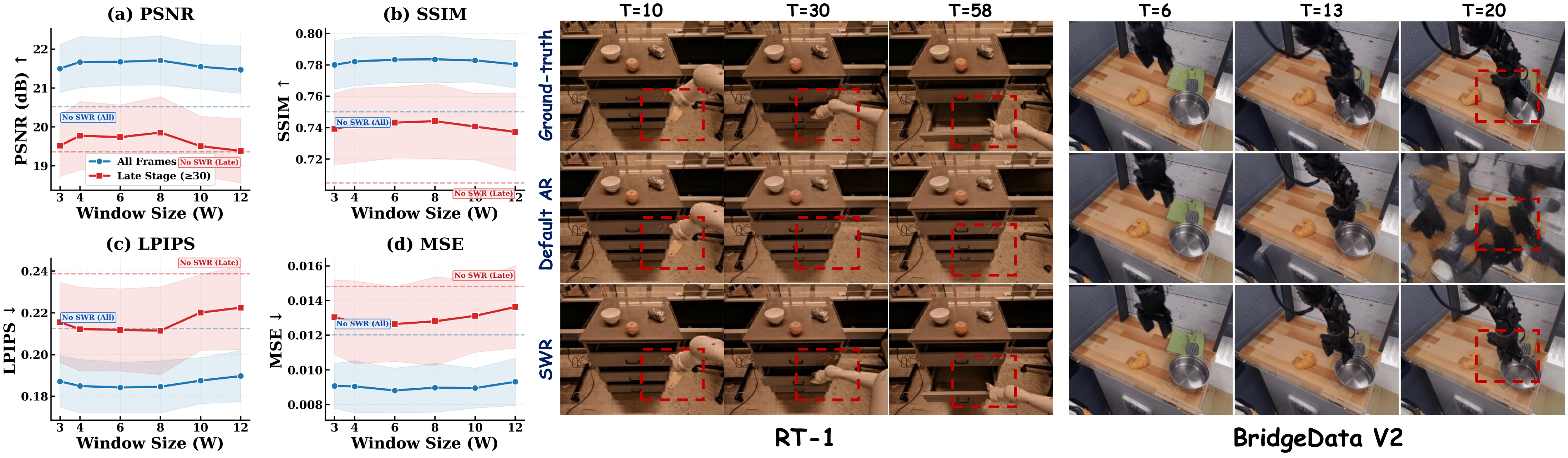}
  \caption{\textbf{Qualitative effect of SWR.} SWR mitigates error accumulation and maintains visual fidelity in long-horizon generation by periodically re-encoding recent context.}
  \vspace{-15pt}
  \label{fig:rq3}
\end{figure*}

\vspace{-5pt}
\subsection[RQ4]{RQ4: Does Online Iterative Distillation Help Maintain Reward Alignment?}
\vspace{-5pt}

\begin{wraptable}{r}{0.43\linewidth}
  \vspace{-1.2em}
  \centering
  \tiny
  \setlength{\tabcolsep}{2.4pt}
  \renewcommand{\arraystretch}{0.98}
  \caption{\textbf{Iterative distillation.}}
  \label{tab:compact_ablation}
  \fitwidth{%
  \begin{tabular}{@{}lccc@{}}
    \toprule
    \rowcolor{tblheadbg}
    Setting & Total $\uparrow$ & Manip. $\uparrow$ & LPIPS $\downarrow$ \\
    \midrule
    \rowcolor{tblstripe}
    One-shot distill & \plainscore{8.09}{0.23} & \plainscore{1.58}{0.09} & \plainscore{0.176}{0.014} \\
    \rowcolor{mydarkblue!6}
    Online iterative & \bestscore{8.52}{0.15} & \bestscore{1.72}{0.07} & \bestscore{0.173}{0.013} \\
    \bottomrule
  \end{tabular}
  }
  \vspace{-1.0em}
\end{wraptable}
Table~\ref{tab:compact_ablation} shows the effect of online iterative distillation under the same student architecture and evaluation protocol. Online iterative distillation improves the aggregate score from $8.09$ to $8.52$ and \textit{Manipulation Success} from $1.58$ to $1.72$, while slightly improving RT-1 LPIPS from $0.176$ to $0.173$. This supports the motivation in Section~\ref{sec:reward}: refreshing the student during RL helps maintain reward alignment under distribution shift.

\vspace{-8pt}
\section{Conclusion \& Limitation}
\vspace{-8pt}
We present \method{}, a framework that improves robot video world models through reward-aligned post-training and stabilized long-horizon decoding. By constructing a robot-centric benchmark to train a multimodal teacher judge and distilling it into a lightweight student reward model, \method{} enables efficient RL post-training that lifts the aggregate in-domain evaluation score by $+10.1\%$ over the strongest baseline under our current protocol, with the overall ranking further supported by external VLM-based validation and a small-scale blinded human study; the proposed sliding window re-encoding further yields $+2.8\%$ SSIM and $-9.8\%$ LPIPS gains at about $1\%$ additional latency. Our current evaluation is still limited to single-arm tabletop and countertop manipulation, and downstream policy improvement remains to be validated. Extending to diverse embodiments, ensembling multiple judges, and measuring control-level gains are promising directions for future work. Code, datasets, models, and video samples are available at the \href{https://roboalign-r1.netlify.app/#bibtex}{project website.}

\clearpage
\section*{Acknowledgements}
This work was supported by the National Natural Science Foundation of China (42125503, 42430602).

\bibliographystyle{unsrtnat}
\bibliography{neurips_2026}

@article{ha2018world,
  title={World models},
  author={Ha, David and Schmidhuber, J{\"u}rgen},
  journal={arXiv preprint arXiv:1803.10122},
  volume={2},
  number={3},
  pages={440},
  year={2018}
}

@article{mathieu2015deep,
  title={Deep multi-scale video prediction beyond mean square error},
  author={Mathieu, Michael and Couprie, Camille and LeCun, Yann},
  journal={arXiv preprint arXiv:1511.05440},
  year={2015}
}

@article{hansen2023td,
  title={Td-mpc2: Scalable, robust world models for continuous control},
  author={Hansen, Nicklas and Su, Hao and Wang, Xiaolong},
  journal={arXiv preprint arXiv:2310.16828},
  year={2023}
}

@article{zhou2024robodreamer,
  title={Robodreamer: Learning compositional world models for robot imagination},
  author={Zhou, Siyuan and Du, Yilun and Chen, Jiaben and Li, Yandong and Yeung, Dit-Yan and Gan, Chuang},
  journal={arXiv preprint arXiv:2404.12377},
  year={2024}
}

@article{wang2004image,
  title={Image quality assessment: from error visibility to structural similarity},
  author={Wang, Zhou and Bovik, Alan C and Sheikh, Hamid R and Simoncelli, Eero P},
  journal={IEEE transactions on image processing},
  volume={13},
  number={4},
  pages={600--612},
  year={2004},
  publisher={IEEE}
}

@article{prabhudesai2024video,
  title={Video diffusion alignment via reward gradients},
  author={Prabhudesai, Mihir and Mendonca, Russell and Qin, Zheyang and Fragkiadaki, Katerina and Pathak, Deepak},
  journal={arXiv preprint arXiv:2407.08737},
  year={2024}
}

@article{brooks2024video,
  title={Video generation models as world simulators},
  author={Brooks, Tim and Peebles, Bill and Holmes, Connor and DePue, Will and Guo, Yufei and Jing, Leo and Schnurr, David and Taylor, Joe and Luhman, Troy and Luhman, Eric and others},
  journal={OpenAI Blog},
  volume={1},
  number={8},
  pages={1},
  year={2024}
}

@article{zhang2025step,
  title={A step toward world models: A survey on robotic manipulation},
  author={Zhang, Peng-Fei and Cheng, Ying and Sun, Xiaofan and Wang, Shijie and Li, Fengling and Zhu, Lei and Shen, Heng Tao},
  journal={arXiv preprint arXiv:2511.02097},
  year={2025}
}

@inproceedings{bruce2024genie,
  title={Genie: Generative interactive environments},
  author={Bruce, Jake and Dennis, Michael D and Edwards, Ashley and Parker-Holder, Jack and Shi, Yuge and Hughes, Edward and Lai, Matthew and Mavalankar, Aditi and Steigerwald, Richie and Apps, Chris and others},
  booktitle={Forty-first International Conference on Machine Learning},
  year={2024}
}

@article{yang2023learning,
  title={Learning interactive real-world simulators},
  author={Yang, Mengjiao and Du, Yilun and Ghasemipour, Kamyar and Tompson, Jonathan and Schuurmans, Dale and Abbeel, Pieter},
  journal={arXiv preprint arXiv:2310.06114},
  volume={1},
  number={2},
  pages={6},
  year={2023}
}

@article{wu2024beamvq,
  title={BeamVQ: aligning space-time forecasting model via self-training on physics-aware metrics},
  author={Wu, Hao and Shi, Xingjian and Huang, Ziyue and Zhao, Penghao and Xiong, Wei and Xue, Jinbao and Tao, Yangyu and Huang, Xiaomeng and Wang, Weiyan},
  journal={arXiv preprint arXiv:2405.17051},
  year={2024}
}

@article{ravuri2021skilful,
  title={Skilful precipitation nowcasting using deep generative models of radar},
  author={Ravuri, Suman and Lenc, Karel and Willson, Matthew and Kangin, Dmitry and Lam, Remi and Mirowski, Piotr and Fitzsimons, Megan and Athanassiadou, Maria and Kashem, Sheleem and Madge, Sam and others},
  journal={Nature},
  volume={597},
  number={7878},
  pages={672--677},
  year={2021},
  publisher={Nature Publishing Group UK London}
}

@article{finn2016unsupervised,
  title={Unsupervised learning for physical interaction through video prediction},
  author={Finn, Chelsea and Goodfellow, Ian and Levine, Sergey},
  journal={Advances in neural information processing systems},
  volume={29},
  year={2016}
}

@inproceedings{zhang2018unreasonable,
  title={The unreasonable effectiveness of deep features as a perceptual metric},
  author={Zhang, Richard and Isola, Phillip and Efros, Alexei A and Shechtman, Eli and Wang, Oliver},
  booktitle={Proceedings of the IEEE conference on computer vision and pattern recognition},
  pages={586--595},
  year={2018}
}

@article{li2024evaluating,
  title={Evaluating real-world robot manipulation policies in simulation},
  author={Li, Xuanlin and Hsu, Kyle and Gu, Jiayuan and Pertsch, Karl and Mees, Oier and Walke, Homer Rich and Fu, Chuyuan and Lunawat, Ishikaa and Sieh, Isabel and Kirmani, Sean and others},
  journal={arXiv preprint arXiv:2405.05941},
  year={2024}
}

@article{mao2025robot,
  title={Robot Learning from a Physical World Model},
  author={Mao, Jiageng and He, Sicheng and Wu, Hao-Ning and You, Yang and Sun, Shuyang and Wang, Zhicheng and Bao, Yanan and Chen, Huizhong and Guibas, Leonidas and Guizilini, Vitor and others},
  journal={arXiv preprint arXiv:2511.07416},
  year={2025}
}

@article{ding2025understanding,
  title={Understanding world or predicting future? a comprehensive survey of world models},
  author={Ding, Jingtao and Zhang, Yunke and Shang, Yu and Zhang, Yuheng and Zong, Zefang and Feng, Jie and Yuan, Yuan and Su, Hongyuan and Li, Nian and Sukiennik, Nicholas and others},
  journal={ACM Computing Surveys},
  volume={58},
  number={3},
  pages={1--38},
  year={2025},
  publisher={ACM New York, NY}
}

@article{huang2025vid2world,
  title={Vid2world: Crafting video diffusion models to interactive world models},
  author={Huang, Siqiao and Wu, Jialong and Zhou, Qixing and Miao, Shangchen and Long, Mingsheng},
  journal={arXiv preprint arXiv:2505.14357},
  year={2025}
}

@article{xiong2026phycritic,
  title={PhyCritic: Multimodal Critic Models for Physical AI},
  author={Xiong, Tianyi and Wang, Shihao and Liu, Guilin and Dong, Yi and Li, Ming and Huang, Heng and Kautz, Jan and Yu, Zhiding},
  journal={arXiv preprint arXiv:2602.11124},
  year={2026}
}

@article{li2025worldmodelbench,
  title={Worldmodelbench: Judging video generation models as world models},
  author={Li, Dacheng and Fang, Yunhao and Chen, Yukang and Yang, Shuo and Cao, Shiyi and Wong, Justin and Luo, Michael and Wang, Xiaolong and Yin, Hongxu and Gonzalez, Joseph E and others},
  journal={arXiv preprint arXiv:2502.20694},
  year={2025}
}

@article{wu2025rlvr,
  title={Rlvr-world: Training world models with reinforcement learning},
  author={Wu, Jialong and Yin, Shaofeng and Feng, Ningya and Long, Mingsheng},
  journal={arXiv preprint arXiv:2505.13934},
  year={2025}
}

@article{wang2025beamvq,
  title={BeamVQ: Beam search with vector quantization to mitigate data scarcity in physical spatiotemporal forecasting},
  author={Wang, Weiyan and Shi, Xingjian and Shu, Ruiqi and Gao, Yuan and Chen, Rui Ray and Wang, Kun and Xu, Fan and Xue, Jinbao and Li, Shuaipeng and Tao, Yangyu and others},
  journal={arXiv preprint arXiv:2502.18925},
  year={2025}
}

@article{sfp2026,
  title={Spatiotemporal Forecasting as Planning: A Model-Based Reinforcement Learning Approach with Generative World Models},
  author={Anonymous},
  journal={Under review as a conference paper at ICLR 2026},
  year={2026}
}

@inproceedings{huang2024vbench,
  title={Vbench: Comprehensive benchmark suite for video generative models},
  author={Huang, Ziqi and He, Yinan and Yu, Jiashuo and Zhang, Fan and Si, Chenyang and Jiang, Yuming and Zhang, Yuanhan and Wu, Tianxing and Jin, Qingyang and Chanpaisit, Nattapol and others},
  booktitle={Proceedings of the IEEE/CVF Conference on Computer Vision and Pattern Recognition},
  pages={21807--21818},
  year={2024}
}

@inproceedings{he2024videoscore,
  title={Videoscore: Building automatic metrics to simulate fine-grained human feedback for video generation},
  author={He, Xuan and Jiang, Dongfu and Zhang, Ge and Ku, Max and Soni, Achint and Siu, Sherman and Chen, Haonan and Chandra, Abhranil and Jiang, Ziyan and Arulraj, Aaran and others},
  booktitle={Proceedings of the 2024 Conference on Empirical Methods in Natural Language Processing},
  pages={2105--2123},
  year={2024}
}

@article{xiao2023efficient,
  title={Efficient streaming language models with attention sinks},
  author={Xiao, Guangxuan and Tian, Yuandong and Chen, Beidi and Han, Song and Lewis, Mike},
  journal={arXiv preprint arXiv:2309.17453},
  year={2023}
}

@inproceedings{xiong2025llava,
  title={Llava-critic: Learning to evaluate multimodal models},
  author={Xiong, Tianyi and Wang, Xiyao and Guo, Dong and Ye, Qinghao and Fan, Haoqi and Gu, Quanquan and Huang, Heng and Li, Chunyuan},
  booktitle={Proceedings of the Computer Vision and Pattern Recognition Conference},
  pages={13618--13628},
  year={2025}
}

@article{yan2021videogpt,
  title={Videogpt: Video generation using vq-vae and transformers},
  author={Yan, Wilson and Zhang, Yunzhi and Abbeel, Pieter and Srinivas, Aravind},
  journal={arXiv preprint arXiv:2104.10157},
  year={2021}
}

@article{wu2024ivideogpt,
  title={ivideogpt: Interactive videogpts are scalable world models},
  author={Wu, Jialong and Yin, Shaofeng and Feng, Ningya and He, Xu and Li, Dong and Hao, Jianye and Long, Mingsheng},
  journal={Advances in Neural Information Processing Systems},
  volume={37},
  pages={68082--68119},
  year={2024}
}

@article{bai2025qwen25vl,
  title={Qwen2.5-{VL} Technical Report},
  author={Bai, Shuai and Chen, Keqin and Liu, Xuejing and Wang, Jialin and Ge, Wenbin and Song, Sibo and Dang, Kai and Wang, Peng and Wang, Shijie and Tang, Jun and others},
  journal={arXiv preprint arXiv:2502.13923},
  year={2025}
}

@article{shao2024deepseekmath,
  title={Deep{S}eek{M}ath: Pushing the limits of mathematical reasoning in open language models},
  author={Shao, Zhihong and Wang, Peiyi and Zhu, Qihao and Xu, Runxin and Song, Junxiao and Zhang, Mingchuan and Li, YK and Wu, Y and Guo, Daya},
  journal={arXiv preprint arXiv:2402.03300},
  year={2024}
}

@inproceedings{mentzer2024finite,
  title={Finite scalar quantization: {VQ-VAE} made simple},
  author={Mentzer, Fabian and Minnen, David and Agustsson, Eirikur and Toderici, George},
  booktitle={International Conference on Learning Representations},
  year={2024}
}

@article{brohan2023rt1,
  title={Rt-1: Robotics transformer for real-world control at scale},
  author={Brohan, Anthony and Brown, Noah and Carbajal, Justice and Chebotar, Yevgen and Dabis, Joseph and Finn, Chelsea and Gober, Keerthana and Hausman, Karol and Herzog, Alexander and Hsu, Jasmine and others},
  journal={arXiv preprint arXiv:2212.06817},
  year={2023}
}

@article{walke2023bridgedata,
  title={Bridgedata v2: A dataset for robot learning at scale},
  author={Walke, Homer and Black, Kevin and Zhao, Tony Z and Vuong, Quan and Zheng, Chongyi and Hansen-Estruch, Philippe and He, Andre Wang and Myers, Vivek and Kim, Moo Jin and Du, Max and others},
  journal={arXiv preprint arXiv:2308.12952},
  year={2023}
}

@inproceedings{mees2022calvin,
  title={{CALVIN}: A benchmark for language-conditioned policy learning for long-horizon robot manipulation tasks},
  author={Mees, Oier and Hermann, Lukas and Rosete-Beas, Erick and Burgard, Wolfram},
  booktitle={IEEE Robotics and Automation Letters},
  volume={7},
  number={3},
  pages={7327--7334},
  year={2022}
}

@inproceedings{liu2024libero,
  title={{LIBERO}: Benchmarking knowledge transfer for lifelong robot learning},
  author={Liu, Bo and Zhu, Yifeng and Gao, Chongkai and Feng, Yizhou and Liu, Qiang and Zhu, Yuke and Stone, Peter},
  booktitle={Advances in Neural Information Processing Systems},
  volume={36},
  year={2024}
}

@article{wang2022predrnn,
  title={Predrnn: A recurrent neural network for spatiotemporal predictive learning},
  author={Wang, Yunbo and Wu, Haixu and Zhang, Jianjin and Gao, Zhifeng and Wang, Jianmin and Yu, Philip S and Long, Mingsheng},
  journal={IEEE Transactions on Pattern Analysis and Machine Intelligence},
  volume={45},
  number={2},
  pages={2208--2225},
  year={2022},
  publisher={IEEE}
}

@article{tan2026towards,
  title={Towards Generalist Embodied AI: A Survey on World Models for VLA Agents},
  author={Tan, Wentao and Zhu, Lei and Wang, Bowen and Xie, Enci and Ji, Baixu and Lin, Zengrong and Yang, Wenjie and Li, Jingjing and Shen, Heng Tao},
  journal={Authorea Preprints},
  year={2026},
  publisher={Authorea}
}

@inproceedings{ni2025maskgwm,
  title={Maskgwm: A generalizable driving world model with video mask reconstruction},
  author={Ni, Jingcheng and Guo, Yuxin and Liu, Yichen and Chen, Rui and Lu, Lewei and Wu, Zehuan},
  booktitle={Proceedings of the Computer Vision and Pattern Recognition Conference},
  pages={22381--22391},
  year={2025}
}

@article{hu2022lora,
  title={Lora: Low-rank adaptation of large language models.},
  author={Hu, Edward J and Shen, Yelong and Wallis, Phillip and Allen-Zhu, Zeyuan and Li, Yuanzhi and Wang, Shean and Wang, Liang and Chen, Weizhu and others},
  journal={Iclr},
  volume={1},
  number={2},
  pages={3},
  year={2022}
}

@article{barcellona2024dream,
  title={Dream to manipulate: Compositional world models empowering robot imitation learning with imagination},
  author={Barcellona, Leonardo and Zadaianchuk, Andrii and Allegro, Davide and Papa, Samuele and Ghidoni, Stefano and Gavves, Efstratios},
  journal={arXiv preprint arXiv:2412.14957},
  year={2024}
}

@inproceedings{pang2025learning,
  title={Learning view-invariant world models for visual robotic manipulation},
  author={Pang, Jing-Cheng and Tang, Nan and Li, Kaiyuan and Tang, Yuting and Cai, Xin-Qiang and Zhang, Zhen-Yu and Niu, Gang and Sugiyama, Masashi and Yu, Yang},
  booktitle={The Thirteenth International Conference on Learning Representations},
  year={2025}
}

\newpage
\appendix

\section*{Appendix Contents}
\begingroup
\footnotesize
\setlength{\parindent}{0pt}
\newcommand{\apptocline}[3][0em]{%
  \noindent\hspace*{#1}\hyperref[#2]{\ref*{#2}\hspace{0.6em}#3}%
  \nobreak\leaders\hbox{$\mkern2mu.\mkern2mu$}\hfill\nobreak\pageref*{#2}\par}
\newcommand{\appsec}[2]{\vspace{0.2em}\apptocline{#1}{\textbf{#2}}}
\newcommand{\appsub}[2]{\apptocline[1.6em]{#1}{#2}}

\appsec{app:related_work}{Extended Related Work}

\appsec{app:vwm_training}{Training Details of the Video World Model}
\appsub{app:vwm_notation}{Notation}
\appsub{app:vwm_stage1}{Stage 1: Context-Aware Compression Tokenizer}
\appsub{app:vwm_stage2}{Stage 2: Autoregressive Transformer World Model}
\appsub{app:vwm_dataset}{Dataset: Bridge V2 Instantiation}
\appsub{app:vwm_repro}{Reproduction Protocol}
\appsub{app:vwm_infra}{Optimization Infrastructure}
\appsub{app:vwm_summary}{End-to-End Quantitative Summary}
\appsub{app:vwm_tb}{TensorBoard Training Curves}

\appsec{app:proof_swr}{Stylized derivation for Proposition~\ref{prop:swr}}
\appsec{app:code_train}{Pseudocode for Reward-Aligned Post-Training}
\appsec{app:student_reward}{Student Reward Model Details}
\appsec{app:code}{Pseudocode and Implementation of Sliding Window Re-encoding}
\appsec{app:swr_window_ablate}{SWR window-size ablation}
\appsec{app:swr_case_figs}{SWR qualitative visualizations}

\appsec{app:eval_metrics}{Evaluation Metrics}
\appsub{app:metric_semantic}{Semantic and Physical Alignment Metrics}
\appsub{app:metric_pixel}{Pixel-Level Reconstruction Metrics}
\appsub{app:roi_metrics}{Motion-Mask-Based ROI Metrics}

\appsec{app:roboalign_judge}{\textsc{RoboAlign-Judge}: Multimodal Teacher Judge Training Details}
\appsub{app:judge_motivation}{Motivation and Design Rationale}
\appsub{app:judge_data}{Training Data Construction}
\appsub{app:judge_rubric}{Six-Dimension Scoring Rubric}
\appsub{app:judge_training}{Model Architecture and LoRA Fine-Tuning}
\appsub{app:judge_cot}{Chain-of-Thought Reasoning}
\appsub{app:judge_to_reward}{From Teacher Judge to Reward Signal}
\appsub{app:judge_vlm_comparison}{Comparison with Alternative VLM Judges}
\appsub{app:human_eval_template}{Small-scale Blinded Human Evaluation}
\appsub{app:judge_comparison}{Comparison with Alternative Judge Backbones}
\appsub{app:judge_limitations}{Limitations and Future Directions}
\appsub{app:judge_prompt}{Judge Prompt Templates}
\endgroup
\clearpage


\section{Extended Related Work}
\label{app:related_work}

\paragraph{Robot Video World Models.}
World models provide learnable internal simulators for planning, policy evaluation, and long-horizon decision making by predicting how environments evolve under candidate actions \cite{ha2018world,hansen2023td}. With the rise of video generation and embodied learning, recent work has increasingly explored video-based world models for robotics, including approaches that model manipulation dynamics, interaction outcomes, and future observations through video prediction or generative simulation \cite{brooks2024video,bruce2024genie,yang2023learning,zhou2024robodreamer,huang2025vid2world,mao2025robot}. These studies highlight the promise of video world models for embodied intelligence, but most still emphasize generation quality, dynamics modeling, or representation learning, and are largely trained with maximum-likelihood or reconstruction-style objectives \cite{zhang2025step,ding2025understanding}. In contrast, our work focuses on reward-aligned post-training for robot video world models.

\paragraph{RL Post-Training for World Models.}
Recent work has begun to explore reinforcement learning as a post-training mechanism for world models, showing that world-model behavior can be further aligned with downstream objectives beyond standard supervised fitting \cite{wu2025rlvr,prabhudesai2024video,sfp2026}. However, existing approaches still rely mainly on low-level or verifiable rewards, such as pixel reconstruction error or perceptual similarity metrics \cite{wu2025rlvr,mathieu2015deep,wu2024beamvq,zhang2018unreasonable,wang2004image,wang2025beamvq}. While these rewards are stable and easy to compute, they remain weak proxies for the properties that matter in robot prediction, such as correct instruction execution, physically plausible contact dynamics, and consistent long-horizon action outcomes. Similar limitations have also been observed in physical spatiotemporal forecasting~\cite{ravuri2021skilful,finn2016unsupervised}, where optimizing average reconstruction objectives often leads to oversmoothed predictions and poor coverage of rare but decision-critical events \cite{wang2025beamvq,sfp2026}. Long-horizon rollout quality is also shaped by inference-time generation strategy: in autoregressive token-based world models~\cite{yan2021videogpt,xiao2023efficient,wu2024ivideogpt,wang2022predrnn}, prediction errors can accumulate over time and progressively degrade later predictions. Our work therefore complements reward-aligned post-training with a simple decoding strategy for stabilizing long-horizon rollouts.

\paragraph{Multimodal Judges and Rewards.}
Another relevant line of research studies multimodal evaluators and reward models for video understanding and generation \cite{huang2024vbench,he2024videoscore,xiong2025llava}. Prior work shows that multimodal judges can provide richer supervision than low-level visual metrics by explicitly assessing instruction following, physical plausibility, and higher-level video consistency \cite{li2025worldmodelbench,xiong2026phycritic,huang2024vbench,he2024videoscore}. Such evaluators are much better aligned with human judgment, but their inference cost and latency make them difficult to use directly as online rewards in reinforcement learning \cite{li2025worldmodelbench,he2024videoscore,xiong2026phycritic}. Our method is closely related in spirit to this line of work, but differs in two important ways: we focus on robot video world models rather than general video evaluation, and we distill high-capacity multimodal judgment into a lightweight reward model that is practical for online RL. This makes multimodal reward supervision not only evaluative, but directly usable for post-training robot world models.

\clearpage
\section{Training Details of the Video World Model}\label{app:vwm_training}

This appendix details the implementation of the two-stage video world model used as the backbone throughout our experiments.
Stage~1 trains a context-aware compression tokenizer that maps raw robot video clips into discrete index sequences; on top of a frozen Stage~1 tokenizer, Stage~2 trains an autoregressive Transformer that jointly models visual and action tokens.
The two stages share the same clip length, spatial resolution, and data pipeline, and communicate exclusively through discrete tokens.

\subsection{Notation}\label{app:vwm_notation}

Unless stated otherwise, the symbols in Table~\ref{tab:vwm_notation} are shared between both stages.

\begin{table}[h]
  \centering
  \scriptsize
  \setlength{\tabcolsep}{6pt}
  \renewcommand{\arraystretch}{1.10}
  \caption{Symbols shared between Stage~1 and Stage~2.}
  \label{tab:vwm_notation}
  \fitwidth{%
  \begin{tabular}{@{}llc@{}}
    \toprule
    \rowcolor{tblheadbg}
    Symbol & Meaning & Value \\
    \midrule
    \rowcolor{tblstripe}
    $T$ & Clip length (frames) & $8$ \\
    $t_c$ & Context frames & $1$ \\
    \rowcolor{tblstripe}
    $t_d = T - t_c$ & Dynamic frames & $7$ \\
    $H \times W$ & Spatial resolution & $256 \times 320$ \\
    \rowcolor{tblstripe}
    $p$ & Dynamic-branch patch size & $4$ \\
    $D$ & FSQ quantization dimension & $5$ \\
    \rowcolor{tblstripe}
    $\mathbf{L}$ & FSQ level vector & $(7,5,5,5,5)$ \\
    $K = \prod_i L_i$ & Codebook size & $4375$ \\
    \rowcolor{tblstripe}
    $N_c$ & Context tokens per clip & $1280$ \\
    $N_d$ & Dynamic tokens per frame & $80$ \\
    \rowcolor{tblstripe}
    $D_a$ & Action-vector dimensionality (Bridge~V2) & $13$ \\
    $B_a$ & Per-dimension action bins & $256$ \\
    \rowcolor{tblstripe}
    $V$ & Joint vocabulary size (Stage~2) & $9008$ \\
    $S$ & Transformer sequence length (Stage~2) & $1931$ \\
    \rowcolor{tblstripe}
    $\phi_l$ & VGG-16 feature map at layer $l$ (LPIPS) & --- \\
    \bottomrule
  \end{tabular}}
\end{table}

\subsection{Stage 1: Context-Aware Compression Tokenizer}\label{app:vwm_stage1}

Stage~1 compresses a clip $\{x_t\}_{t=1}^{T}\!\in\!\mathbb{R}^{T\times 3\times H\times W}$ into $1840$ discrete indices using two asymmetric branches: a \emph{Full VAE} over the context frame $x_c\!:=\!x_1$ and a \emph{Conditional VAE} over the dynamic frames $x_d\!:=\!\{x_t\}_{t=2}^{T}$. Quantization is carried out by Finite Scalar Quantization (FSQ); no learnable codebook is used.

\paragraph{Architecture.}
The \emph{context branch} uses a convolutional encoder $E_c$ that produces a latent feature map together with an intermediate feature pyramid $\{f_l^{(E)}\}_{l=1}^{L}$:
\[
\bigl(h,\;\{f_l^{(E)}\}_{l=1}^{L}\bigr) = E_c(x_c),\qquad h\in\mathbb{R}^{256\times 32\times 40}.
\]
A $1{\times}1$ convolution projects $h$ into the quantization space, $z_c = \mathrm{QuantConv}(h)\in\mathbb{R}^{5\times 32\times 40}$. No spatial patchification is applied, so this branch emits $N_c = 32\cdot 40 = 1280$ tokens, each covering an $8{\times}8$ pixel block of the input.

The \emph{dynamic branch} uses a conditional encoder $E_d$ whose downsampling stages are interleaved with cross-attention layers that inject the context feature pyramid,
\[
s_{l+1} = \mathrm{CrossAttn}\!\bigl(\mathrm{DownBlock}_l(s_l),\; f_{l+1}^{(E)}\bigr),
\]
so that $E_d$ encodes only the residual dynamics relative to the static context. Its output $d'\!\in\!\mathbb{R}^{64\times 32\times 40}$ is patchified with $p=4$ and linearly projected,
\[
d'' = \mathrm{Patchify}_p(d')\in\mathbb{R}^{80\times 1024},\qquad
z_d = W_q\,d'' + b_q \in \mathbb{R}^{80\times 5},\quad W_q\in\mathbb{R}^{5\times 1024},
\]
producing $N_d = 80$ tokens per dynamic frame, each covering a $32{\times}32$ pixel block.

\paragraph{FSQ quantization.}
For a vector $z\in\mathbb{R}^{D}$ with $D=5$ and $\mathbf{L}=(7,5,5,5,5)$, FSQ performs per-dimension bounding, straight-through rounding, and mixed-radix encoding:
\[
\tilde z_i = \tfrac{L_i}{2}\tanh(z_i + s_i) - o_i,\qquad
\hat z_i = \mathrm{round}(\tilde z_i) + \bigl(\tilde z_i - \mathrm{sg}[\mathrm{round}(\tilde z_i)]\bigr),
\]
\[
\mathrm{idx}(\hat z) \;=\; \sum_{i=1}^{D}\Bigl(\hat z_i+\lfloor L_i/2\rfloor\Bigr)\!\prod_{j<i}L_j \;\in\; \{0,\ldots,K-1\},\qquad K=\prod_i L_i = 4375.
\]
Because the index set is analytically fixed by $\mathbf{L}$ and every codeword is reachable by construction, no commitment loss is required.

\paragraph{Decoding.}
Symmetric decoders $D_c$ and $D_d$ reconstruct the clip,
\[
\hat x_c = D_c\!\bigl(\mathrm{PostQuantConv}(\mathcal{Q}(z_c))\bigr),\qquad
\hat x_d = D_d\!\bigl(\mathrm{UnPatchify}(\mathrm{PostQuantLinear}(\mathcal{Q}(z_d)));\;\{f_l^{(D)}\}\bigr),
\]
with cross-attention inserted at every upsampling stage of $D_d$ to read the decoder-side feature pyramid $\{f_l^{(D)}\}$ of $D_c$. \emph{No skip connection bypasses the quantization bottleneck}: the decoders depend solely on the discrete indices. Each clip therefore carries $N_c + t_d\cdot N_d = 1280 + 7\cdot 80 = 1840$ tokens, corresponding to a compression ratio of roughly $1070\times$ relative to the $8\!\cdot\!3\!\cdot\!256\!\cdot\!320 = 1{,}966{,}080$ raw pixel values.

\paragraph{End-to-end data flow.}
Figure~\ref{fig:vwm_stage1_dataflow} summarizes the Stage~1 forward pipeline, making explicit how the two branches share the context feature pyramid yet quantize independently through FSQ, and how the decoders reconstruct the clip from discrete indices alone.

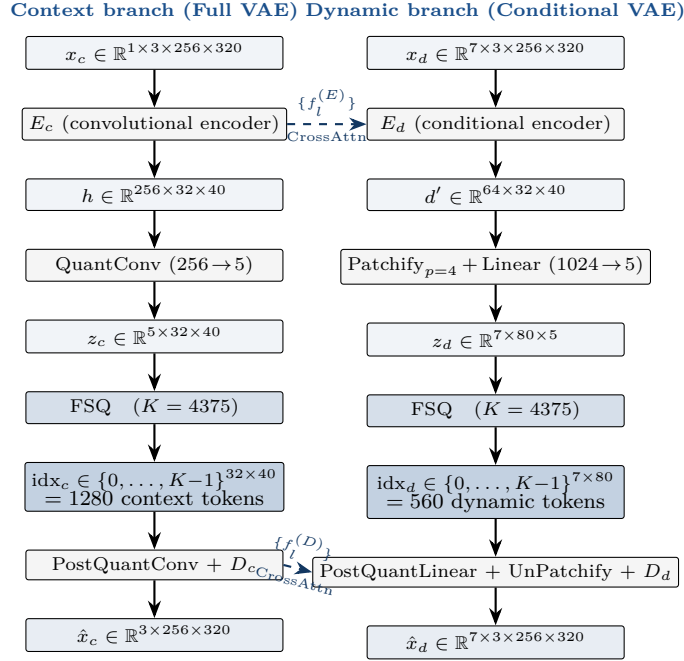
\begin{figure}[h]
  \centering
  \scriptsize
  \begin{tikzpicture}[
      node distance=5mm and 11mm,
      font=\scriptsize,
      tensor/.style={draw, rounded corners=1pt, fill=mydarkblue!6, inner sep=2.5pt, align=center, minimum width=34mm},
      op/.style={draw, rounded corners=1pt, fill=gray!8, inner sep=2.5pt, align=center, minimum width=34mm},
      fsq/.style={draw, rounded corners=1pt, fill=mydarkblue!18, inner sep=2.5pt, align=center, minimum width=34mm},
      tok/.style={draw, rounded corners=1pt, fill=mydarkblue!26, inner sep=2.5pt, align=center, minimum width=34mm},
      outbox/.style={draw, rounded corners=1pt, fill=mydarkblue!10, inner sep=2.5pt, align=center, minimum width=34mm},
      arr/.style={-Stealth, thick},
      cross/.style={-Stealth, thick, dashed, mydarkblue!70!black},
    ]
    \node[tensor] (xc) {$x_c \in \mathbb{R}^{1\times 3\times 256\times 320}$};
    \node[op, below=of xc] (Ec) {$E_c$ (convolutional encoder)};
    \node[tensor, below=of Ec] (h) {$h \in \mathbb{R}^{256\times 32\times 40}$};
    \node[op, below=of h] (qconv) {$\mathrm{QuantConv}$ ($256\!\to\!5$)};
    \node[tensor, below=of qconv] (zc) {$z_c \in \mathbb{R}^{5\times 32\times 40}$};
    \node[fsq, below=of zc] (fsqc) {FSQ\quad($K=4375$)};
    \node[tok, below=of fsqc] (idxc) {$\mathrm{idx}_c \in \{0,\ldots,K{-}1\}^{32\times 40}$\\[-1pt]\footnotesize $= 1280$ context tokens};
    \node[op, below=of idxc] (dc) {PostQuantConv $+$ $D_c$};
    \node[outbox, below=of dc] (xchat) {$\hat{x}_c \in \mathbb{R}^{3\times 256\times 320}$};

    \node[tensor, right=of xc] (xd) {$x_d \in \mathbb{R}^{7\times 3\times 256\times 320}$};
    \node[op, below=of xd] (Ed) {$E_d$ (conditional encoder)};
    \node[tensor, below=of Ed] (dp) {$d' \in \mathbb{R}^{64\times 32\times 40}$};
    \node[op, below=of dp] (patch) {$\mathrm{Patchify}_{p=4}\,+\,$Linear ($1024\!\to\!5$)};
    \node[tensor, below=of patch] (zd) {$z_d \in \mathbb{R}^{7\times 80\times 5}$};
    \node[fsq, below=of zd] (fsqd) {FSQ\quad($K=4375$)};
    \node[tok, below=of fsqd] (idxd) {$\mathrm{idx}_d \in \{0,\ldots,K{-}1\}^{7\times 80}$\\[-1pt]\footnotesize $= 560$ dynamic tokens};
    \node[op, below=of idxd] (dd) {PostQuantLinear $+$ UnPatchify $+$ $D_d$};
    \node[outbox, below=of dd] (xdhat) {$\hat{x}_d \in \mathbb{R}^{7\times 3\times 256\times 320}$};

    \foreach \a/\b in {xc/Ec, Ec/h, h/qconv, qconv/zc, zc/fsqc, fsqc/idxc, idxc/dc, dc/xchat}{
      \draw[arr] (\a) -- (\b);
    }
    \foreach \a/\b in {xd/Ed, Ed/dp, dp/patch, patch/zd, zd/fsqd, fsqd/idxd, idxd/dd, dd/xdhat}{
      \draw[arr] (\a) -- (\b);
    }

    \draw[cross] (Ec.east) -- node[above, sloped, font=\tiny]{$\{f_l^{(E)}\}$} node[below, sloped, font=\tiny]{CrossAttn} (Ed.west);
    \draw[cross] (dc.east) -- node[above, sloped, font=\tiny]{$\{f_l^{(D)}\}$} node[below, sloped, font=\tiny]{CrossAttn} (dd.west);

    \node[above=1mm of xc, font=\scriptsize\bfseries, text=mydarkblue] {Context branch (Full VAE)};
    \node[above=1mm of xd, font=\scriptsize\bfseries, text=mydarkblue] {Dynamic branch (Conditional VAE)};
  \end{tikzpicture}
  \caption{\textbf{End-to-end data flow of the Stage~1 tokenizer.}
  The context frame $x_c$ passes through a Full VAE and is quantized into $1280$ context tokens at $8{\times}8$ granularity; the $7$ dynamic frames $x_d$ pass through a Conditional VAE whose encoder and decoder are cross-attention-conditioned on the context feature pyramids $\{f_l^{(E)}\}$ and $\{f_l^{(D)}\}$ (dashed arrows), producing $560$ dynamic tokens at $32{\times}32$ granularity. FSQ uses a shared codebook of size $K=4{,}375$ without any commitment loss, and both decoders depend solely on the discrete indices -- no skip connection bypasses the quantization bottleneck.}
  \label{fig:vwm_stage1_dataflow}
\end{figure}

\paragraph{Training objective.}
Let $\hat x := [\hat x_c;\,\hat x_d]$. The generator loss combines reconstruction, perceptual, and adversarial terms,
\[
\mathcal{L}_G \;=\; \lambda_r\bigl(\mathcal{L}_{\text{rec}}^{d} + \mathcal{L}_{\text{rec}}^{c}\bigr) + \lambda_p\bigl(\mathcal{L}_{\text{perc}}^{d} + \mathcal{L}_{\text{perc}}^{c}\bigr) + \lambda_a\cdot\mathbb{1}[t\ge T_{\mathrm{adv}}]\cdot\mathcal{L}_{\text{adv}}^{G},
\]
with
\[
\mathcal{L}_{\text{rec}}^{d} = \tfrac{1}{t_d\,3HW}\textstyle\sum_{k=1}^{t_d}\|x_d^{(k)}-\hat x_d^{(k)}\|_1,\quad
\mathcal{L}_{\text{rec}}^{c} = \tfrac{1}{3HW}\|x_c-\hat x_c\|_1,
\]
\[
\mathcal{L}_{\text{perc}}^{\bullet} = \textstyle\sum_{l}w_l\|\phi_l(\tilde x^{\bullet})-\phi_l(\tilde{\hat x}^{\bullet})\|_2^2,\quad \tilde x = 2x-1,\qquad
\mathcal{L}_{\text{adv}}^{G} = -\mathbb{E}\bigl[D_\psi(\hat x)\bigr],
\]
where $\phi_l$ denotes ImageNet-pretrained VGG-16 features with LPIPS linear weights $\{w_l\}$. The discriminator $D_\psi$ is a depth-$6$ PatchGAN trained with the hinge objective and an $R_1$-style gradient penalty,
\[
\mathcal{L}_D = \mathbb{E}\bigl[\max(0,\,1+D_\psi(\hat x))\bigr] + \mathbb{E}\bigl[\max(0,\,1-D_\psi(x))\bigr] + \gamma\,\mathbb{E}_x\!\bigl[(\|\nabla_x D_\psi(x)\|_2 - 1)^2\bigr].
\]
The adversarial term is disabled for the first $T_{\mathrm{adv}}=10{,}000$ steps; thereafter $G$ and $D$ are updated alternately under a shared global-step counter.

\paragraph{Data and augmentation.}
Training follows the Open~X-Embodiment (OXE) mixed-dataset protocol, with mixing weights proportional to the sample counts of each sub-dataset; the numbers reported in this appendix are obtained on the Bridge~V2 subset for reproducibility. For each trajectory we sample a start index uniformly and extract a clip of length $T=8$ with stride~$1$. Photometric and geometric augmentations are sampled once per clip and shared across all $T$ frames (Table~\ref{tab:vwm_augment}).

\begin{table}[h]
  \centering
  \scriptsize
  \setlength{\tabcolsep}{6pt}
  \renewcommand{\arraystretch}{1.10}
  \caption{Clip-wise augmentation ranges used by the Stage~1 tokenizer.}
  \label{tab:vwm_augment}
  \fitwidth{%
  \begin{tabular}{@{}ll@{}}
    \toprule
    \rowcolor{tblheadbg}
    Augmentation & Range \\
    \midrule
    \rowcolor{tblstripe}
    Brightness / contrast / saturation & $[0.9,\,1.1]$ \\
    Hue & $[-0.05,\,0.05]$ \\
    \rowcolor{tblstripe}
    \texttt{RandomResizedCrop} scale & $[0.8,\,1.0]$ \\
    \texttt{RandomResizedCrop} aspect ratio & $[1.0,\,1.4]$ \\
    \bottomrule
  \end{tabular}}
\end{table}

\paragraph{Hyperparameters.}
Table~\ref{tab:vwm_stage1_hparams} lists the Stage~1 hyperparameters. Generator components, discriminator components, and validation-set metrics are logged as separate TensorBoard scalar groups. Validation is run every $1{,}000$ steps over $100$ batches, reporting mean L1 and LPIPS together with a few reconstruction visualizations.

\begin{table}[h]
  \centering
  \scriptsize
  \setlength{\tabcolsep}{6pt}
  \renewcommand{\arraystretch}{1.10}
  \caption{Stage~1 tokenizer hyperparameters.}
  \label{tab:vwm_stage1_hparams}
  \fitwidth{%
  \begin{tabular}{@{}ll@{}}
    \toprule
    \rowcolor{tblheadbg}
    Item & Value \\
    \midrule
    \rowcolor{tblstripe}
    Resolution $H\times W$ & $256\times 320$ \\
    Clip length $T$ / context $t_c$ & $8\,/\,1$ \\
    \rowcolor{tblstripe}
    Context latent channels $C_h$ & $256$ \\
    Dynamic latent channels & $64$ \\
    \rowcolor{tblstripe}
    Patch size $p$ & $4$ \\
    FSQ $(D,\mathbf{L},K)$ & $(5,\,(7,5,5,5,5),\,4375)$ \\
    \rowcolor{tblstripe}
    Tokens per clip & $1280 + 7\times 80 = 1840$ \\
    Discriminator depth & $6$ \\
    \rowcolor{tblstripe}
    Reconstruction loss & L1 \\
    Loss weights $(\lambda_r,\lambda_p,\lambda_a)$ & $(1.0,\,1.0,\,0.1)$ \\
    \rowcolor{tblstripe}
    Gradient-penalty coefficient $\gamma$ & $10$ \\
    Adversarial start step $T_{\mathrm{adv}}$ & $10{,}000$ \\
    \rowcolor{tblstripe}
    Optimizer & AdamW, $(\beta_1,\beta_2,\varepsilon)\!=\!(0.9,0.999,10^{-8})$, weight decay $0$ \\
    Learning rate ($G$ / $D$) & $5\times 10^{-4}$ / $5\times 10^{-4}$ \\
    \rowcolor{tblstripe}
    LR schedule ($G$ / $D$) & cosine / constant-with-warmup, warmup $5{,}000$ steps \\
    Gradient clipping & $\|\nabla\|_2 \le 1.0$ \\
    \rowcolor{tblstripe}
    Mixed precision & bf16 \\
    Gradient checkpointing & enabled \\
    \rowcolor{tblstripe}
    Per-GPU batch / accumulation & $1\,/\,4$ \\
    GPUs & $8\times$ A100-40GB \\
    \rowcolor{tblstripe}
    Global effective batch & $32$ \\
    Total training steps & $60{,}000$ \\
    \rowcolor{tblstripe}
    Checkpoint / validation frequency & every $2{,}000$ / $1{,}000$ steps \\
    \bottomrule
  \end{tabular}}
\end{table}

\subsection{Stage 2: Autoregressive Transformer World Model}\label{app:vwm_stage2}

Stage~2 trains a causal Transformer $\pi_\theta$ on top of the \emph{frozen} Stage~1 tokenizer. With $\mathbf{c}\!\in\!\{K,\ldots,2K-1\}^{N_c}$ the context tokens, $\mathbf{d}_t\!\in\!\{0,\ldots,K-1\}^{N_d}$ the dynamic tokens of frame $t$, and $\mathbf{a}_t\!\in\!\{2K,\ldots,2K+B_a-1\}^{D_a}$ the action tokens at step $t$, the model factorizes
\[
p_\theta\bigl(\mathbf{d}_2,\mathbf{d}_3,\ldots,\mathbf{d}_T\;\big|\;\mathbf{c},\,\mathbf{d}_1,\,\mathbf{a}_1,\ldots,\mathbf{a}_{T-1}\bigr).
\]

\paragraph{Joint vocabulary.}
Visual and action tokens are merged into a single vocabulary of size $V$ via fixed offsets, so that a standard LM head can emit either modality (Table~\ref{tab:vwm_vocab}). The two visual segments share the same FSQ codebook but are disambiguated by an offset of $K$, allowing the embedding matrix $W_E\!\in\!\mathbb{R}^{V\times d}$ to learn separate embeddings for context and dynamic semantics.

\begin{table}[h]
  \centering
  \scriptsize
  \setlength{\tabcolsep}{6pt}
  \renewcommand{\arraystretch}{1.10}
  \caption{Stage~2 joint vocabulary.}
  \label{tab:vwm_vocab}
  \fitwidth{%
  \begin{tabular}{@{}llc@{}}
    \toprule
    \rowcolor{tblheadbg}
    Range & Semantics & Size \\
    \midrule
    \rowcolor{tblstripe}
    $[0,\,K)$ & Dynamic visual tokens (FSQ indices of $E_d$) & $K = 4375$ \\
    $[K,\,2K)$ & Context visual tokens (FSQ indices of $E_c$, offset by $K$) & $K = 4375$ \\
    \rowcolor{tblstripe}
    $[2K,\,2K+B_a)$ & Action tokens & $B_a = 256$ \\
    $\{2K+B_a\}$ & BOS (reserved) & $1$ \\
    \rowcolor{tblstripe}
    $\{2K+B_a+1\}$ & EOS (reserved) & $1$ \\
    \midrule
    \rowcolor{mydarkblue!6}
    \multicolumn{2}{l}{\textbf{Total $V$}} & $\mathbf{9008}$ \\
    \bottomrule
  \end{tabular}}
\end{table}

\paragraph{Action discretization.}
Each dimension of $\mathbf{a}_t\!\in\!\mathbb{R}^{D_a}$ is independently mapped into $B_a=256$ equal-width bins. The per-dimension lower and upper bounds $(a_i^{\min},\,a_i^{\max})$ are pre-computed on the training split and persisted as an offline lookup table shared between training and inference:
\[
\mathrm{bin}_i(a_{t,i}) = \Bigl\lfloor B_a\cdot\frac{a_{t,i}-a_i^{\min}}{a_i^{\max}-a_i^{\min}+\varepsilon}\Bigr\rfloor\in\{0,\ldots,B_a-1\},\qquad
a^{\text{tok}}_{t,i} = \mathrm{bin}_i(a_{t,i}) + 2K.
\]

\paragraph{Sequence layout and loss masking.}
The data module concatenates the segments into
\[
\mathbf{x} = \bigl[\underbrace{\mathbf{c}}_{1280}\;\|\;
\underbrace{\mathbf{d}_1\|\mathbf{a}_1}_{80+13}\;\|\;
\underbrace{\mathbf{d}_2\|\mathbf{a}_2}_{80+13}\;\|\;\cdots\;\|\;
\underbrace{\mathbf{d}_{T}\|\mathbf{a}_{T}}_{80+13}\bigr]\in\{0,\ldots,V-1\}^{S},
\]
where $\mathbf{a}_T$ is filled with the action associated with the last transition. The total sequence length is
\[
S = N_c + t_d\cdot(N_d + D_a) = 1280 + 7\cdot(80+13) = 1931.
\]
The label sequence has the same shape as $\mathbf{x}$; all positions except the dynamic tokens of frames $2,\ldots,T$ are set to the ignore index and excluded from the cross-entropy loss (Table~\ref{tab:vwm_label_mask}). The number of supervised positions per sample is
\[
|\mathcal{M}| = (t_d - 1)\cdot N_d = 6\cdot 80 = 480.
\]
Since a single causal forward pass simultaneously supervises the predictions of frames $2,\ldots,T$, this objective is equivalent to $t_d-1=6$ one-step next-frame prediction subtasks that share attention compute, i.e.\ multi-step prediction (MSP).

\begin{table}[h]
  \centering
  \scriptsize
  \setlength{\tabcolsep}{6pt}
  \renewcommand{\arraystretch}{1.10}
  \caption{Loss-mask layout of a single Stage~2 sample.}
  \label{tab:vwm_label_mask}
  \fitwidth{%
  \begin{tabular}{@{}ll@{}}
    \toprule
    \rowcolor{tblheadbg}
    Segment & Label \\
    \midrule
    \rowcolor{tblstripe}
    Context tokens ($N_c=1280$) & $-100$ (condition) \\
    Dynamic tokens at frame~1 ($N_d=80$) & $-100$ (initial state) \\
    \rowcolor{tblstripe}
    Dynamic tokens at frames $2{:}T$ ($6\cdot 80=480$) & ground-truth token ids \\
    All action tokens ($7\cdot 13=91$) & $-100$ (condition) \\
    \bottomrule
  \end{tabular}}
\end{table}

\paragraph{Backbone.}
$\pi_\theta$ is a causal LLaMA configured as in Table~\ref{tab:vwm_backbone}. The causal mask is realized inside the FlashAttention-2 kernel; weights, activations, and gradients are all stored in bf16, since FlashAttention-2 is incompatible with fp32.

\begin{table}[h]
  \centering
  \scriptsize
  \setlength{\tabcolsep}{6pt}
  \renewcommand{\arraystretch}{1.10}
  \caption{Stage~2 Transformer backbone.}
  \label{tab:vwm_backbone}
  \fitwidth{%
  \begin{tabular}{@{}ll@{}}
    \toprule
    \rowcolor{tblheadbg}
    Component & Value \\
    \midrule
    \rowcolor{tblstripe}
    Layers & $12$ \\
    Hidden size $d$ & $768$ \\
    \rowcolor{tblstripe}
    Attention heads & $12$ (no GQA) \\
    FFN size & $3072$ \\
    \rowcolor{tblstripe}
    Activation & SwiGLU \\
    Normalization & RMSNorm, $\varepsilon=10^{-6}$ \\
    \rowcolor{tblstripe}
    Positional encoding & RoPE \\
    Max positions & $8192$ \\
    \rowcolor{tblstripe}
    Tied word embeddings & disabled \\
    Parameters (incl.\ embeddings) & ${\approx}1.38\times 10^{8}$ \\
    \rowcolor{tblstripe}
    Attention kernel & FlashAttention-2 \\
    Compute precision & bf16 \\
    \bottomrule
  \end{tabular}}
\end{table}

\paragraph{Training objective.}
Let $h_{i-1}\!\in\!\mathbb{R}^{d}$ denote the last-layer hidden state at position $i{-}1$ and $W_O\!\in\!\mathbb{R}^{d\times V}$ the LM head, so that
\[
p_\theta(x_i\mid x_{<i}) = \bigl[\mathrm{softmax}(W_O\,h_{i-1})\bigr]_{x_i}.
\]
The training loss is the causal cross-entropy restricted to the supervised positions~$\mathcal{M}$,
\[
\mathcal{L}_{\mathrm{VGPT}}(\theta) = -\frac{1}{|\mathcal{M}|}\sum_{i\in\mathcal{M}}\log p_\theta(x_i\mid x_{<i}),
\]
and the perplexity $\mathrm{PPL} = \exp(\mathcal{L}_{\mathrm{VGPT}}) \in [1,\,V]$ serves as an interpretable monitoring metric.

\paragraph{Tokenization during training.}
Tokenization is performed inside the data pipeline with all Stage~1 parameters frozen and detached from the computation graph:
\[
(\mathbf{c},\mathbf{d}_{1:T}) = \mathcal{T}_{\text{Stage1}}(x_{1:T}),\qquad
\mathbf{a}_{1:T-1} = \mathrm{discretize}(a_{1:T-1}),
\]
followed by the offsetting and concatenation above. The Transformer therefore only ever consumes integer token ids, and gradients flow exclusively through $\pi_\theta$.

\paragraph{Hyperparameters.}
Stage~2 hyperparameters are summarized in Table~\ref{tab:vwm_stage2_hparams}. Validation computes $\mathcal{L}_{\mathrm{VGPT}}$ and $\mathrm{PPL}$ on a held-out split every $5{,}000$ steps. At inference, autoregressive sampling uses temperature $\tau=1.0$ with top-$k$ and top-$p$ truncation; beam search is not used.

\begin{table}[h]
  \centering
  \scriptsize
  \setlength{\tabcolsep}{6pt}
  \renewcommand{\arraystretch}{1.10}
  \caption{Stage~2 autoregressive Transformer hyperparameters.}
  \label{tab:vwm_stage2_hparams}
  \fitwidth{%
  \begin{tabular}{@{}ll@{}}
    \toprule
    \rowcolor{tblheadbg}
    Item & Value \\
    \midrule
    \rowcolor{tblstripe}
    Backbone & LLaMA-12L-768d-12h (FlashAttention-2, bf16) \\
    Vocabulary size $V$ & $9008$ \\
    \rowcolor{tblstripe}
    Sequence length $S$ & $1931$ (max positions $8192$) \\
    Supervised positions $|\mathcal{M}|$ & $480$ \\
    \rowcolor{tblstripe}
    Frozen tokenizer & Stage~1 checkpoint at $60{,}000$ steps \\
    Data module & Context + MSP processor \\
    \rowcolor{tblstripe}
    Optimizer & AdamW, $(\beta_1,\beta_2,\varepsilon)\!=\!(0.9,0.999,10^{-8})$, weight decay $0$ \\
    Learning rate & $5\times 10^{-5}$ \\
    \rowcolor{tblstripe}
    LR schedule & constant-with-warmup, warmup $5{,}000$ steps \\
    Gradient clipping & $\|\nabla\|_2 \le 1.0$ \\
    \rowcolor{tblstripe}
    Dropout / stochastic depth & $0$ \\
    Per-GPU batch / accumulation & $4\,/\,1$ \\
    \rowcolor{tblstripe}
    GPUs & $8\times$ A100-40GB \\
    Global effective batch & $32$ \\
    \rowcolor{tblstripe}
    Max steps $T_{\max}$ & $1{,}000{,}000$ \\
    Checkpoint / validation frequency & every $10{,}000$ / $5{,}000$ steps \\
    \bottomrule
  \end{tabular}}
\end{table}

\subsection{Dataset: Bridge V2 Instantiation}\label{app:vwm_dataset}

The two-stage pipeline is dataset-agnostic and only assumes synchronized image observations paired with low-dimensional action sequences. For reproducibility we instantiate it on \textbf{Bridge V2}~(Walke~\emph{et~al.}, 2023) as a representative case; any dataset satisfying the same interface can be substituted without further changes.

Bridge V2 is a large-scale real-robot manipulation dataset collected with a WidowX-250 six-DoF arm across $24$ indoor kitchen and tabletop environments. It contains approximately $60{,}096$ human-teleoperated trajectories covering $13$ skill categories (grasping, placing, pushing, pulling, opening and closing drawers, flipping, stacking, and so on). Each trajectory logs synchronized third-person RGB, end-effector pose, and gripper state at $5$\,Hz. The dataset is distributed in the TensorFlow Datasets (TFDS) format; we use version $1.0.0$.

\paragraph{Raw sample structure.}
Each raw trajectory is a variable-length sequence with the fields summarized in Table~\ref{tab:vwm_bridge_fields}, where $L$ is the trajectory length (typically $20$--$60$ frames). The data is distributed as TFRecord shards and split into training and validation partitions.

\begin{table}[h]
  \centering
  \scriptsize
  \setlength{\tabcolsep}{6pt}
  \renewcommand{\arraystretch}{1.10}
  \caption{Core fields of a Bridge V2 trajectory.}
  \label{tab:vwm_bridge_fields}
  \fitwidth{%
  \begin{tabular}{@{}llll@{}}
    \toprule
    \rowcolor{tblheadbg}
    Field & Shape & Dtype & Description \\
    \midrule
    \rowcolor{tblstripe}
    Primary RGB & $(L,\,480,\,640,\,3)$ & uint8 & Third-person primary-camera observation \\
    Proprio.\ state & $(L,\,7)$ & float32 & End-effector $(x,y,z,\text{roll},\text{pitch},\text{yaw},\text{gripper})$ \\
    \rowcolor{tblstripe}
    Action & $(L,\,7)$ & float32 & EE increments $(\Delta x,\Delta y,\Delta z,\Delta\text{roll},\Delta\text{pitch},\Delta\text{yaw},\text{gripper\_cmd})$ \\
    Language instruction & string & --- & Natural-language task description (unused here) \\
    \rowcolor{tblstripe}
    Terminal flag & bool & --- & Episode-end indicator \\
    \bottomrule
  \end{tabular}}
\end{table}

\paragraph{Pre-processing (TFDS $\rightarrow$ per-trajectory arrays).}
The TFDS source is converted offline into per-trajectory array files once, so that the downstream data pipeline does not need to touch TFRecords again:
\begin{itemize}[leftmargin=1.2em,itemsep=1pt,topsep=1pt]
  \item \textbf{Image resampling.} $480\!\times\!640 \to H\!\times\!W = 256\!\times\!320$ via aspect-ratio-preserving resize followed by center cropping, stored as uint8.
  \item \textbf{Action expansion.} The raw $\mathbb{R}^{7}$ action is expanded to $\mathbb{R}^{13}$ by appending several dimensions of the previous end-effector state, matching the shared OXE action layout ($D_a=13$).
  \item \textbf{Action-range statistics.} The training split is scanned once to record per-dimension $(a_i^{\min},a_i^{\max})$; the resulting table is persisted and reused by the equal-width binning in \S\ref{app:vwm_stage2}.
\end{itemize}
After conversion, each trajectory becomes an independent array file containing an image tensor $(L,\,256,\,320,\,3)$ in uint8, an action tensor $(L,\,13)$ in float32, and a proprioceptive state $(L,\,7)$ in float32. The resulting corpus has ${\approx}60{,}096$ trajectories, over $1.92$\,M frames, and occupies roughly $260$\,GB on disk.

\paragraph{Sampling and slicing.}
Each optimization step samples a clip according to:
\begin{enumerate}[leftmargin=1.2em,itemsep=1pt,topsep=1pt]
  \item Uniformly sample an episode from the training split.
  \item Uniformly sample a start frame $s\in[0,\,L-T]$.
  \item Take $\{I_{s},\ldots,I_{s+T-1}\}$ and the paired actions $\{a_{s},\ldots,a_{s+T-1}\}$ with $T=8$ and stride~$1$.
  \item Apply the clip-wise photometric and geometric augmentations of Table~\ref{tab:vwm_augment}, with parameters shared across the $T$ frames to preserve temporal consistency.
  \item Normalize images to $[0,\,1]$ before feeding them to the tokenizer.
\end{enumerate}
Stage~1 consumes only the pixels $\{I_{s+k}\}_{k=0}^{T-1}$; Stage~2 additionally consumes the paired actions to construct the joint token sequence of \S\ref{app:vwm_stage2}. The two stages share identical sampling and augmentation logic, and differ only in their collate-level outputs.

\subsection{Reproduction Protocol}\label{app:vwm_repro}

We describe the end-to-end reproduction protocol in terms of inputs, computation, and outputs, deliberately avoiding prescriptive script layouts or launch commands.

\paragraph{Software environment.}
The implementation is built on PyTorch with the version constraints in Table~\ref{tab:vwm_env}. Distributed training uses Accelerate over the NCCL backend; both stages train under bf16 mixed precision to satisfy the dtype constraints of FlashAttention-2.

\begin{table}[h]
  \centering
  \scriptsize
  \setlength{\tabcolsep}{6pt}
  \renewcommand{\arraystretch}{1.10}
  \caption{Key software dependencies.}
  \label{tab:vwm_env}
  \fitwidth{%
  \begin{tabular}{@{}ll@{}}
    \toprule
    \rowcolor{tblheadbg}
    Component & Version \\
    \midrule
    \rowcolor{tblstripe}
    PyTorch & $2.3$ \\
    FlashAttention & $2.5$ \\
    \rowcolor{tblstripe}
    HuggingFace Accelerate & $0.30$ \\
    HuggingFace Transformers & latest stable \\
    \rowcolor{tblstripe}
    HuggingFace Diffusers & latest stable \\
    TensorFlow (only for TFDS I/O) & $2.15$ \\
    \rowcolor{tblstripe}
    LPIPS / timm / einops & latest stable \\
    \bottomrule
  \end{tabular}}
\end{table}

\paragraph{Step 0: Offline pre-processing of raw data.}
This step is CPU-only and converts TFRecord shards into the per-trajectory array files of \S\ref{app:vwm_dataset}, reporting the total number of converted trajectories as a consistency check. It runs once per dataset release and takes ${\approx}2$--$3$ hours on a single CPU node at the scale of the example dataset.

\paragraph{Step 1: Train the context-aware tokenizer.}
Stage~1 is trained on the pre-processed per-trajectory data with the hyperparameters of \S\ref{app:vwm_stage1}: global effective batch $32$, base learning rate $5\times 10^{-4}$, discriminator start at step $10{,}000$, $6\times 10^{4}$ total steps, checkpoint every $2{,}000$ steps, and validation every $1{,}000$ steps on a held-out split (L1 and LPIPS). Activation checkpointing is enabled to reduce memory. The final checkpoint becomes the frozen backbone of Stage~2. Training takes ${\approx}72$~hours on a single $8\times$ A100-40GB node.

\paragraph{Step 2: Train the autoregressive Transformer.}
Stage~2 is trained on the joint token sequence of \S\ref{app:vwm_stage2} with the hyperparameters therein: per-GPU batch $4$, no accumulation, global effective batch $32$, learning rate $5\times 10^{-5}$ with a $5{,}000$-step linear warmup, gradient clipping $\|\nabla\|_2\le 1.0$, $10^{6}$ total steps, checkpoint every $10{,}000$ steps, and validation every $5{,}000$ steps. The Stage~1 tokenizer is loaded read-only and remains frozen throughout. Training throughput is ${\approx}3.5$~it/s on a single $8\times$ A100-40GB node, totalling ${\approx}80$~hours for the full $10^{6}$ steps.

\paragraph{Monitoring and sanity checks.}
All scalar and image summaries are logged to TensorBoard. In practice we monitor three families of health indicators:
\emph{(i) mixed precision}---bf16 must be explicitly enabled, since FlashAttention-2 refuses to execute under fp32;
\emph{(ii) hardware utilization}---compute utilization should remain consistently above $90\%$ with near-uniform per-GPU memory; otherwise the bottleneck typically lies in data loading or NCCL communication;
\emph{(iii) training stability}---a smooth generator-loss transition around discriminator activation in Stage~1, and long-term stability of the gradient norm together with a monotone decrease of the held-out PPL in Stage~2.
Common failure modes and mitigations are listed in Table~\ref{tab:vwm_failures}.

\begin{table}[h]
  \centering
  \scriptsize
  \setlength{\tabcolsep}{6pt}
  \renewcommand{\arraystretch}{1.10}
  \caption{Common failure modes observed during training.}
  \label{tab:vwm_failures}
  \begin{tabular}{@{}p{0.30\linewidth}p{0.30\linewidth}p{0.34\linewidth}@{}}
    \toprule
    \rowcolor{tblheadbg}
    Symptom & Root cause & Fix \\
    \midrule
    \rowcolor{tblstripe}
    FlashAttention only supports fp16 / bf16 & Mixed precision misconfigured as fp32 & Force the mixed-precision policy to bf16 \\
    NVML / NVLink symbol missing & Legacy node driver (e.g.\ the 450.x series) & Disable framework-side NVML probing and turn off NCCL P2P/NVLS fast paths \\
    \rowcolor{tblstripe}
    Nested snapshots and abnormal disk usage under the checkpoint directory & Residuals from previous smoke-test runs & Clean the smoke-test directory before starting the production run \\
    \bottomrule
  \end{tabular}
\end{table}

\subsection{Optimization Infrastructure}\label{app:vwm_infra}

\begin{table}[h]
  \centering
  \scriptsize
  \setlength{\tabcolsep}{6pt}
  \renewcommand{\arraystretch}{1.10}
  \caption{Optimization infrastructure shared by Stage~1 and Stage~2.}
  \label{tab:vwm_infra}
  \fitwidth{%
  \begin{tabular}{@{}ll@{}}
    \toprule
    \rowcolor{tblheadbg}
    Item & Value \\
    \midrule
    \rowcolor{tblstripe}
    Framework & PyTorch + HuggingFace Accelerate (multi-GPU) \\
    Distributed backend & NCCL \\
    \rowcolor{tblstripe}
    Mixed precision & bf16 (both stages) \\
    Activation checkpointing & Stage~1 enabled, Stage~2 disabled \\
    \rowcolor{tblstripe}
    Logging & TensorBoard scalar and image summaries \\
    Random seed & set per experiment via Accelerate \\
    \rowcolor{tblstripe}
    Checkpoint format & HuggingFace Safetensors (unwrapped) \\
    \bottomrule
  \end{tabular}}
\end{table}

Stage~1 produces tokenizer checkpoints at a fixed cadence; Stage~2 loads a specific checkpoint in read-only mode as its frozen backbone. No optimizer state, LR scheduler state, or data-loader state is shared between the two stages.

\subsection{End-to-End Quantitative Summary}\label{app:vwm_summary}

\begin{table}[h]
  \centering
  \scriptsize
  \setlength{\tabcolsep}{5pt}
  \renewcommand{\arraystretch}{1.10}
  \caption{Side-by-side comparison of the two training stages.}
  \label{tab:vwm_summary}
  \fitwidth{%
  \begin{tabular}{@{}lll@{}}
    \toprule
    \rowcolor{tblheadbg}
    Dimension & Stage 1 & Stage 2 \\
    \midrule
    \rowcolor{tblstripe}
    Input & $x_{1:T}\in\mathbb{R}^{T\times 3\times H\times W}$ & $\mathbf{x}\in\{0,\ldots,V-1\}^{S}$ \\
    Output & Reconstructions $\hat x_{1:T}$ and FSQ indices & Logits $\in\mathbb{R}^{S\times V}$ \\
    \rowcolor{tblstripe}
    Objective & $\mathcal{L}_G$ (L1 + LPIPS + hinge-GAN) & $\mathcal{L}_{\mathrm{VGPT}}$ (masked causal CE) \\
    Quantization & FSQ, $K = 4375$ & --- \\
    \rowcolor{tblstripe}
    Tokens per clip & $1840$ & $1931$ (supervised $480$) \\
    Trainable parameters & Stage~1 only & Stage~2 only (Stage~1 frozen) \\
    \rowcolor{tblstripe}
    Total steps & $6\times 10^{4}$ & $10^{6}$ \\
    Effective batch & $32$ & $32$ \\
    \rowcolor{tblstripe}
    Mixed precision & bf16 & bf16 \\
    Hardware & $8\times$ A100-40GB & $8\times$ A100-40GB \\
    \bottomrule
  \end{tabular}}
\end{table}

\paragraph{Compact formula card.}
For reference, the end-to-end forward and training equations are collected below:
{\footnotesize
\begin{align*}
\text{Stage 1 encoders:}\quad &
z_c = \mathrm{QuantConv}\bigl(E_c(x_c)\bigr),\quad
z_d = W_q\,\mathrm{Patchify}_p\!\bigl(E_d(x_d;\,\{f_l^{(E)}\})\bigr),\\[2pt]
\text{Stage 1 quantization:}\quad &
\hat z = \mathcal{Q}_{\mathbf{L}}(z),\quad
\mathrm{idx}(\hat z)=\sum_{i}\bigl(\hat z_i+\lfloor L_i/2\rfloor\bigr)\!\!\prod_{j<i}\!L_j,\\[2pt]
\text{Stage 1 decoders:}\quad &
\hat x_c = D_c\bigl(\mathcal{Q}(z_c)\bigr),\quad
\hat x_d = D_d\!\bigl(\mathcal{Q}(z_d);\,\{f_l^{(D)}\}\bigr),\\[2pt]
\text{Stage 1 loss:}\quad &
\mathcal{L}_G = \lambda_r\bigl(\|x_d-\hat x_d\|_1+\|x_c-\hat x_c\|_1\bigr)\\
& \hphantom{\mathcal{L}_G = } + \lambda_p\bigl(\mathrm{LPIPS}(x_d,\hat x_d)+\mathrm{LPIPS}(x_c,\hat x_c)\bigr)\\
& \hphantom{\mathcal{L}_G = } - \lambda_a\,\mathbb{1}[t\ge T_{\mathrm{adv}}]\,\mathbb{E}\bigl[D_\psi(\hat x)\bigr],\\[2pt]
\text{Stage 2 tokens:}\quad &
\mathbf{c}\in[K,2K)^{N_c},\;\mathbf{d}_t\in[0,K)^{N_d},\;\mathbf{a}_t\in[2K,2K+B_a)^{D_a},\\[2pt]
\text{Stage 2 sequence:}\quad &
\mathbf{x}=\bigl[\mathbf{c}\,\|\,\mathbf{d}_1\|\mathbf{a}_1\,\|\,\cdots\,\|\,\mathbf{d}_T\|\mathbf{a}_T\bigr],\quad S=1931,\\[2pt]
\text{Stage 2 model:}\quad &
p_\theta(x_i\mid x_{<i})=\bigl[\mathrm{softmax}(W_O\,h_{i-1})\bigr]_{x_i},\\[2pt]
\text{Stage 2 loss:}\quad &
\mathcal{L}_{\mathrm{VGPT}}=-\tfrac{1}{|\mathcal{M}|}\sum_{i\in\mathcal{M}}\log p_\theta(x_i\mid x_{<i}),\;\;
\mathcal{M}=\{\text{positions of }\mathbf{d}_2,\ldots,\mathbf{d}_T\},\;|\mathcal{M}|=480.
\end{align*}
}
Constants: $T=8,\;t_c=1,\;t_d=7,\;H\times W=256\times 320,\;p=4,\;D=5,\;\mathbf{L}=(7,5,5,5,5),\;K=4375,\;N_c=1280,\;N_d=80,\;D_a=13,\;B_a=256,\;V=9008,\;S=1931,\;|\mathcal{M}|=480$.

\subsection{TensorBoard Training Curves}\label{app:vwm_tb}

This section summarizes the convergence curves of both stages on the Bridge~V2 instantiation. All curves are exported directly from TensorBoard scalars without smoothing, with global optimization steps on the horizontal axis.

\subsubsection{Stage 1: Tokenizer (GAN + Reconstruction + Perception)}

Stage~1 logs four prefix-grouped scalar families. Tables~\ref{tab:vwm_tb_gen}--\ref{tab:vwm_tb_val} report the trends observed over the first $60{,}000$ steps; each TensorBoard tag (quoted verbatim) maps one-to-one to the exported CSV file.

\begin{table}[h]
  \centering
  \scriptsize
  \setlength{\tabcolsep}{4pt}
  \renewcommand{\arraystretch}{1.10}
  \caption{Generator loss $\mathcal{L}_G$ and its components (\texttt{gen\_loss/*}).}
  \label{tab:vwm_tb_gen}
  \begin{tabular}{@{}p{0.28\linewidth}p{0.22\linewidth}p{0.42\linewidth}@{}}
    \toprule
    \rowcolor{tblheadbg}
    Curve & Term & Observed trend ($0\!\to\!60{,}000$ steps) \\
    \midrule
    \rowcolor{tblstripe}
    \texttt{gen\_loss/recon\_loss} & $\mathcal{L}_{\text{rec}}$ (dynamic-frame L1) & Monotone drop from ${\sim}0.11$ to $\mathbf{0.0169}$ \\
    \texttt{gen\_loss/perceptual\_loss} & $\mathcal{L}_{\text{perc}}$ (dynamic-frame LPIPS) & Monotone drop from ${\sim}0.30$ to $\mathbf{0.0607}$ \\
    \rowcolor{tblstripe}
    \texttt{gen\_loss/ref\_recon\_loss} & Context-frame L1 & Same magnitude as \texttt{recon\_loss}; final ${\sim}0.02$ \\
    \texttt{gen\_loss/ref\_perceptual\_loss} & Context-frame LPIPS & Tracks \texttt{perceptual\_loss}; final ${\sim}0.07$ \\
    \rowcolor{tblstripe}
    \texttt{gen\_loss/gan\_loss} & $\mathcal{L}_{\text{adv}}^{G}$ (hinge, $\lambda_a=0.1$) & $0$ for step~$\le 10{,}000$; afterwards oscillates in $[0.10,\,0.40]$; final $\mathbf{0.2546}$ \\
    \texttt{gen\_loss/commit\_loss} & VQ commitment (exactly $0$ under FSQ) & $0$ throughout \\
    \rowcolor{tblstripe}
    \texttt{gen\_loss/dyna\_commit\_loss} & Dynamic-branch commitment (exactly $0$ under FSQ) & $0$ throughout \\
    \texttt{step\_gen\_loss} & Generator total loss $\mathcal{L}_G$ & Drops from ${\sim}0.50$ to $\mathbf{0.1630}$ \\
    \bottomrule
  \end{tabular}
\end{table}

\begin{table}[h]
  \centering
  \scriptsize
  \setlength{\tabcolsep}{4pt}
  \renewcommand{\arraystretch}{1.10}
  \caption{Discriminator loss $\mathcal{L}_D$ and logits (\texttt{disc\_loss/*}; logged only for step~$\ge 10{,}000$).}
  \label{tab:vwm_tb_disc}
  \begin{tabular}{@{}llp{0.50\linewidth}@{}}
    \toprule
    \rowcolor{tblheadbg}
    Curve & Term & Observed trend \\
    \midrule
    \rowcolor{tblstripe}
    \texttt{disc\_loss/real\_logits} & $\mathbb{E}\!\left[D_\psi(x)\right]$ & Mild drift in $[-0.30,\,0.10]$; final $\mathbf{-0.2297}$ \\
    \texttt{disc\_loss/fake\_logits} & $\mathbb{E}\!\left[D_\psi(\hat x)\right]$ & Drifts synchronously with \texttt{real\_logits}; final $\mathbf{-0.2463}$ \\
    \rowcolor{tblstripe}
    \texttt{disc\_loss/logit\_diff} & $\mathbb{E}\!\left[D_\psi(x)-D_\psi(\hat x)\right]$ & Hovers around $0$ with amplitude $<0.05$, indicating $D$ does not overwhelm $G$ \\
    \texttt{step\_discr\_loss} & Discriminator total loss $\mathcal{L}_D$ & Quickly settles into the $[1.5,\,2.2]$ band after activation; final $\mathbf{1.9864}$, consistent with a healthy GAN balance against $\mathcal{L}_{\text{adv}}^{G}$ \\
    \bottomrule
  \end{tabular}
\end{table}

\begin{table}[h]
  \centering
  \scriptsize
  \setlength{\tabcolsep}{4pt}
  \renewcommand{\arraystretch}{1.10}
  \caption{Validation scalars (\texttt{val\_loss/*}; averaged over $100$ batches every $1{,}000$ steps).}
  \label{tab:vwm_tb_val}
  \begin{tabular}{@{}llp{0.50\linewidth}@{}}
    \toprule
    \rowcolor{tblheadbg}
    Curve & Meaning & Observed trend \\
    \midrule
    \rowcolor{tblstripe}
    \texttt{val\_loss/recon\_loss} & Mean pixel L1 & Monotone drop; best $\mathbf{0.01944}$ at step~$50{,}001$, final $0.01968$ at step~$59{,}001$ \\
    \texttt{val\_loss/perceptual\_loss} & Mean LPIPS & Monotone drop; best $\mathbf{0.06840}$ at step~$50{,}001$, final $0.07147$ at step~$59{,}001$ \\
    \bottomrule
  \end{tabular}
\end{table}

\paragraph{Optimization and throughput.}
The auxiliary scalar \texttt{lr} tracks the generator learning rate (warmup for $5{,}000$ steps, followed by cosine decay); \texttt{samples\_sec\_gpu}, \texttt{batch\_time}, and \texttt{data\_time} respectively record per-GPU throughput, per-step wall time, and data-loading time. Reconstruction previews \texttt{images/reconstruction\_*} are saved every $1{,}000$ steps for qualitative inspection of high-frequency texture fidelity.

\paragraph{Key training phases.}
\textbf{(i)} Steps $0$--$10{,}000$: pure L1 + LPIPS supervision; reconstruction and perceptual curves descend fastest.
\textbf{(ii)} Step~$10{,}000$: discriminator activation (\texttt{disc\_start}~$=10{,}000$); \texttt{gen\_loss/gan\_loss} becomes nonzero and \texttt{step\_gen\_loss} shows a mild ${+}0.05$-magnitude rebound.
\textbf{(iii)} Steps $10{,}000$--$60{,}000$: reconstruction and adversarial terms decrease jointly, and \texttt{val\_loss/*} continues to drop monotonically.

\paragraph{Training-curve visualization.}
Figure~\ref{fig:vwm_stage1_gen} shows the four generator-side training curves, and Figure~\ref{fig:vwm_stage1_disc_val} shows the discriminator logits, validation metrics, and learning-rate schedule. In each panel the light trace is the raw sampled values and the dark trace an exponential moving average ($\alpha=0.08$); a dash-dotted vertical line marks discriminator activation at step~$10{,}000$.

\begin{figure}[h]
  \centering
  \includegraphics[width=1\linewidth]{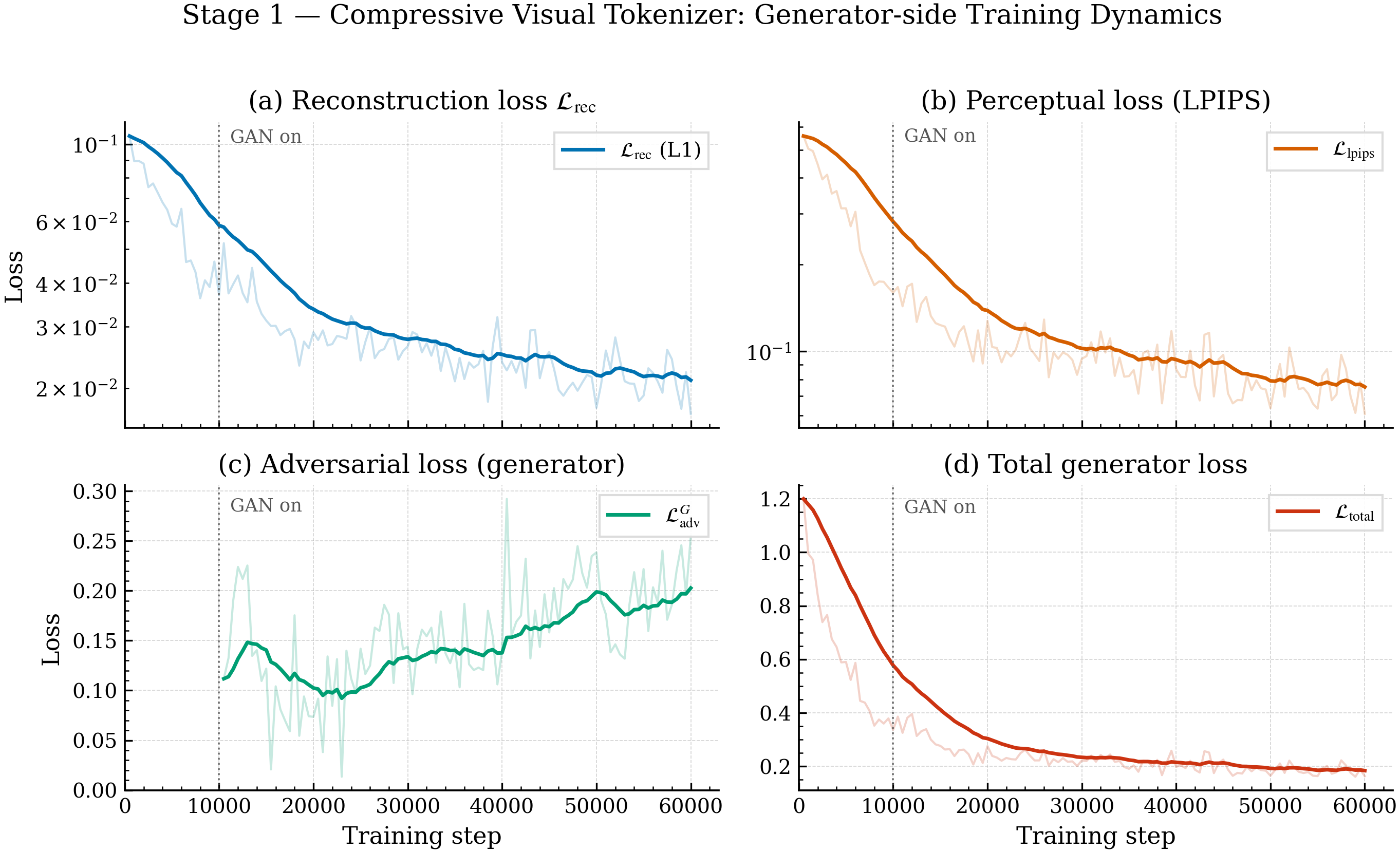}
  \caption{\textbf{Stage~1 generator-side training dynamics.}
  (a) Reconstruction loss $\mathcal{L}_{\mathrm{rec}}$ (\texttt{gen\_loss/recon\_loss}, L1);
  (b) perceptual loss $\mathcal{L}_{\mathrm{lpips}}$ (\texttt{gen\_loss/perceptual\_loss});
  (c) adversarial loss $\mathcal{L}_{\mathrm{adv}}^{G}$ (\texttt{gen\_loss/gan\_loss});
  (d) generator total loss $\mathcal{L}_{G}$ (\texttt{step\_gen\_loss}).
  Panels (a) and (b) descend fastest before discriminator activation and continue to decrease monotonically afterwards; panel (c) stabilizes in $[0.10,\,0.40]$ after activation, reflecting the dynamic $G$--$D$ equilibrium.}
  \label{fig:vwm_stage1_gen}
\end{figure}

\begin{figure}[h]
  \centering
  \includegraphics[width=1\linewidth]{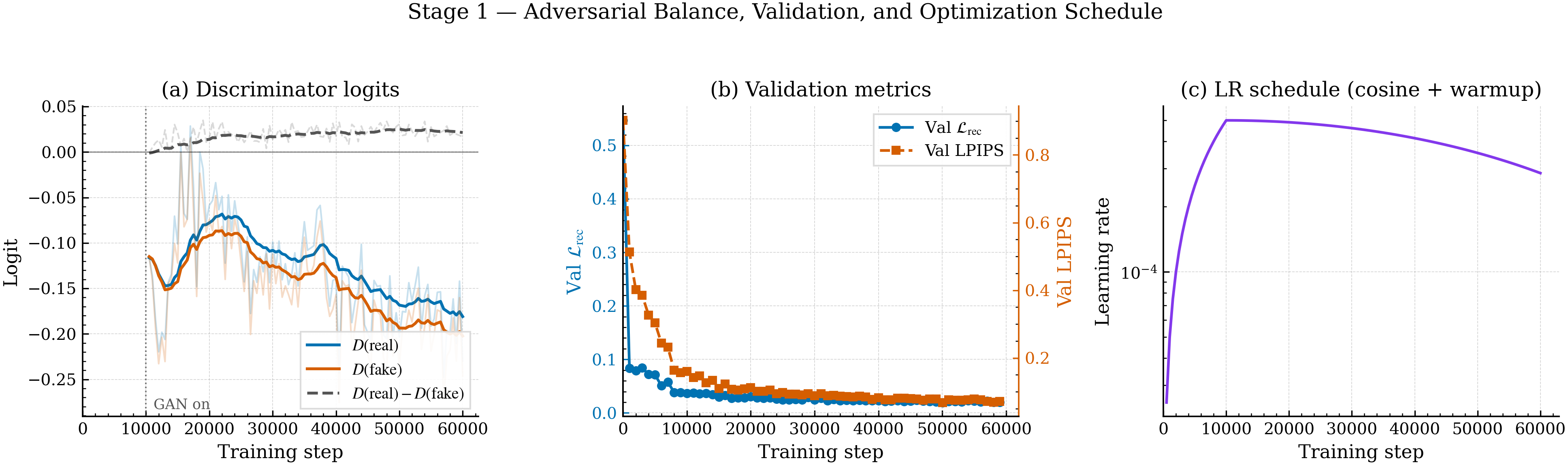}
  \caption{\textbf{Stage~1 adversarial balance, validation metrics, and optimization schedule.}
  (a) Discriminator logits on real and fake samples (\texttt{disc\_loss/real\_logits}, \texttt{disc\_loss/fake\_logits}) and their difference \texttt{disc\_loss/logit\_diff}~$=D(\mathrm{real})-D(\mathrm{fake})$, which hovers around $0$ with a mild positive drift and indicates that $D$ does not overwhelm $G$;
  (b) validation L1 (\texttt{val\_loss/recon\_loss}, left axis) and LPIPS (\texttt{val\_loss/perceptual\_loss}, right axis), both decreasing monotonically;
  (c) generator learning rate \texttt{lr}, which undergoes a $5{,}000$-step warmup followed by cosine decay.}
  \label{fig:vwm_stage1_disc_val}
\end{figure}

\paragraph{Reconstruction evolution.}
To visualize how the numerical trends translate into perceptual quality, Figure~\ref{fig:vwm_stage1_recon_evolution} shows pixel-space reconstructions from the Stage~1 tokenizer at six representative training steps ($\{1,\,2{,}500,\,10{,}000,\,25{,}000,\,50{,}000,\,60{,}000\}$). Each training step is rendered as a pair of adjacent rows: the upper row shows the ground-truth frames $x$ and the lower row the decoder outputs $\hat x$; columns uniformly sample $4$ temporal positions ($t_1,\,t_3,\,t_5,\,t_7$) out of the $T=8$ clip frames, so that both spatial fidelity and temporal consistency become visually apparent.

\begin{figure}[h]
  \centering
  \includegraphics[width=1\linewidth]{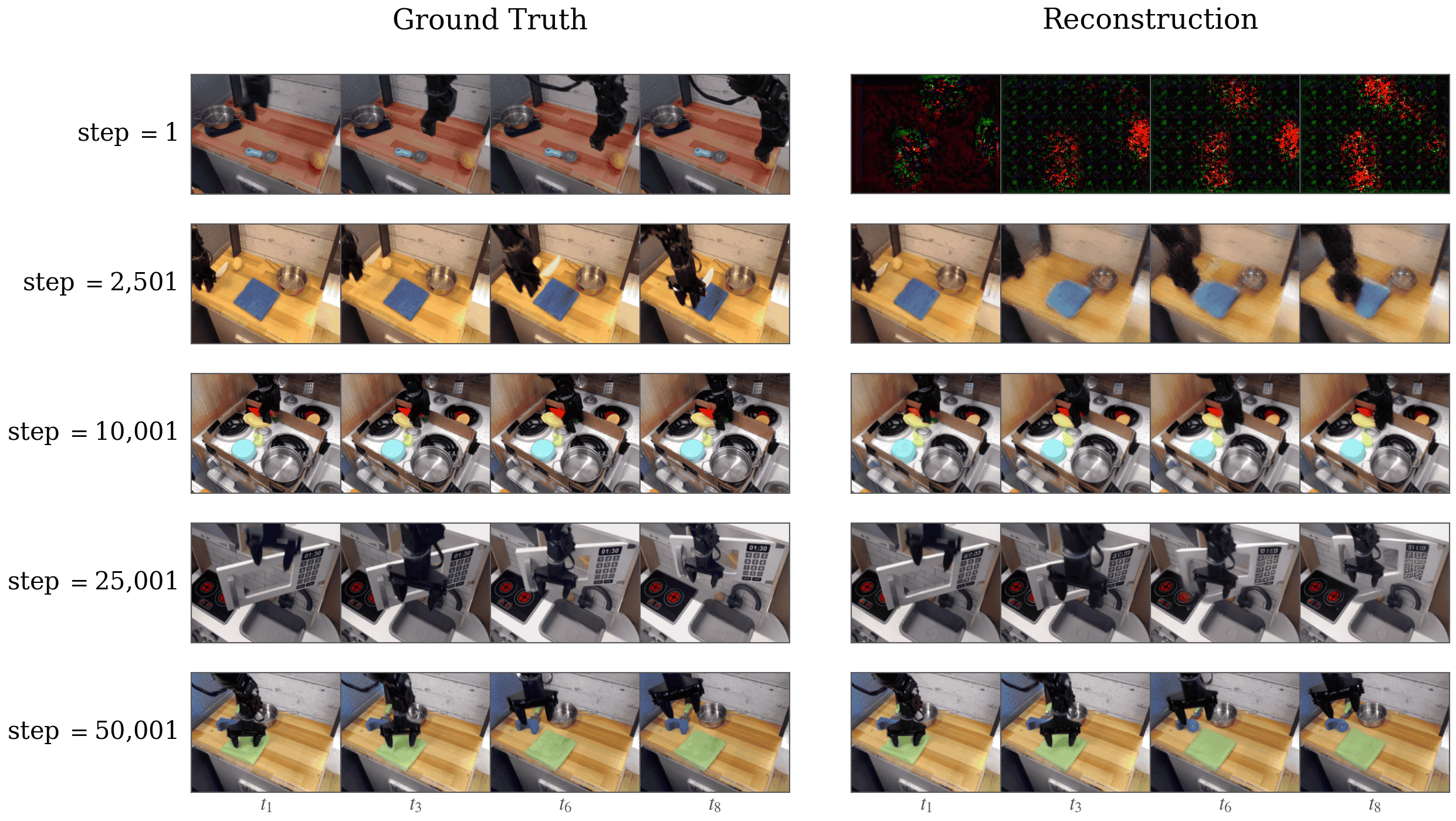}
  \caption{\textbf{Evolution of Stage~1 reconstructions across training steps.}
  Horizontal axis: temporal frame position $t_i$ inside a clip. Vertical axis: training steps from early to late; every two rows form a comparison pair, with the upper row showing the ground-truth frame and the lower row showing the decoded $\hat x$ from the same FSQ indices.
  At step~$\le 10^{3}$ the dynamic branch fails to recover high-frequency textures occluded by the manipulator, and reconstructions exhibit global blur and mild colour bias;
  around step~$10^{4}$ (discriminator activation) texture sharpness improves noticeably, yet locally visible grid-like artefacts still contribute a non-negligible LPIPS gap;
  beyond step~$2.5\times 10^{4}$, reconstructions become visually indistinguishable from ground truth, with residual errors concentrated on the narrow contact region between the end-effector and the manipulated object.
  This qualitative evolution is step-aligned with the monotone L1/LPIPS curves in Figures~\ref{fig:vwm_stage1_gen}--\ref{fig:vwm_stage1_disc_val} and with the best-validation step ($5\times 10^{4}$) of \texttt{val\_loss/*}.}
  \label{fig:vwm_stage1_recon_evolution}
\end{figure}

\subsubsection{Stage 2: Autoregressive Transformer (Causal CE + Perplexity)}

Stage~2 is monitored with TensorBoard scalars centered on the masked causal cross-entropy objective and its induced perplexity. Figure~\ref{fig:vwm_stage2_curves} reports the exported raw traces together with an exponential moving average (EMA, $\alpha=0.9$), avoiding manually tabulated estimates from the plots.

\begin{figure}[h]
  \centering
  \includegraphics[width=0.78\linewidth]{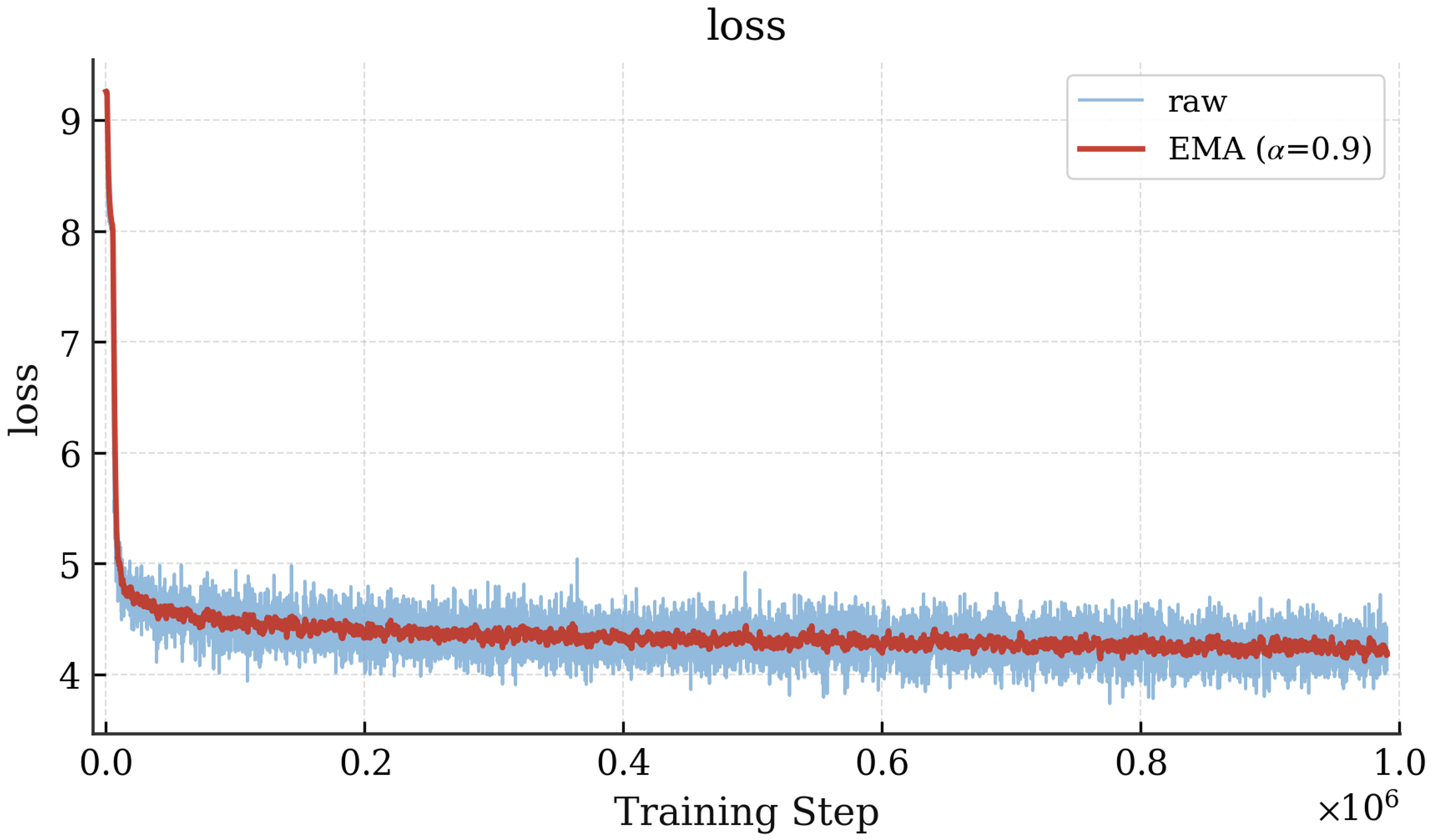}\\[0.4em]
  \begin{minipage}[t]{0.49\linewidth}
    \centering
    \includegraphics[width=\linewidth]{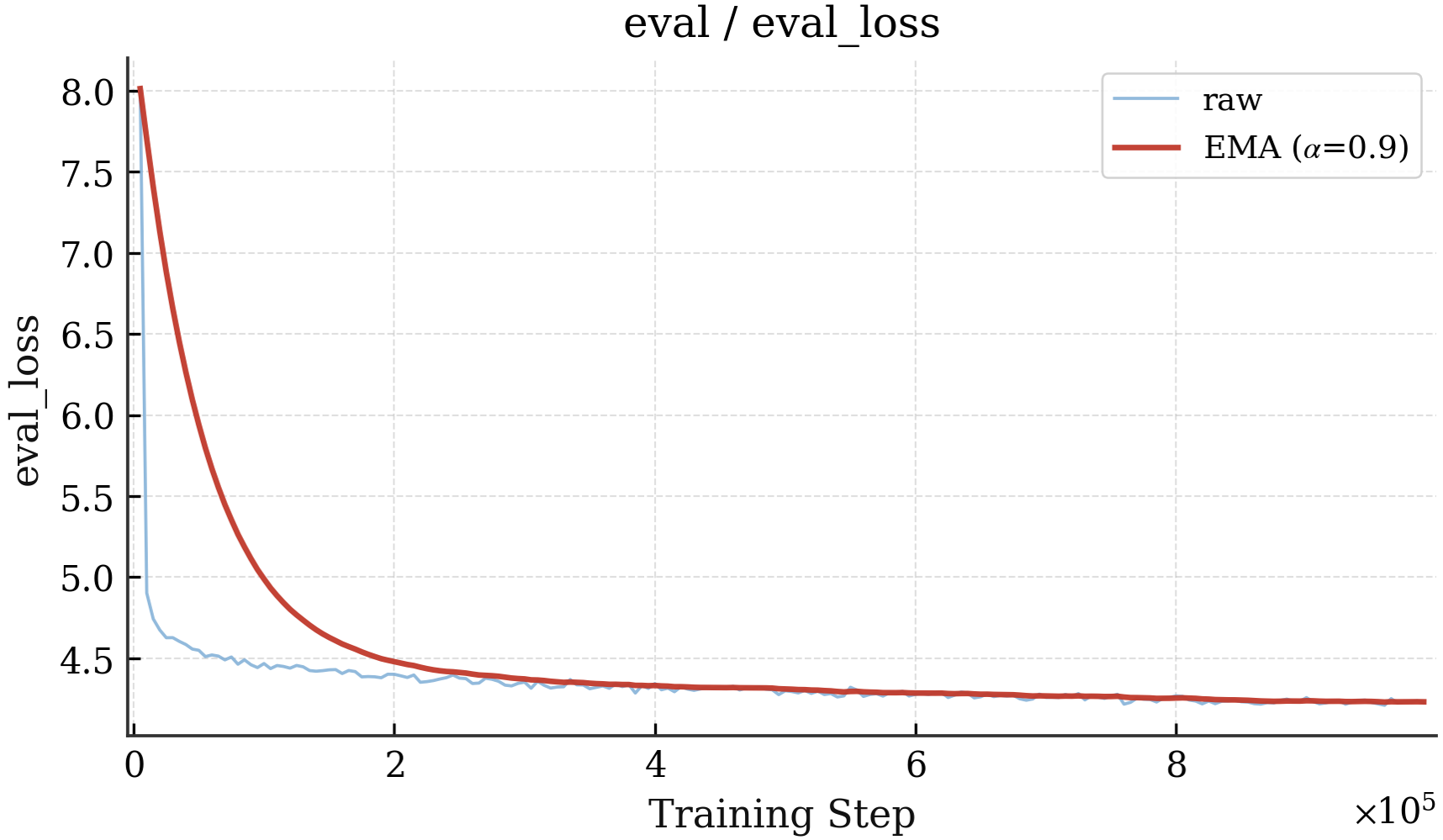}
  \end{minipage}\hfill
  \begin{minipage}[t]{0.49\linewidth}
    \centering
    \includegraphics[width=\linewidth]{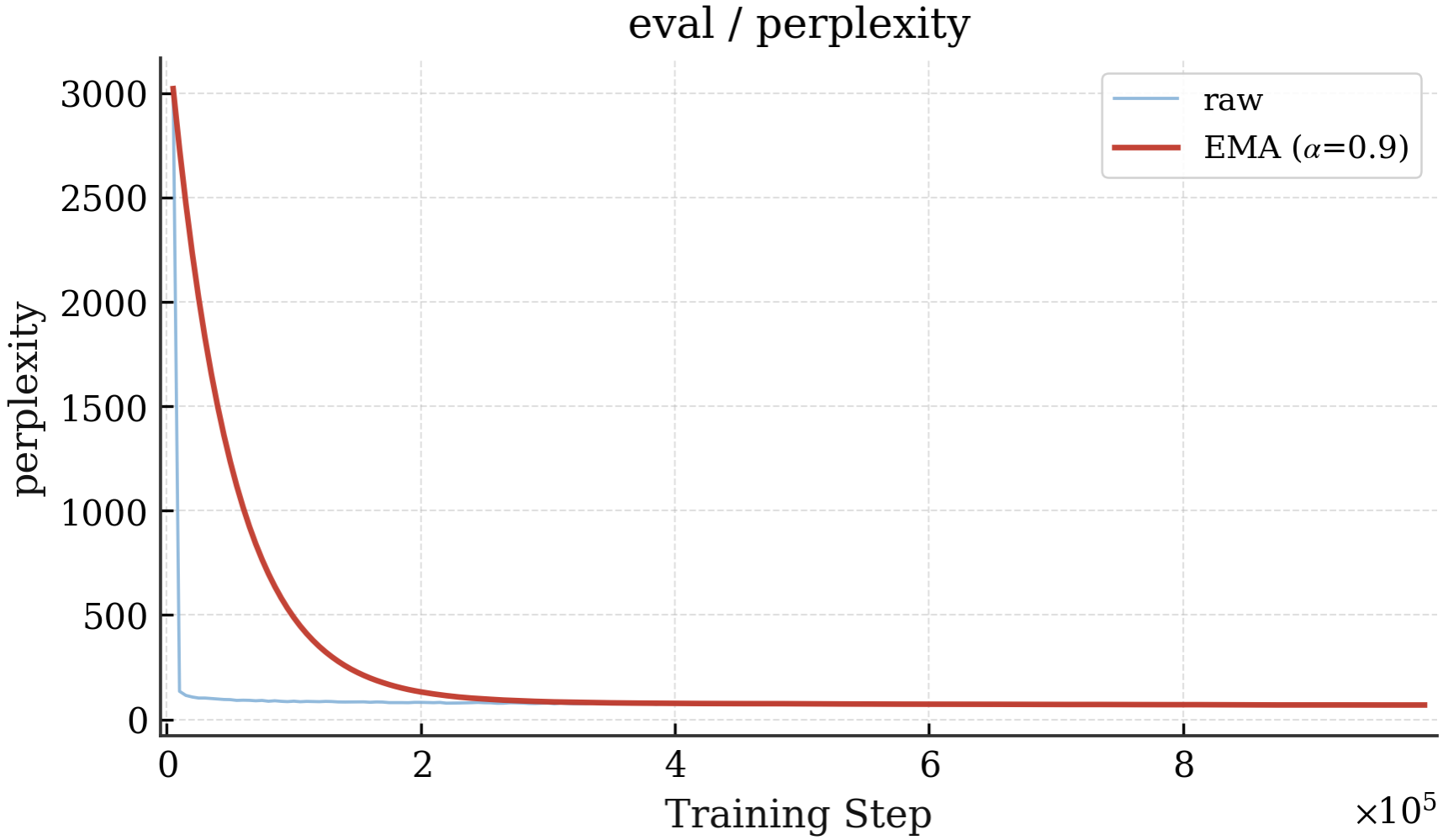}
  \end{minipage}
  \caption{\textbf{Stage~2 autoregressive Transformer training curves.}
  Top: training masked causal cross-entropy loss. Bottom left: held-out evaluation loss. Bottom right: held-out evaluation perplexity. Each panel shows both the raw TensorBoard trace and an EMA-smoothed curve with $\alpha=0.9$.}
  \label{fig:vwm_stage2_curves}
\end{figure}

\paragraph{Observed convergence.}
The training loss drops sharply at the beginning of optimization and then enters a slower refinement regime, with the EMA curve continuing to decrease over the full training horizon. The held-out loss follows the same overall pattern: a rapid early reduction followed by a gradual flattening, without a late-stage upward trend. The evaluation perplexity is consistent with the held-out cross-entropy curve, decreasing rapidly in the early phase and then approaching a stable low-variance regime. Together, these curves indicate that the Stage~2 Transformer learns the token dynamics early and continues to refine its next-token distribution over long training, while the validation metrics do not show visible signs of divergence from the training trajectory.

\clearpage
\section{Stylized derivation for Proposition~\ref{prop:swr}}\label{app:proof_swr}

We provide the simplified derivation underlying the main-text stability discussion. The goal is to illustrate why periodic refresh can limit long-horizon drift under a local contraction assumption, rather than to claim a complete dynamical model of the tokenizer--decoder pair or a quantitative predictor of the exact gains observed in experiments.

\paragraph{Setup and notation.}
Consider a token-based autoregressive world model generating a rollout of $T$ frames partitioned into $K = \lceil T/W \rceil$ segments of window size $W$.
Let $x_t^{*}$ denote the ground-truth frame at step $t$ and $\hat{x}_t$ the predicted frame.
We define the frame-level error as $e_t = \lVert \hat{x}_t - x_t^{*} \rVert$.
We make two assumptions:

\begin{enumerate}[label=\textbf{A\arabic*}.,leftmargin=2em]
  \item \textbf{Bounded per-step error.} Conditioned on a correct context, the single-step prediction error satisfies $e_t \le \varepsilon$ for all $t$.
  More generally, if the context carries an error $\eta$, the prediction error at the next step satisfies $e_{t} \le \varepsilon + \alpha\,\eta$ for some contraction factor $\alpha \in [0,1)$.
  \item \textbf{Bounded quantization error.} The decode--re-encode cycle introduces a quantization error bounded by $\delta_q$: for any predicted frame $\hat{x}$, $\lVert \mathcal{D}_{\mathrm{vis}}(\mathcal{T}_{\mathrm{vis}}(\hat{x})) - \hat{x} \rVert \le \delta_q$.
\end{enumerate}

\paragraph{Standard autoregressive generation (no re-encoding).}
Under native AR decoding, the context is never refreshed, so context-carried error keeps accumulating across the full horizon. If a contraction assumption analogous to \textbf{A1} held globally, one would obtain the geometric series
\[
e_t \;\le\; \varepsilon \sum_{i=0}^{t-1} \alpha^i \;=\; \frac{\varepsilon\,(1-\alpha^t)}{1-\alpha}.
\]
For any fixed $\alpha < 1$, this is bounded by $\varepsilon/(1-\alpha)$, but this constant can still blow up as $\alpha\!\to\!1$. In the non-contractive worst case, the accumulation becomes linear in the horizon, yielding the coarse comparison bound
\begin{equation}\label{eq:ar_error}
\mathcal{E}_{\mathrm{AR}}(T) \;=\; \max_{t \le T}\, e_t \;\le\; T\,\varepsilon.
\end{equation}
The AR regime therefore either pays a horizon-independent but potentially large constant $\varepsilon/(1-\alpha)$, or a linear-in-$T$ bound when contraction fails; SWR will instead replace this with a window-size-controlled constant.

\paragraph{Sliding-window re-encoding.}
With SWR, at each segment boundary $t = kW$, the model decodes the last predicted frame $\hat{x}_{kW}$ to pixel space and re-encodes it as a fresh context/state pair $(z^{\mathrm{ctx}}_{k+1},\, z^{\mathrm{dyn}}_0) = \mathcal{T}_{\mathrm{vis}}(\hat{x}_{kW})$; concretely, $\hat{x}_{kW}$ is treated both as the new conditioning frame for the Full-VAE branch $E_c$ and as a length-one dynamic clip for the Conditional-VAE branch $E_d$ (Appendix~\ref{app:vwm_stage1}).

\textit{Step 1: Within-segment error.}
Within segment $k$, the context $z^{\mathrm{ctx}}_k$ is fixed and the model generates at most $W$ frames. By \textbf{A1}, the error at relative step $j \in \{1, \dots, W\}$ within the segment satisfies:
\[
e_{(k-1)W+j} \;\le\; \varepsilon \sum_{i=0}^{j-1} \alpha^i + \alpha^{j}\,\eta_k,
\]
where $\eta_k$ is the error carried by the context $z^{\mathrm{ctx}}_k$. For $j \le W$, using $\sum_{i=0}^{j-1}\alpha^i = \frac{1-\alpha^j}{1-\alpha} \le j \le W$ and $\alpha^j\le 1$:
\begin{equation}\label{eq:within_seg}
e_{(k-1)W+j} \;\le\; \frac{\varepsilon(1-\alpha^W)}{1-\alpha} + \alpha^W \eta_k \;\le\; W\varepsilon + \eta_k.
\end{equation}
Taking $j=W$ gives the tighter intermediate bound $e_{kW}\le W\varepsilon + \alpha^W\eta_k$, which we use in Step~2.

\textit{Step 2: Cross-segment error refresh.}
At the boundary, the context error of segment $k{+}1$ is:
\[
\eta_{k+1}
= \lVert z^{\mathrm{ctx}}_{k+1} - z^{\mathrm{ctx},*}_{k+1} \rVert
\le \lVert \hat{x}_{kW} - x_{kW}^{*} \rVert + \delta_q
= e_{kW} + \delta_q.
\]
Crucially, $z^{\mathrm{ctx}}_{k+1}$ is obtained by re-encoding the \emph{decoded} frame, so the next segment depends on the current decoded observation rather than the entire raw token history of all prior segments. Substituting the tighter form of Eq.~\ref{eq:within_seg} at $j=W$:
\[
\eta_{k+1} \;\le\; W\varepsilon + \alpha^W \eta_k + \delta_q.
\]

\textit{Step 3: Steady-state bound.}
Since $\alpha^W < 1$, the recurrence $\eta_{k+1} \le W\varepsilon + \alpha^W \eta_k + \delta_q$ converges to a fixed point:
\[
\eta^{*} = \frac{W\varepsilon + \delta_q}{1 - \alpha^W}.
\]
Starting from $\eta_1 = 0$ (ground-truth conditioning frame), by induction $\eta_k \le \eta^{*}$ for all $k$. Combining with the loose within-segment bound $e_{(k-1)W+j}\le W\varepsilon + \eta_k$ from Eq.~\ref{eq:within_seg} (which upper-bounds $\alpha^W\eta_k$ by $\eta_k$), the maximum frame-level error in any segment is therefore:
\begin{equation}\label{eq:swr_final}
\mathcal{E}_{\mathrm{SWR}}(T) = \max_{t \le T}\, e_t \;\le\; W\varepsilon + \eta^{*} = W\varepsilon + \frac{W\varepsilon + \delta_q}{1 - \alpha^W}.
\end{equation}

In the simplified case $\alpha \to 0$ (errors do not propagate within the causal Transformer's effective receptive field beyond one step), $\alpha^W \to 0$ and $\eta^{*} \to W\varepsilon + \delta_q$, so the bound becomes $\mathcal{E}_{\mathrm{SWR}}(T) \le 2W\varepsilon + \delta_q$.
More generally, for any $\alpha \in [0,1)$ and any window size $W$ such that $\alpha^W < 1$, the bound in Eq.~\ref{eq:swr_final} is finite and does not grow explicitly with $T$, showing that under this stylized local model the effect of SWR is governed by the refresh window and the local error parameters rather than the rollout horizon alone.
We do not estimate $\alpha$ or $\delta_q$ from the trained model; instead, the derivation is meant to justify the qualitative window-size trade-off observed in Table~\ref{tab:swr_quality_efficiency} and Appendix~\ref{app:swr_window_ablate}, where moderate refresh intervals outperform both overly frequent and overly infrequent refresh. \qed

\paragraph{Asymptotic optimum of the window size.}
Treating $W$ as a positive real variable, let
\[
f(W) \;=\; W\varepsilon + \frac{W\varepsilon + \delta_q}{1 - \alpha^{W}}.
\]
This is the SWR bound of Eq.~\ref{eq:swr_final} written as a function of $W$ at fixed $(\varepsilon,\delta_q,\alpha)$. The qualitative trade-off discussed in Remark~\ref{rem:alpha} in the main text can be made precise by analyzing the first-order stationarity of $f$.

\begin{corollary}[\textbf{Asymptotic optimum $W^{*}$}]\label{cor:wstar}
Assume $\alpha \in (0,1)$ and $\delta_q \ge 0$. In the moderate-window regime $\alpha^{W} \ll 1$, the bound $f(W)$ is quasi-convex in $W$ and attains its minimum at
\begin{equation}\label{eq:wstar}
W^{*}
\;\approx\;
\frac{1}{-\ln \alpha}\,
\ln\!\Biggl(\,
\frac{(-\ln \alpha)\,\delta_q}{2\,\varepsilon}
\,\Biggr),
\end{equation}
provided the argument of the logarithm exceeds $1$; otherwise the minimum is achieved at the smallest admissible $W$ (i.e.\ the quantization cost $\delta_q$ is too small relative to $\varepsilon$ to justify any batching). In particular, $W^{*}$ is logarithmic in $\delta_q/\varepsilon$ and in $1/(-\ln \alpha)$.
\end{corollary}

\begin{proof}
In the regime $\alpha^W \ll 1$, we expand $1/(1-\alpha^W) \approx 1 + \alpha^W + \mathcal{O}(\alpha^{2W})$ in Eq.~\ref{eq:swr_final}, giving
\[
f(W) \;\approx\; 2 W\varepsilon + \delta_q \;+\; \alpha^{W}\,(W\varepsilon + \delta_q) \;+\; \mathcal{O}(\alpha^{2W}).
\]
Differentiating with respect to $W$ and retaining leading-order terms (using $\frac{d}{dW}\alpha^W = \alpha^W \ln \alpha$ and discarding the subleading $\alpha^W \varepsilon$ contribution),
\[
\frac{df}{dW} \;\approx\; 2\varepsilon \;+\; \alpha^{W}\,\ln\alpha \,\bigl(W\varepsilon + \delta_q\bigr).
\]
Setting $df/dW = 0$ and using $\ln\alpha < 0$,
\[
2\varepsilon \;=\; \alpha^{W}\,(-\ln\alpha)\,\bigl(W\varepsilon + \delta_q\bigr).
\]
In the typical regime $\delta_q \gg W\varepsilon$ (single-step token error small relative to the decode--re-encode distortion, as observed in our token-level probes), the right-hand side is dominated by $\delta_q$, giving
\[
\alpha^{W} \;\approx\; \frac{2\varepsilon}{(-\ln\alpha)\,\delta_q},
\qquad
W^{*} \;=\; \frac{1}{-\ln \alpha}\,
\ln\!\Bigl(\tfrac{(-\ln \alpha)\,\delta_q}{2\,\varepsilon}\Bigr),
\]
which is Eq.~\ref{eq:wstar}. The second-derivative test at this point gives $d^2f/dW^2 \approx -(\ln\alpha)^2 \alpha^{W}(W\varepsilon+\delta_q) > 0$, so $f$ is locally convex and $W^{*}$ is a minimum.
When the argument of the outer logarithm drops below $1$, i.e.\ $\delta_q \le 2\varepsilon/(-\ln\alpha)$, the stationarity condition has no interior solution and $f$ is monotone decreasing over its admissible range; the constrained minimum then lies at the smallest $W$ permitted by the inference hardware or the decoder receptive field.
\end{proof}

\paragraph{Physical interpretation.}
Corollary~\ref{cor:wstar} attaches a concrete meaning to the three parameters of the stylized model.
\textbf{(i) Role of the contraction factor $\alpha$.} As $\alpha \to 1^{-}$ (within-segment error is barely damped), the prefactor $1/(-\ln\alpha)$ diverges and $W^{*}$ grows, i.e.\ the model benefits from \emph{longer} segments that amortize the one-time refresh cost; in the opposite limit $\alpha \to 0^{+}$ (errors decay within a single step), short segments suffice and $W^{*}$ is small.
\textbf{(ii) Role of the noise ratio $\delta_q / \varepsilon$.} Large $\delta_q/\varepsilon$ (each refresh is comparatively expensive) pushes $W^{*}$ up, while small $\delta_q/\varepsilon$ (refresh is nearly free) pushes $W^{*}$ down; the logarithmic dependence means that $W^{*}$ is insensitive to the precise value of $\delta_q$ and $\varepsilon$ within an order of magnitude.
\textbf{(iii) Consistency with the ablation.} In our setting we expect the tokenizer--decoder cycle to be moderately contractive and the decode--re-encode distortion $\delta_q$ to dominate the single-step token error $\varepsilon$, so Eq.~\ref{eq:wstar} predicts $W^{*}$ to lie in the small-integer regime. Empirically, Appendix~\ref{app:swr_window_ablate} finds $W=6$ optimal among $\{4,6,8,10,15\}$, which is consistent with this regime; $W=4$ pays too much refresh overhead while $W\ge 10$ leaves too much within-segment drift, as Corollary~\ref{cor:wstar} predicts.
These conclusions are meant only to explain the \emph{shape} of the $W$--quality trade-off under the stylized assumptions \textbf{A1}--\textbf{A2}; we do not claim that Eq.~\ref{eq:wstar} is a numerical predictor of the exact optimum, nor that a single global $(\alpha, \varepsilon, \delta_q)$ triple suffices to describe every manipulation task in the benchmark.

\paragraph{Comparison with vanilla AR decoding.}
The two regimes can be placed on a common axis by viewing $\mathcal{E}_{\mathrm{AR}}(T)$ of Eq.~\ref{eq:ar_error} against $\mathcal{E}_{\mathrm{SWR}}(T)$ of Eq.~\ref{eq:swr_final}. This yields the following direct consequence, which makes the $T$--dependence gap explicit rather than only qualitative.

\begin{corollary}[\textbf{SWR versus AR asymptotic gap}]\label{cor:swr_vs_ar}
Under Assumptions \textbf{A1}--\textbf{A2} and for any $W$ with $\alpha^W < 1$,
\begin{equation}\label{eq:swr_vs_ar}
\mathcal{E}_{\mathrm{AR}}(T) \;-\; \mathcal{E}_{\mathrm{SWR}}(T)
\;\ge\; T\varepsilon \;-\; \Bigl( W\varepsilon + \tfrac{W\varepsilon + \delta_q}{1-\alpha^W} \Bigr),
\end{equation}
so the advantage of SWR grows at least linearly in the rollout horizon $T$ once $T$ exceeds the horizon-independent SWR constant. In the contractive regime $\alpha<1$ with $\alpha^W \ll 1$, the SWR constant simplifies to $\approx 2W\varepsilon + \delta_q$, and the crossover horizon above which SWR strictly dominates AR is $T^{\dagger} \approx W + \delta_q/\varepsilon$.
\end{corollary}

\begin{proof}
Subtract Eq.~\ref{eq:swr_final} from the non-contractive-regime bound of Eq.~\ref{eq:ar_error} directly. The simplified constant follows from $(W\varepsilon+\delta_q)/(1-\alpha^W) \to W\varepsilon+\delta_q$ as $\alpha^W \to 0$, and solving $T^{\dagger}\varepsilon \ge 2W\varepsilon+\delta_q$ for $T^{\dagger}$ gives the stated crossover.
\end{proof}

Eq.~\ref{eq:swr_vs_ar} turns the qualitative ``does not grow with $T$'' statement in the main text into an order-of-magnitude prediction: the longer the rollout, the larger the SWR advantage, with the crossover controlled by the ratio $\delta_q/\varepsilon$ rather than by $T$ itself.

\paragraph{Discussion of assumptions.}
Assumption \textbf{A1} replaces the true non-stationary rollout dynamics with a locally linear contraction around the on-policy trajectory; this is standard in stylized error-propagation analyses of autoregressive generation (see, e.g., the attention-sink regime of \cite{xiao2023efficient}) and is used here only to obtain a closed-form window-size trade-off. A global contraction assumption on the full horizon would in general fail for autoregressive token models, which is precisely the empirical observation that motivates SWR; our analysis therefore imposes contraction only \emph{within} a segment of length $W$, consistent with the segmented generation procedure of Eq.~\ref{eq:decode}.
Assumption \textbf{A2} treats the decode--re-encode cycle as a bounded, input-independent distortion. This is a conservative over-approximation: in practice $\delta_q$ depends on the content of the refreshed frame, but the \emph{worst-case} bound we derive only requires a uniform upper bound, which is furnished by the FSQ tokenizer's finite codebook. Finally, the corollaries above are invariant under any affine rescaling of $(\varepsilon, \delta_q)$: only the ratio $\delta_q/\varepsilon$ enters $W^{*}$, which explains why the optimum window is robust to the specific scale on which token-level error is measured.

\clearpage
\section{Pseudocode for Reward-Aligned Post-Training}\label{app:code_train}

Algorithm~\ref{alg:train} summarizes the complete \method{} training pipeline, covering all four stages described in \S\ref{sec:reward}.
Color coding marks each stage:
{\color{mydarkblue}\textbf{benchmark construction}} (\textcolor{mydarkblue}{blue}),
{\color{codegreen}\textbf{teacher judge training}} (\textcolor{codegreen}{green}),
{\color{codepurple}\textbf{student reward distillation}} (\textcolor{codepurple}{purple}), and
{\color{red!55!black}\textbf{GRPO post-training}} (\textcolor{red!55!black}{red}).

\begin{algorithm}[H]
\caption{Reward-Aligned Post-Training of \method{}}\label{alg:train}
\begin{algorithmic}[1]
\Require Pre-trained world model $p_\theta$;\; robot datasets $\{$RT-1, Bridge, CALVIN, LIBERO$\}$;\; T2V model;\; base VLM (Qwen3-VL-8B-Thinking)
\Ensure Post-trained world model $p_{\theta^{*}}$
\Statex
\Statex \textcolor{mydarkblue}{\textit{--- Stage 1: Benchmark Construction ---}}
\For{each instruction $l$ in robot datasets}
    \State \textcolor{mydarkblue}{Sample ground-truth video $v^{+}$ from dataset}
    \State \textcolor{mydarkblue}{Generate candidate video $v^{-} \gets \mathrm{T2V}(l)$}
    \State \textcolor{mydarkblue}{Annotate $(l, v)$ with raw rubric scores $\mathbf{r} = (r_1,\dots,r_6)$ using ranges $[3,2,1,1,1,2]$}
    \Comment{\textcolor{mydarkblue}{Six dimensions}}
\EndFor
\State \textcolor{mydarkblue}{$\mathcal{D}_{\mathrm{bench}} \gets \{(l_i, v_i, \mathbf{r}_i)\}_{i=1}^{N}$}
\Statex
\Statex \textcolor{codegreen}{\textit{--- Stage 2: Teacher Judge Training ---}}
\State \textcolor{codegreen}{Initialize teacher $f_\phi$ from Qwen3-VL-8B-Thinking}
\State \textcolor{codegreen}{Fine-tune $f_\phi$ on $\mathcal{D}_{\mathrm{bench}}$: $\;\mathcal{L}_{\mathrm{teacher}} = -\mathbb{E}\bigl[\log p_\phi(\mathbf{r}\mid l, v)\bigr]$}
\Comment{\textcolor{codegreen}{$f_\phi$ $\to$ \textsc{RoboAlign-Judge}}}
\Statex
\Statex \textcolor{codepurple}{\textit{--- Stage 3: Student Reward Distillation ---}}
\State \textcolor{codepurple}{Initialize student $g_\psi$ (compact visual--text encoder + linear head)}
    \State \textcolor{codepurple}{Construct $\mathcal{D}_{\mathrm{distill}}$ from benchmark videos and generated rollouts scored by teacher $f_\phi$}
\Comment{\textcolor{codepurple}{Teacher-labeled mixed corpus}}
    \State \textcolor{codepurple}{Train $g_\psi$ with per-dimension weighted Huber regression on normalized teacher scores (Eq.~\ref{eq:distill})}
\Statex
\Statex \textcolor{red!55!black}{\textit{--- Stage 4: GRPO Post-Training with Online Iterative Distillation ---}}
\For{iteration $n = 1, 2, \dots$}
    \State \textcolor{red!55!black}{Sample group of $G$ rollouts $\{\hat{x}^{(j)}\}_{j=1}^{G} \sim p_\theta$}
    \State \textcolor{red!55!black}{Score each rollout: $R^{(j)} = \sum_{k=1}^{6} w_k [g_\psi(l, \hat{x}^{(j)})]_k$}
    \State \textcolor{red!55!black}{Compute advantage: $A^{(j)} = (R^{(j)} - \mathrm{mean}_j R) / \mathrm{std}_j R$}
    \State \textcolor{red!55!black}{Update $\theta$ via clipped policy gradient $\mathcal{L}_{\mathrm{GRPO}}$ (Eq.~\ref{eq:grpo})}
    \If{\textcolor{codepurple}{$n \bmod K = 0$}}
        \State \textcolor{codepurple}{Score fresh rollouts with teacher $f_\phi$ and update student $g_\psi$}
        \Comment{\textcolor{codepurple}{Online iterative distillation}}
    \EndIf
\EndFor
\State \Return $p_{\theta^{*}}$
\end{algorithmic}
\end{algorithm}

\clearpage
\section{Student Reward Model Details}\label{app:student_reward}

Figure~\ref{fig:student_reward_model} presents the architecture of the distilled student reward model $g_\psi$. The student is designed to provide efficient reward estimation within the inner RL loop by approximating the six-dimensional score vector produced by the teacher judge $f_\phi$. In our implementation, the model contains approximately 98M parameters in total, of which about 40.6M are trainable under partial ViT unfreezing. Its inference latency is approximately 20\,ms per video (roughly 50 videos/s), making it substantially more efficient than direct teacher inference.

\begin{figure}[t]
  \centering
  \includegraphics[width=\linewidth]{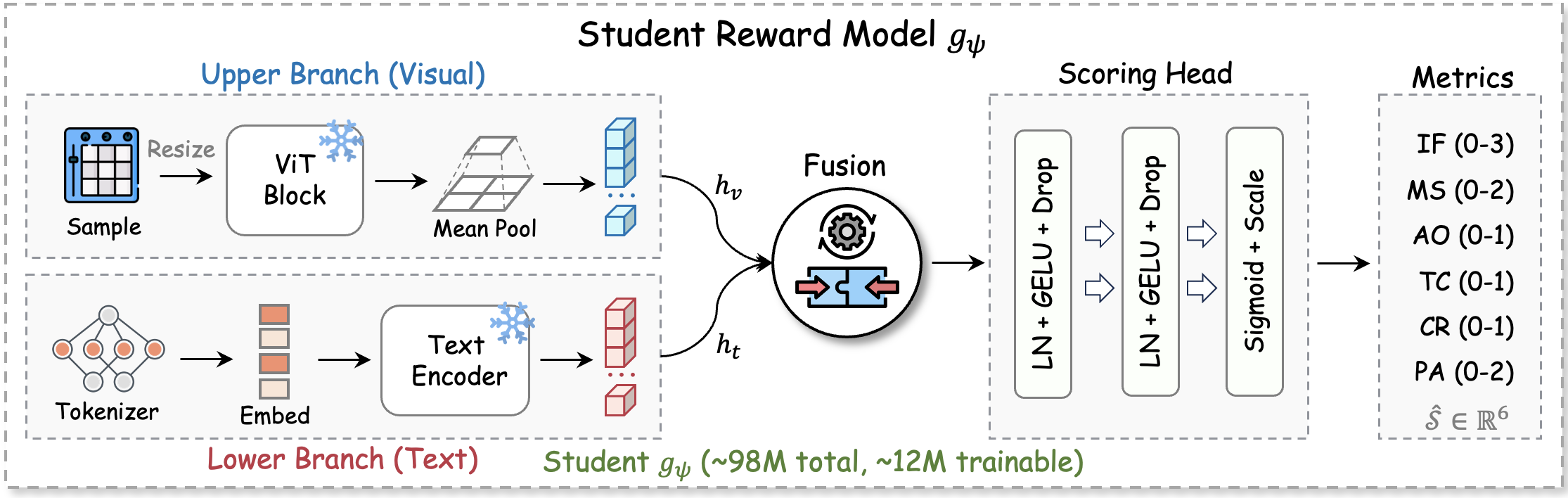}
  \caption{Architecture of the distilled student reward model $g_\psi$. A visual branch encodes uniformly sampled video frames using a partially unfrozen ViT-Base/16 backbone, while a text branch encodes the instruction using a lightweight Transformer. The fused multimodal representation is then mapped to six normalized reward dimensions by a compact MLP head.}
  \label{fig:student_reward_model}
\end{figure}

\paragraph{Visual branch.}
Given a generated video $v$ consisting of $T$ frames, we uniformly sample $N{=}8$ key frames. During training, we introduce temporal jitter by perturbing the uniform frame indices with small random offsets, thereby improving robustness to minor timing variations. The sampled frames are further processed with video-level data augmentation, including a shared random resized crop, a shared horizontal flip, and mild per-frame color jitter. During validation and testing, all frames are deterministically resized to $224 \times 224$.

The augmented frames are encoded independently by a pretrained ViT-Base/16 visual encoder. Rather than freezing the backbone entirely, we freeze the first 8 Transformer blocks and unfreeze the final 4 blocks together with the final normalization layer. This design allows high-level visual features to adapt to failure modes specific to robot-manipulation videos while preserving generic low-level image representations. For each frame, we extract the [CLS] feature $\mathbf{h}_i \in \mathbb{R}^{768}$ and perform temporal mean pooling to obtain a video-level representation,
\[
\mathbf{v}_{\text{pool}} = \frac{1}{N}\sum_{i=1}^{N}\mathbf{h}_i.
\]
This representation is then passed through a lightweight projection layer (LayerNorm $\rightarrow$ Linear $\rightarrow$ GELU) to produce the final visual embedding $\mathbf{v} \in \mathbb{R}^{768}$.

\paragraph{Text branch.}
The instruction $l$ is tokenized using the BERT tokenizer from \texttt{bert-base-uncased} and truncated or padded to a maximum length of $L{=}64$. The resulting subword sequence is encoded by a lightweight 4-layer Transformer encoder with hidden size 256, 4 attention heads, and feed-forward dimension 1024. A learnable [CLS] token is prepended to the input sequence and combined with learnable positional embeddings. After passing through the 4 Transformer layers, the output corresponding to the [CLS] token is used as the instruction embedding $\mathbf{t} \in \mathbb{R}^{256}$. Each layer adopts a Pre-Norm Transformer block with GELU activations and dropout at rate 0.2.

\paragraph{Multimodal fusion and scoring head.}
We fuse the visual and textual modalities by simple concatenation,
\[
\mathbf{z} = [\mathbf{v};\mathbf{t}] \in \mathbb{R}^{1024}.
\]
The fused representation is processed by a 3-layer MLP scoring head with the following structure: Linear$(1024 \rightarrow 512)$ $\rightarrow$ LayerNorm $\rightarrow$ GELU $\rightarrow$ Dropout, followed by Linear$(512 \rightarrow 256)$ $\rightarrow$ LayerNorm $\rightarrow$ GELU $\rightarrow$ Dropout, and a final Linear$(256 \rightarrow 6)$ layer. A sigmoid activation maps the output to normalized reward scores in $[0,1]^6$, corresponding to instruction following, manipulation success, action--outcome consistency, temporal consistency, contact realism, and physics adherence. During distillation training and during GRPO post-training (Eq.~\ref{eq:reward}), the student is always consumed in this normalized $[0,1]^6$ form; the raw-range rescaling to $[3,2,1,1,1,2]$ is applied only when student outputs need to be reported alongside teacher scores for diagnostic or visualization purposes, and is never re-applied before the reward aggregation $R(\hat{x}_{1:T})$.

\paragraph{Distillation training.}
We train the student on teacher-labeled video--instruction pairs using normalized teacher scores, which prevents dimensions with larger numerical ranges from dominating optimization. Instead of plain MSE, we adopt a per-dimension weighted Huber loss:
\begin{equation}
\mathcal{L}_{\mathrm{distill}}
\;=\;
\frac{1}{B}\sum_{b=1}^{B}\sum_{k=1}^{6}
\lambda_k \,\mathrm{Huber}_{\delta_h}\!\bigl(\hat{s}_{b,k}, \tilde{s}_{b,k}\bigr),
\end{equation}
where $\hat{s}_{b,k}\!=\![g_\psi(l_b,v_b)]_k\!\in\![0,1]$ is the student's sigmoid-activated output, $\tilde{s}_{b,k}$ is the teacher score normalized dimension-wise to $[0,1]$, and $\delta_h=0.5$ is the Huber threshold. Since both sides already share the same $[0,1]$ scale after normalization, the per-dimension weights $\{\lambda_k\}$ are used only to equalize convergence speed across dimensions with unequal label-noise levels (e.g., 3-level rubric vs.\ binary rubric); in practice we tune them on a validation split rather than tying them to the raw score ranges. Note that $\{\lambda_k\}$ are distinct from the reward aggregation weights $\{w_k\}$ used in Eq.~\ref{eq:reward}.

Training further employs mixup regularization with probability 0.3, AdamW with weight decay 0.05, and differential learning rates: $5\times 10^{-6}$ for the unfrozen ViT blocks and $5\times 10^{-4}$ for the text encoder and scoring head. The learning-rate schedule uses linear warmup over the first 10\% of training steps followed by cosine decay, with a minimum learning rate of $10^{-6}$. We additionally apply gradient clipping with maximum norm 1.0, maintain an exponential moving average (EMA) of the trainable parameters with decay 0.999, and select checkpoints based on validation-set Pearson correlation with early stopping (patience 10). Collectively, these design choices improve fidelity to the teacher's ranking behavior and stabilize the student when used as an online reward proxy during RL post-training.

\paragraph{Backbone and architecture comparison.}
Because the student reward model is intended to serve as a high-throughput reward proxy rather than the primary source of semantic supervision, an important design question is whether a simple CNN-style reward model is sufficient, or whether a partially unfrozen ViT-based architecture is necessary. Table~\ref{tab:student_backbone_ablation} compares representative lightweight baselines with the final partially unfrozen ViT student used in our method. The results indicate that the final student achieves a more favorable trade-off between fidelity to the teacher and online efficiency than simpler reward regressors.

\begin{table}[t]
  \centering
  \scriptsize
  \setlength{\tabcolsep}{3pt}
  \renewcommand{\arraystretch}{1.05}
  \caption{\textbf{Backbone and architecture ablation of the student reward model.} Comparison of representative visual encoders and fusion designs under matched training and evaluation settings. Metrics capture both fidelity to teacher judgments and inference efficiency.}
  \label{tab:student_backbone_ablation}
  \fitwidth{%
  \begin{tabular}{@{}lccccc@{}}
    \toprule
    \rowcolor{tblheadbg}
    Method & Visual encoder & Fusion module & Pearson corr. $\uparrow$ & MSE $\downarrow$ & Latency (ms) $\downarrow$ \\
    \midrule
    \rowcolor{tblstripe}
    ResNet student & ResNet-50 & concat + MLP & 0.79 & 0.186 & 13 \\
    Frozen ViT student & ViT-Base/16 (frozen) & concat + MLP & 0.84 & 0.142 & 18 \\
    \rowcolor{tblstripe}
    Text-only student & none & Transformer + MLP & 0.52 & 0.421 & 4 \\
    \rowcolor{tblstripe}
    Ours & ViT-Base/16 (last 4 blocks unfrozen) & concat + MLP & 0.88 & 0.096 & 20 \\
    \bottomrule
  \end{tabular}}
\end{table}

\paragraph{Teacher--student efficiency trade-off.}
The teacher and student play complementary roles in our framework. The high-capacity teacher remains essential for providing semantically rich, high-fidelity supervision during offline reward distillation. The student is introduced solely to make such supervision practical within the inner RL loop, where thousands of rollout evaluations may be required. Table~\ref{tab:student_efficiency} summarizes the resulting efficiency gap.

\begin{table}[t]
  \centering
  \scriptsize
  \setlength{\tabcolsep}{3pt}
  \renewcommand{\arraystretch}{1.06}
  \caption{\textbf{Teacher--student efficiency comparison for reward evaluation.} The teacher provides high-fidelity supervisory signals, whereas the distilled student is optimized for fast online reward evaluation during RL post-training. All numbers are measured or estimated on a single A100 40GB GPU.}
  \label{tab:student_efficiency}
  \fitwidth{%
  \begin{tabular}{@{}lcc@{}}
    \toprule
    \rowcolor{tblheadbg}
    Statistic & Teacher judge ($f_\phi$) & Student reward model ($g_\psi$) \\
    \midrule
    \rowcolor{tblstripe}
    Architecture & Qwen3-VL-8B-Thinking + LoRA & ViT-Base/16 + 4L text Transformer + MLP \\
    Total parameters & ${\sim}8$B & ${\sim}98$M \\
    \rowcolor{tblstripe}
    Trainable parameters & LoRA-only adaptation & ${\sim}40.6$M \\
    Per-video latency & ${\sim}2800$ ms & ${\sim}20$ ms \\
    \rowcolor{tblstripe}
    Inference throughput & ${\sim}0.5$ videos/s & ${\sim}50$ videos/s \\
    Inference GPU memory & ${\sim}16$--$20$ GB & ${\sim}0.5$--$1$ GB \\
    \rowcolor{tblstripe}
    Input representation & video frames + judge prompt & 8 sampled frames + instruction \\
    Output representation & autoregressive text + parsed scores & direct 6D score vector \\
    \bottomrule
  \end{tabular}}
\end{table}

This efficiency gap is particularly important for RL post-training. For example, if a single training step requires scoring 16 generated videos, direct teacher evaluation would still take on the order of seconds even with moderate parallelism, whereas the student can score the same batch in a fraction of a second. In practice, this shifts reward computation from a dominant bottleneck to a comparatively minor overhead, while preserving the teacher as the source of high-quality supervision.

\begin{table}[t]
  \centering
  \scriptsize
  \setlength{\tabcolsep}{3pt}
  \renewcommand{\arraystretch}{1.06}
  \caption{\textbf{Reward-evaluation cost within the RL loop.} Comparison of teacher-only evaluation and distilled-student evaluation for representative RL workloads.}
  \label{tab:student_rl_cost}
  \fitwidth{%
  \begin{tabular}{@{}lcc@{}}
    \toprule
    \rowcolor{tblheadbg}
    RL workload & Teacher judge ($f_\phi$) & Student reward model ($g_\psi$) \\
    \midrule
    \rowcolor{tblstripe}
    Reward evaluation for 16 videos / step & ${\sim}45$ s (serial) / ${\sim}12$ s (4-way parallel) & ${\sim}0.32$ s (batched) \\
    Reward cost over 1000 RL steps & ${\sim}12{,}000$--$45{,}000$ s & ${\sim}320$ s \\
    \rowcolor{tblstripe}
    Share of total training time & dominant bottleneck & minor overhead \\
    Additional hardware demand & dedicated large-VLM inference budget & can share training GPU \\
    \bottomrule
  \end{tabular}}
\end{table}

\clearpage
\section{Pseudocode and Implementation of Sliding Window Re-encoding}\label{app:code}

We present the sliding-window re-encoding inference procedure in two complementary formats: Algorithm~\ref{alg:swr} provides a formal pseudocode overview, and Listing~\ref{lst:swr} shows the corresponding PyTorch implementation extracted from our codebase.
Both use matching color coding for the three key stages:
{\color{mydarkblue}\textbf{prompt construction}} (\textcolor{mydarkblue}{blue}),
{\color{codegreen}\textbf{autoregressive generation}} (\textcolor{codegreen}{green}), and
{\color{red!55!black}\textbf{context refresh}} (\textcolor{red!55!black}{red}).

\begin{algorithm}[H]
\caption{Sliding Window Re-encoding Inference}\label{alg:swr}
\begin{algorithmic}[1]
\Require World model $p_\theta$;\; tokenizer $\mathcal{T}_{\mathrm{vis}}$;\; decoder $\mathcal{D}_{\mathrm{vis}}$;\; conditioning frame $x_0$;\; actions $\{a_t\}_{t=1}^{T}$;\; window size $W$
\Ensure Generated video $\hat{x}_{1:T}$
\State $z^{\mathrm{ctx}}_1,\, z^{\mathrm{dyn}}_0 \gets \mathcal{T}_{\mathrm{vis}}(x_0)$
\Comment{Encode initial context ($x_0$ used as both $E_c$ and length-1 $E_d$ input)}
\State $\mathrm{frames\_out} \gets [\,]$;\quad $m \gets 1$
\While{$|\mathrm{frames\_out}| < T$}
    \State $W' \gets \min(W,\; T - |\mathrm{frames\_out}|)$
    \State \textcolor{mydarkblue}{$\mathrm{prompt} \gets [\,z^{\mathrm{ctx}}_m \;\|\; z^{\mathrm{dyn}}_0 \;\|\; \hat{a}_1\,]$}
    \Comment{\textcolor{mydarkblue}{Build prompt}}
    \State $\mathrm{window\_tokens} \gets [\,]$
    \For{$j = 1, \dots, W'$}
        \State \textcolor{codegreen}{$\hat{z}^{\mathrm{dyn}}_j \gets p_\theta\!\bigl(\cdot \mid \mathrm{prompt}\bigr)$}
        \Comment{\textcolor{codegreen}{Generate $N_d{=}80$ dynamics tokens}}
        \State Append $\hat{z}^{\mathrm{dyn}}_j$ to $\mathrm{window\_tokens}$
        \State \textcolor{mydarkblue}{$\mathrm{prompt} \gets \mathrm{prompt} \;\|\; \hat{z}^{\mathrm{dyn}}_j \;\|\; \hat{a}_{j+1}$}
        \Comment{\textcolor{mydarkblue}{Extend prompt}}
    \EndFor
    \State \textcolor{codegreen}{$\hat{x}_{\mathrm{win}} \gets \mathcal{D}_{\mathrm{vis}}(z^{\mathrm{ctx}}_m,\; \mathrm{window\_tokens})$}
    \Comment{\textcolor{codegreen}{Decode window to pixels}}
    \State Append $\hat{x}_{\mathrm{win}}$ to $\mathrm{frames\_out}$
    \If{$|\mathrm{frames\_out}| < T$}
        \State \textcolor{red!55!black}{$\hat{x}_{\mathrm{last}} \gets \hat{x}_{\mathrm{win}}[-1]$}
        \Comment{\textcolor{red!55!black}{Take last decoded frame}}
        \State \textcolor{red!55!black}{$z^{\mathrm{ctx}}_{m+1},\, z^{\mathrm{dyn}}_0 \gets \mathcal{T}_{\mathrm{vis}}(\hat{x}_{\mathrm{last}})$}
        \Comment{\textcolor{red!55!black}{Re-encode last frame via $E_c$ and length-1 $E_d$}}
    \EndIf
    \State $m \gets m + 1$
\EndWhile
\State \Return $\mathrm{frames\_out}$
\end{algorithmic}
\end{algorithm}

\FloatBarrier
\begin{lstlisting}[style=pythonstyle, caption={\textbf{PyTorch implementation} of sliding window re-encoding.  Comments are color-coded to match Algorithm~\ref{alg:swr}. @\color{mydarkblue}{\texttt{ctx\_tokens}}@ corresponds to @\color{mydarkblue}{$z^{\mathrm{ctx}}_m$}@ and @\color{mydarkblue}{\texttt{dyn\_tokens}}@ to @\color{mydarkblue}{$z^{\mathrm{dyn}}_0$}@ in Algorithm~\ref{alg:swr}.}, label={lst:swr}, escapechar=@]
def generate_sliding_window(
    model, tokenizer, ctx_tokens, dyn_tokens,
    actions, num_frames, window_size=6,
):
    all_frames = []
    frames_done, step = 0, 0

    while frames_done < num_frames:
        W = min(window_size, num_frames - frames_done)

        @\color{mydarkblue}{\# [BLUE] Build initial prompt: [ctx | dyn\_0 | action\_0]}@
        prompt = build_prompt(ctx_tokens, dyn_tokens, actions[step])
        window_dyn = []

        @\color{codegreen}{\# [GREEN] Autoregressive generation within one window}@
        for j in range(W):
            gen_tokens = model.generate(prompt, max_tokens=80)
            window_dyn.append(gen_tokens)
            prompt = prompt + gen_tokens + discretize(actions[step+j+1])

        @\color{codegreen}{\# [GREEN] Decode entire window to pixel space}@
        decoded = tokenizer.decode(ctx_tokens, window_dyn)
        all_frames.extend(decoded)
        frames_done += W
        step += W

        @\color{red!55!black}{\# [RED] Context refresh: decode-re-encode cycle}@
        @\color{red!55!black}{\# last\_frame fed to both E\_c and length-1 E\_d}@
        if frames_done < num_frames:
            last_frame = decoded[-1]
            ctx_tokens, dyn_tokens = tokenizer.encode(
                last_frame, last_frame
            )

    return all_frames
\end{lstlisting}
\FloatBarrier

\clearpage
\section{SWR window-size ablation}\label{app:swr_window_ablate}

This section complements the main-text SWR panel (Table~\ref{tab:swr_quality_efficiency}).
Unless noted, measurements use Llama-style causal LM (12 layers, hidden 768, 12 heads), vLLM inference, and 300 episodes $\times 30$ frames.

\begin{table}[t]
  \centering
  \scriptsize
  \setlength{\tabcolsep}{3pt}
  \renewcommand{\arraystretch}{1.06}
  \caption{SWR efficiency vs.\ default AR across $W$: timing, throughput, GPU memory, prompt length (tokens), mean per-frame latency, and re-encoding.}
  \label{tab:swr_efficiency_full}
  \fitwidth{%
  \begin{tabular}{@{}lcccccc@{}}
    \toprule
    \rowcolor{tblheadbg}
    Metric & AR ($\infty$) & $W{=}4$ & $W{=}6$ & $W{=}8$ & $W{=}10$ & $W{=}15$ \\
    \midrule
    Total time (s) & \plainscore{5.646}{0.01} & \plainscore{5.863}{0.02} & \plainscore{5.709}{0.07} & \plainscore{5.858}{0.28} & \plainscore{5.740}{0.11} & \plainscore{5.732}{0.07} \\
    \rowcolor{tblstripe}
    FPS & $5.31$ & $5.12$ & $5.26$ & $5.13$ & $5.23$ & $5.23$ \\
    Peak GPU (MB) & $34{,}206$ & $32{,}066$ & $32{,}781$ & $33{,}493$ & $34{,}214$ & $34{,}209$ \\
    \rowcolor{tblstripe}
    Mean prompt len. & $2{,}722$ & $1{,}506$ & $1{,}606$ & $1{,}680$ & $1{,}792$ & $2{,}024$ \\
    Max prompt len. & $4{,}070$ & $1{,}652$ & $1{,}838$ & $2{,}024$ & $2{,}210$ & $2{,}675$ \\
    \rowcolor{tblstripe}
    Mean ms/frame & $182.2$ & $174.7$ & $175.7$ & $181.5$ & $180.7$ & $181.7$ \\
    \#Re-encodings & $0$ & $7$ & $4$ & $3$ & $2$ & $1$ \\
    \rowcolor{tblstripe}
    Re-enc.\ total (ms) & $0$ & $99.8$ & $73.8$ & $68.3$ & $48.3$ & $87.0$ \\
    \bottomrule
  \end{tabular}}
\end{table}

\begin{table}[t]
  \centering
  \footnotesize
  \setlength{\tabcolsep}{4pt}
  \renewcommand{\arraystretch}{1.10}
  \caption{Re-encoding overhead vs.\ $W$. \% total: re-encoding time / wall time. Per step: mean ms per refresh.}
  \label{tab:swr_reencode}
  \fitwidth{%
  \begin{tabular}{@{}lccccc@{}}
    \toprule
    \rowcolor{tblheadbg}
    $W$ & \#Re & Total (ms) & \% total & Per step (ms) \\
    \midrule
    $4$ & $7$ & $99.8$ & $1.70\%$ & $14.3$ \\
    \rowcolor{tblstripe}
    $6$ & $4$ & $73.8$ & $1.29\%$ & $18.5$ \\
    $8$ & $3$ & $68.3$ & $1.17\%$ & $22.8$ \\
    \rowcolor{tblstripe}
    $10$ & $2$ & $48.3$ & $0.84\%$ & $24.1$ \\
    $15$ & $1$ & $87.0$ & $1.52\%$ & $87.0$ \\
    \bottomrule
  \end{tabular}}
\end{table}

\begin{table}[t]
  \centering
  \footnotesize
  \setlength{\tabcolsep}{4pt}
  \renewcommand{\arraystretch}{1.10}
  \caption{Long-horizon prompt scaling: AR max prompt extrapolates linearly with rollout length; SWR ($W{=}6$) caps at the 30-frame plateau. Slow-down columns are indicative.}
  \label{tab:swr_scale}
  \fitwidth{%
  \begin{tabular}{@{}lccccc@{}}
    \toprule
    \rowcolor{tblheadbg}
    Frames & AR max prompt & SWR $W{=}6$ max & Prompt drop & AR slow-down & SWR slow-down \\
    \midrule
    $30$ & $4{,}070$ & $1{,}838$ & $54.8\%$ & ${\sim}7\%$ & ${\sim}0\%$ \\
    \rowcolor{tblstripe}
    $100$ & ${\sim}12{,}673$ & $1{,}838$ & $85.5\%$ & ${\sim}25$--$30\%$ & ${\sim}0\%$ \\
    $300$ & ${\sim}37{,}273$ & $1{,}838$ & $95.1\%$ & OOM risk & ${\sim}0\%$ \\
    \bottomrule
  \end{tabular}}
\end{table}

\clearpage
\section{SWR qualitative visualizations}\label{app:swr_case_figs}

This section provides qualitative case studies comparing default autoregressive decoding with SWR under different refresh windows.
In each figure, columns show sampled frames along the rollout, while rows compare Ground Truth, Default AR, and SWR with different window sizes $W$.
The highlighted row marks the visually best SWR variant for that episode.
Consistent with the quantitative results, periodic refresh generally reduces long-horizon drift and better preserves task progression and object configuration, although the best per-episode $W$ can vary with the manipulation trajectory.

\begin{figure}[t]
  \centering
  \includegraphics[width=\linewidth]{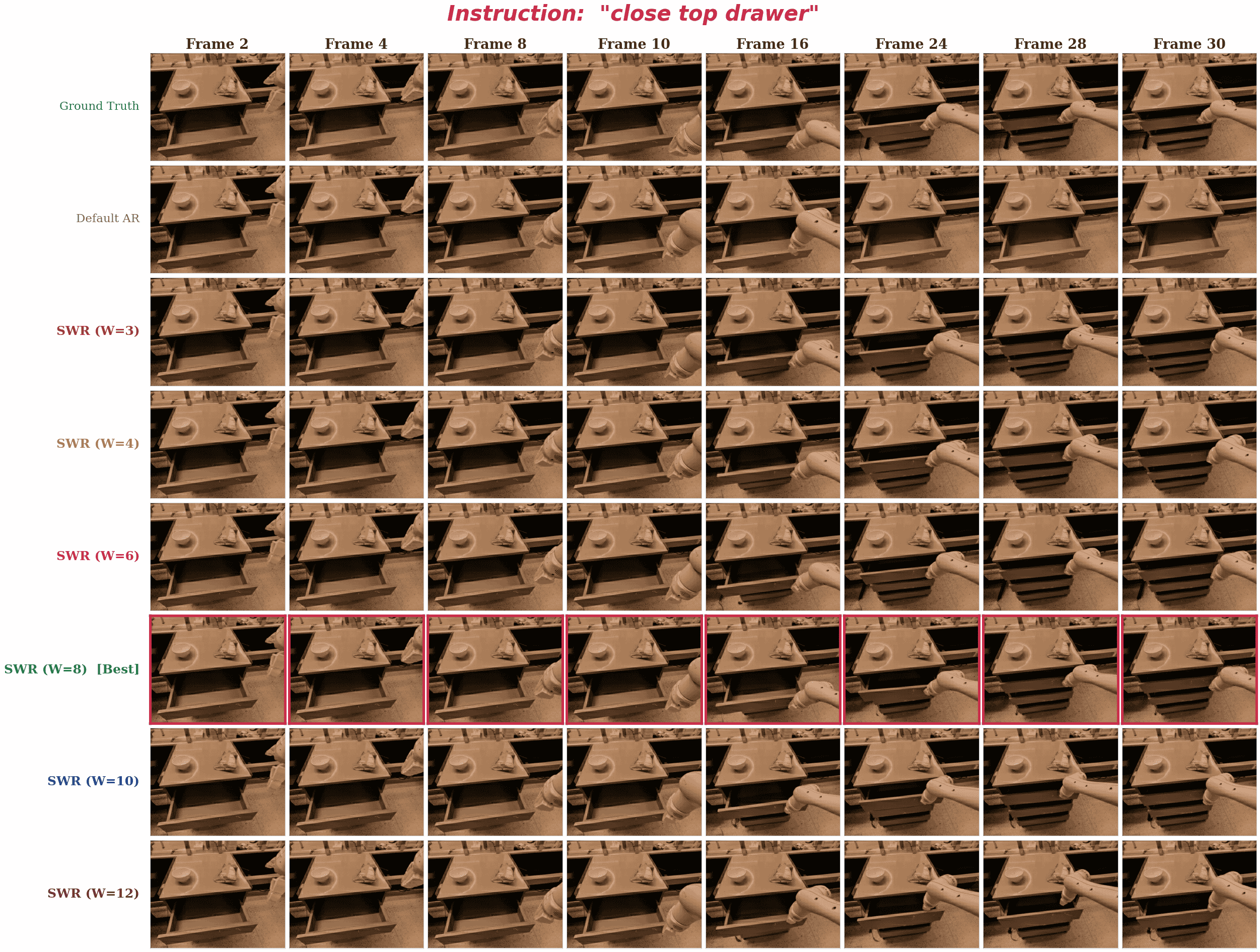}
  \caption{Qualitative comparison for the instruction \textit{``close top drawer''}. The highlighted row denotes the visually best SWR result for this episode.}
  \label{fig:swr_case_1}
\end{figure}

\begin{figure}[t]
  \centering
  \includegraphics[width=\linewidth]{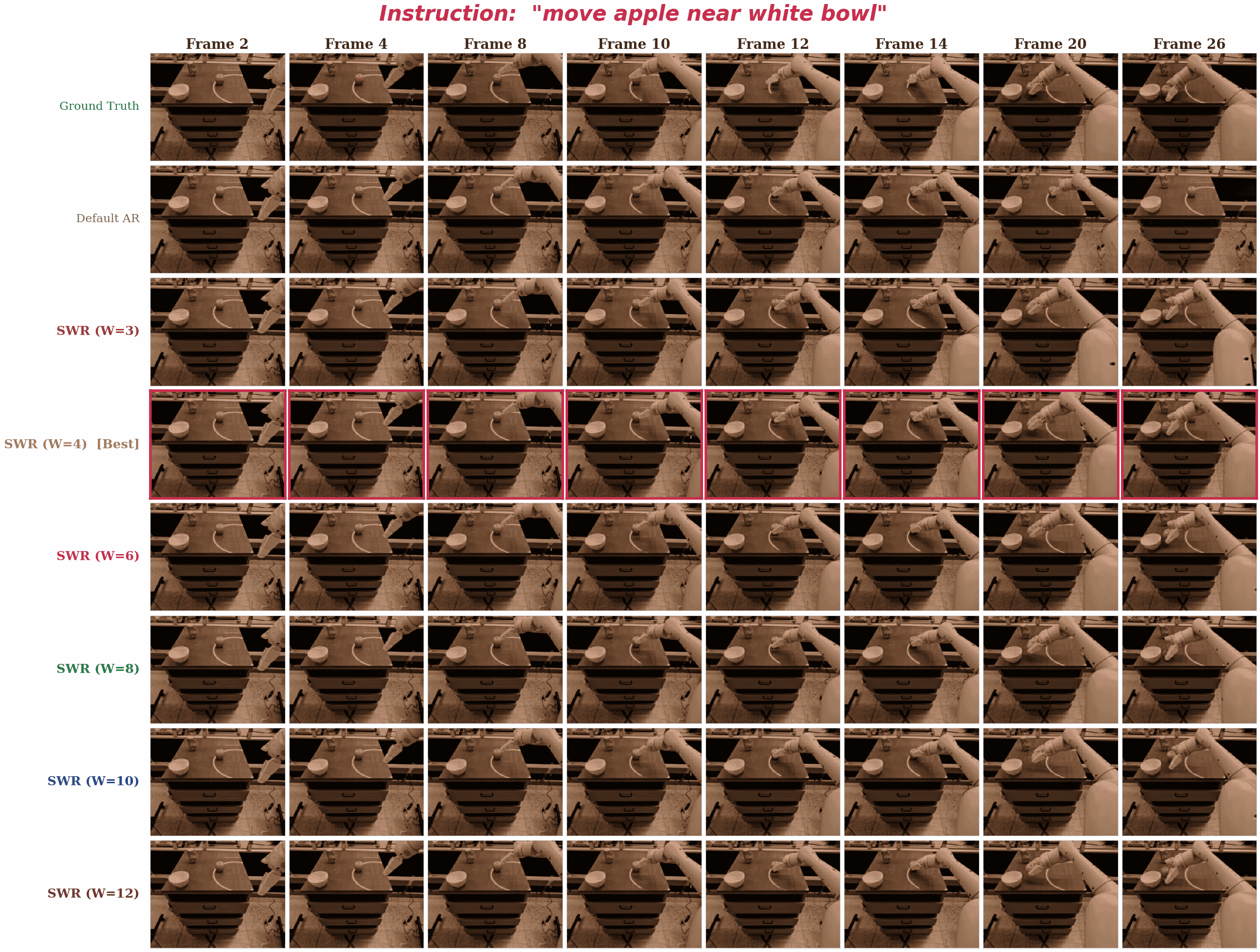}
  \caption{Qualitative comparison for the instruction \textit{``move apple near white bowl''}. The highlighted row denotes the visually best SWR result for this episode.}
  \label{fig:swr_case_2}
\end{figure}

\begin{figure}[t]
  \centering
  \includegraphics[width=\linewidth]{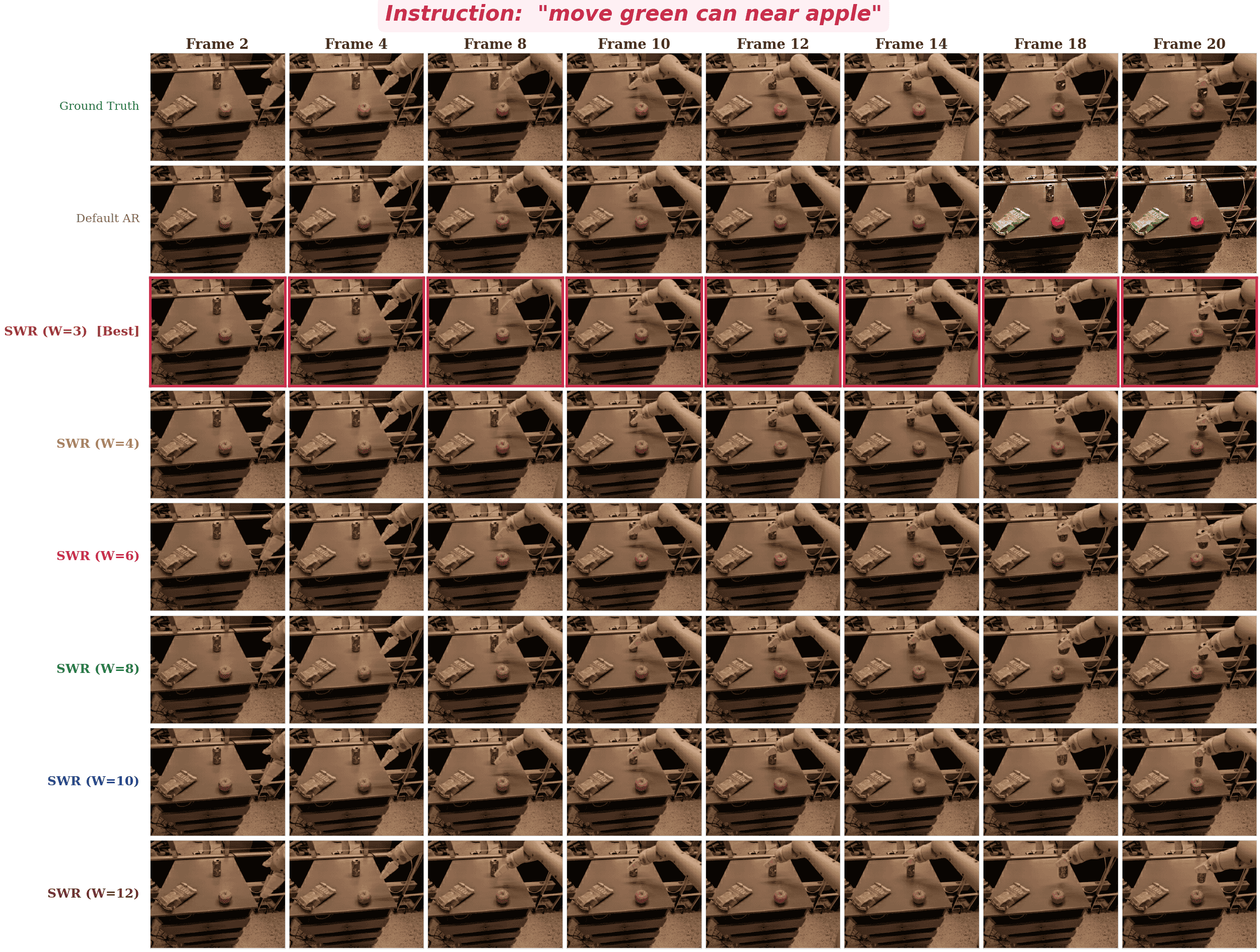}
  \caption{Qualitative comparison for the instruction \textit{``move green can near apple''}. The highlighted row denotes the visually best SWR result for this episode.}
  \label{fig:swr_case_3}
\end{figure}

\begin{figure}[t]
  \centering
  \includegraphics[width=\linewidth]{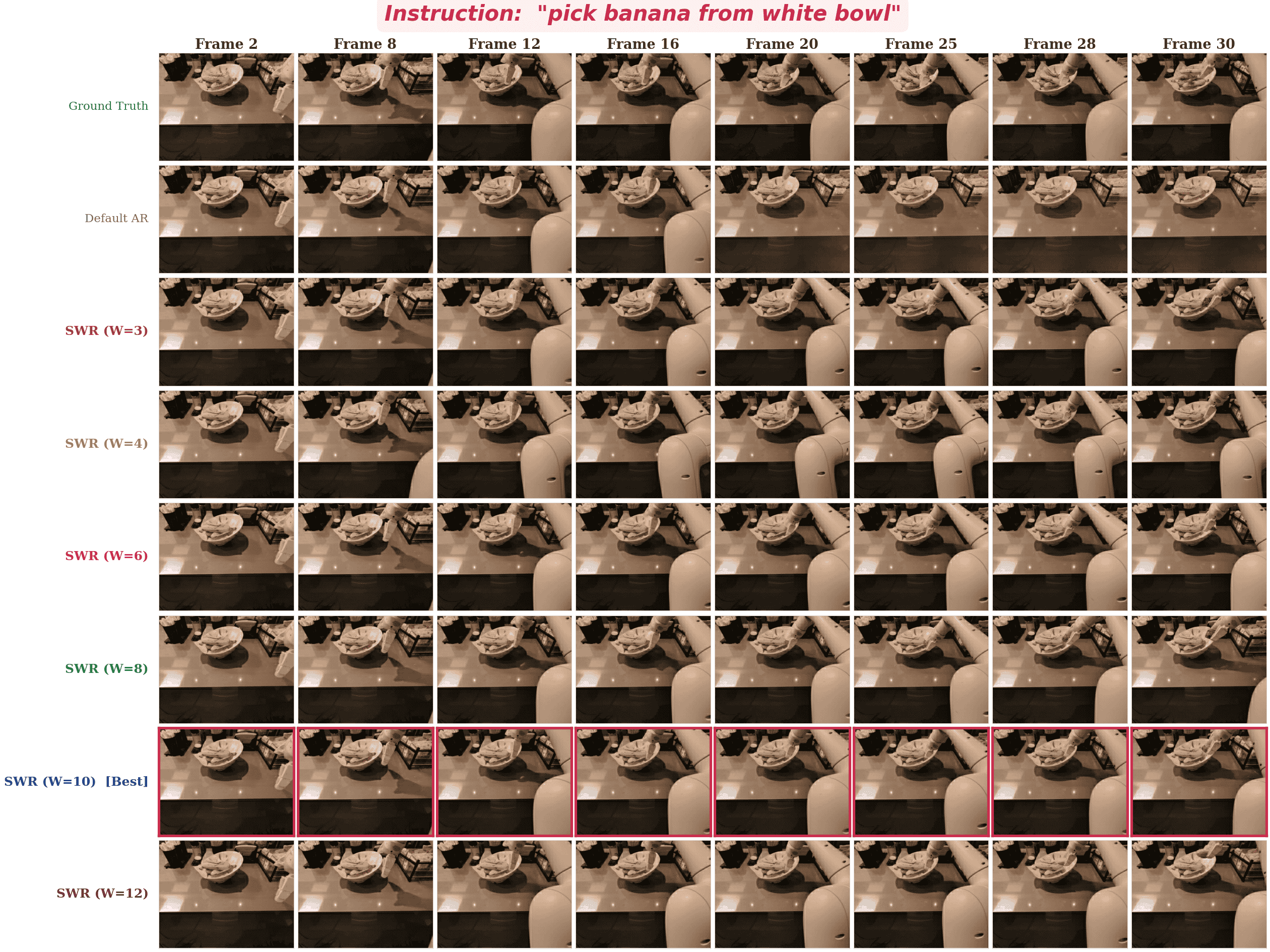}
  \caption{Qualitative comparison for the instruction \textit{``pick banana from white bowl''}. The highlighted row denotes the visually best SWR result for this episode.}
  \label{fig:swr_case_4}
\end{figure}

\clearpage
\section{Evaluation Metrics}\label{app:eval_metrics}

This section provides a detailed description of all evaluation metrics used in our experiments.

\subsection{Semantic and Physical Alignment Metrics}\label{app:metric_semantic}

The \textsc{RoboAlign-Judge} evaluates generated videos along six fine-grained dimensions, each scored on a bounded scale:

\begin{itemize}[leftmargin=1.5em, itemsep=0.15em]
  \item \textbf{Instruction Following (IF)} $\in [0, 3]$: Whether the generated video faithfully executes the language instruction.
  \item \textbf{Manipulation Success (MS)} $\in [0, 2]$: Whether the robot successfully completes the target manipulation task.
  \item \textbf{Action--Outcome Consistency (AO)} $\in [0, 1]$: Whether the observed outcome is consistent with the executed action.
  \item \textbf{Temporal Consistency (TC)} $\in [0, 1]$: Whether the video maintains smooth and temporally consistent transitions.
  \item \textbf{Contact Realism (CR)} $\in [0, 1]$: Whether contact events between the robot and objects appear physically plausible.
  \item \textbf{Physics Adherence (PA)} $\in [0, 2]$: Whether the video obeys basic physical laws (gravity, rigidity, collision response).
\end{itemize}

\noindent The aggregate score is the sum of the raw six-dimensional scores, yielding a maximum of $10.0$. When these scores are used for student distillation, each dimension is normalized independently to $[0,1]$ before regression.

\subsection{Pixel-Level Reconstruction Metrics}\label{app:metric_pixel}

We report four standard metrics computed over the full image:

\begin{itemize}[leftmargin=1.5em, itemsep=0.15em]
  \item \textbf{MSE}: Mean Squared Error between generated and ground-truth frames (pixel values normalized to $[0,1]$).
  \item \textbf{PSNR}: Peak Signal-to-Noise Ratio, defined as $\mathrm{PSNR} = 10\log_{10}(1/\mathrm{MSE})$.
  \item \textbf{SSIM}~\cite{wang2004image}: Structural Similarity Index, measuring luminance, contrast, and structural similarity.
  \item \textbf{LPIPS}~\cite{zhang2018unreasonable}: Learned Perceptual Image Patch Similarity, measuring perceptual distance using deep features.
\end{itemize}

\subsection{Motion-Mask-Based ROI Metrics}\label{app:roi_metrics}

\paragraph{Motivation.}
In robot manipulation videos, the background typically remains static while only the robot arm and manipulated objects undergo significant motion. As a result, global pixel-level metrics (e.g., PSNR, SSIM) are dominated by the large static background, which is trivially easy to reconstruct. This dilutes the evaluation signal from the critical dynamic interaction regions. To address this, we introduce \textbf{Region-of-Interest (ROI) metrics} that focus exclusively on the dynamic foreground.

\paragraph{Motion mask extraction.}
We derive the ROI from ground-truth video sequences using temporal frame differencing. Specifically, for each frame $x_t$, we compute the average absolute pixel difference against neighboring frames within a temporal window of size $\tau$:
\begin{equation}\label{eq:frame_diff}
  D_t(i,j) = \frac{1}{|\mathcal{N}_t|} \sum_{t' \in \mathcal{N}_t} \frac{1}{3}\sum_{c=1}^{3} |x_t^c(i,j) - x_{t'}^c(i,j)|,
\end{equation}
where $\mathcal{N}_t = \{t' : 0 < |t'-t| \le \tau,\; 1 \le t' \le T\}$ and $c$ indexes the RGB channels. The binary motion mask is obtained by thresholding:
\begin{equation}\label{eq:motion_mask}
  M_t(i,j) = \mathbf{1}[D_t(i,j) > \theta],
\end{equation}
followed by morphological closing and dilation (elliptical kernel, size $k$) to fill small gaps and include surrounding context. In our experiments, we use $\tau{=}3$, $\theta{=}15$, and $k{=}15$. The \textbf{union mask} $M_{\cup} = \max_t M_t$ aggregates all per-frame masks into a single ROI that covers the entire dynamic region across the episode.

\paragraph{Illustration.}
Figure~\ref{fig:motion_mask} visualizes the motion mask extraction process on a representative RT-1 episode. The left panel shows the original ground-truth frame; the right panel overlays the extracted motion mask (highlighted region) on the frame. The mask accurately captures the robot arm and the manipulated object while excluding the static background, confirming that ROI metrics focus evaluation on the most task-relevant regions.

\begin{figure}[t]
  \centering
  \includegraphics[width=0.85\linewidth]{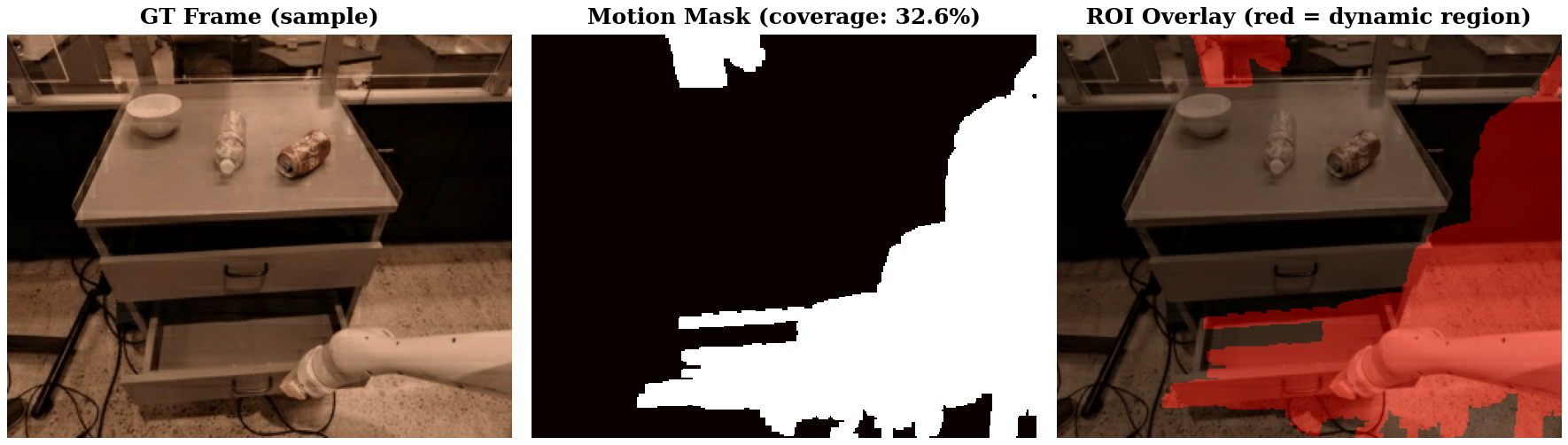}
  \caption{\textbf{Motion mask visualization.} The extracted motion mask (highlighted) isolates the dynamic interaction region---the robot arm and manipulated object---from the static background. ROI metrics are computed exclusively within this masked region, providing a more faithful assessment of generation quality for task-critical areas.}
  \label{fig:motion_mask}
\end{figure}

\paragraph{ROI metric definitions.}
Given a ground-truth frame $x_t$, a generated frame $\hat{x}_t$, and the binary mask $M_t$ (or $M_{\cup}$), the ROI metrics are defined as:

\begin{itemize}[leftmargin=1.5em, itemsep=0.15em]
  \item \textbf{ROI-MSE}: $\displaystyle \mathrm{ROI\text{-}MSE} = \frac{\sum_{i,j} M(i,j) \cdot \|x_t(i,j) - \hat{x}_t(i,j)\|^2}{3 \cdot \sum_{i,j} M(i,j)}$
  \item \textbf{ROI-PSNR}: $\displaystyle \mathrm{ROI\text{-}PSNR} = 10\log_{10}\!\left(\frac{1}{\mathrm{ROI\text{-}MSE}}\right)$
  \item \textbf{ROI-SSIM}: Local SSIM map computed over the full image, then averaged only within the masked region.
  \item \textbf{ROI-LPIPS}: LPIPS computed by masking out the static background (setting it to zero in both ground-truth and generated frames) before passing through the perceptual network.
\end{itemize}

\paragraph{ROI coverage.}
We additionally report \textbf{ROI coverage}, defined as the fraction of pixels in the union mask: $|M_{\cup}| / (H \times W)$. This provides context for interpreting ROI metrics---a typical RT-1 episode has ROI coverage of approximately $30$--$35\%$, confirming that the majority of the image is static background.

\clearpage
\section{\textsc{RoboAlign-Judge}: Multimodal Teacher Judge Training Details}
\label{app:roboalign_judge}

This appendix provides a comprehensive account of the training pipeline, data construction, and evaluation of \textsc{RoboAlign-Judge}---the multimodal teacher judge at the core of \method{}'s reward-aligned post-training framework (\S\ref{sec:reward}).
\textsc{RoboAlign-Judge} is instantiated by fine-tuning Qwen3-VL-8B-Thinking~\cite{bai2025qwen25vl} with LoRA~\cite{hu2022lora}, producing structured six-dimensional quality scores together with chain-of-thought reasoning for robot manipulation videos.

\subsection{Motivation and Design Rationale}
\label{app:judge_motivation}

As discussed in \S\ref{sec:reward}, standard low-level metrics (MSE, LPIPS, SSIM) are poorly aligned with the properties that matter for robot world models: instruction correctness, physical plausibility, and action--outcome consistency.
A multimodal judge that can assess these high-level properties provides much richer supervision for RL post-training.
We choose Qwen3-VL-8B-Thinking as the teacher backbone for three reasons:
\begin{enumerate}[leftmargin=*,nosep,label=(\roman*)]
    \item \textbf{Native video understanding.} Qwen3-VL supports multi-frame visual inputs with dynamic resolution, enabling direct processing of robot manipulation video sequences without frame-level feature extraction.
    \item \textbf{Built-in reasoning.} The ``Thinking'' variant provides structured chain-of-thought reasoning via \texttt{<think>} tokens, which naturally aligns with our requirement for interpretable, dimension-wise scoring with explicit justification.
    \item \textbf{Parameter efficiency.} At 8B parameters, the model is large enough to capture nuanced physical and semantic judgments, yet small enough for efficient LoRA fine-tuning and batch inference during teacher labeling.
\end{enumerate}

\subsection{Training Data Construction}
\label{app:judge_data}

The training data for \textsc{RoboAlign-Judge} forms the annotation component of \textsc{RobotWorldBench} (Eq.~\ref{eq:bench}), which spans four robot datasets and 10{,}000 annotated video--instruction pairs in total. In this appendix, we summarize two representative construction pipelines used in the corpus and report the cross-dataset composition used in the current study.

\paragraph{Part A: Synthetic degradation with rule-based annotation.}
As one representative source within the broader corpus, we extract 500 ground-truth (GT) episodes from the RT-1 dataset~\cite{brohan2023rt1}, each containing a robot manipulation video and its corresponding language instruction.
For each GT video, we generate two degraded variants by randomly sampling 1--3 degradation types from a pool of 10 operations spanning four categories:
\begin{itemize}[leftmargin=*,nosep]
    \item \textbf{Temporal} (3 types): frame shuffling, frame dropping, segment reversal---simulating temporal incoherence and causal violations.
    \item \textbf{Visual} (4 types): Gaussian blur, Gaussian noise, color jitter, resolution degradation---simulating perceptual quality loss and compression artifacts.
    \item \textbf{Spatial} (1 type): random per-frame spatial shifts---simulating camera instability and spatial inconsistency.
    \item \textbf{Semantic} (2 types): video truncation, end-frame freezing---simulating task incompleteness and action termination failures.
\end{itemize}
Each degradation is applied at a severity level sampled from $\{\text{mild}, \text{moderate}, \text{severe}\}$ with probabilities $\{0.2, 0.4, 0.4\}$, biasing toward harder examples.
Scores are computed deterministically: each degradation type has predefined per-dimension deduction fractions that scale with severity, ensuring consistent and reproducible annotations.
This yields ${\sim}$994 augmented videos.
An additional 100 GT videos are retained as perfect-score references (total score $= 10$), providing the upper anchor for the score distribution.
Figure~\ref{fig:judge_synth_degradation_examples} shows representative examples from this synthetic-degradation pipeline, illustrating how the rule-based transformations create controlled failures across multiple error types and severity levels.

\begin{figure}[t]
\centering
\includegraphics[width=\linewidth]{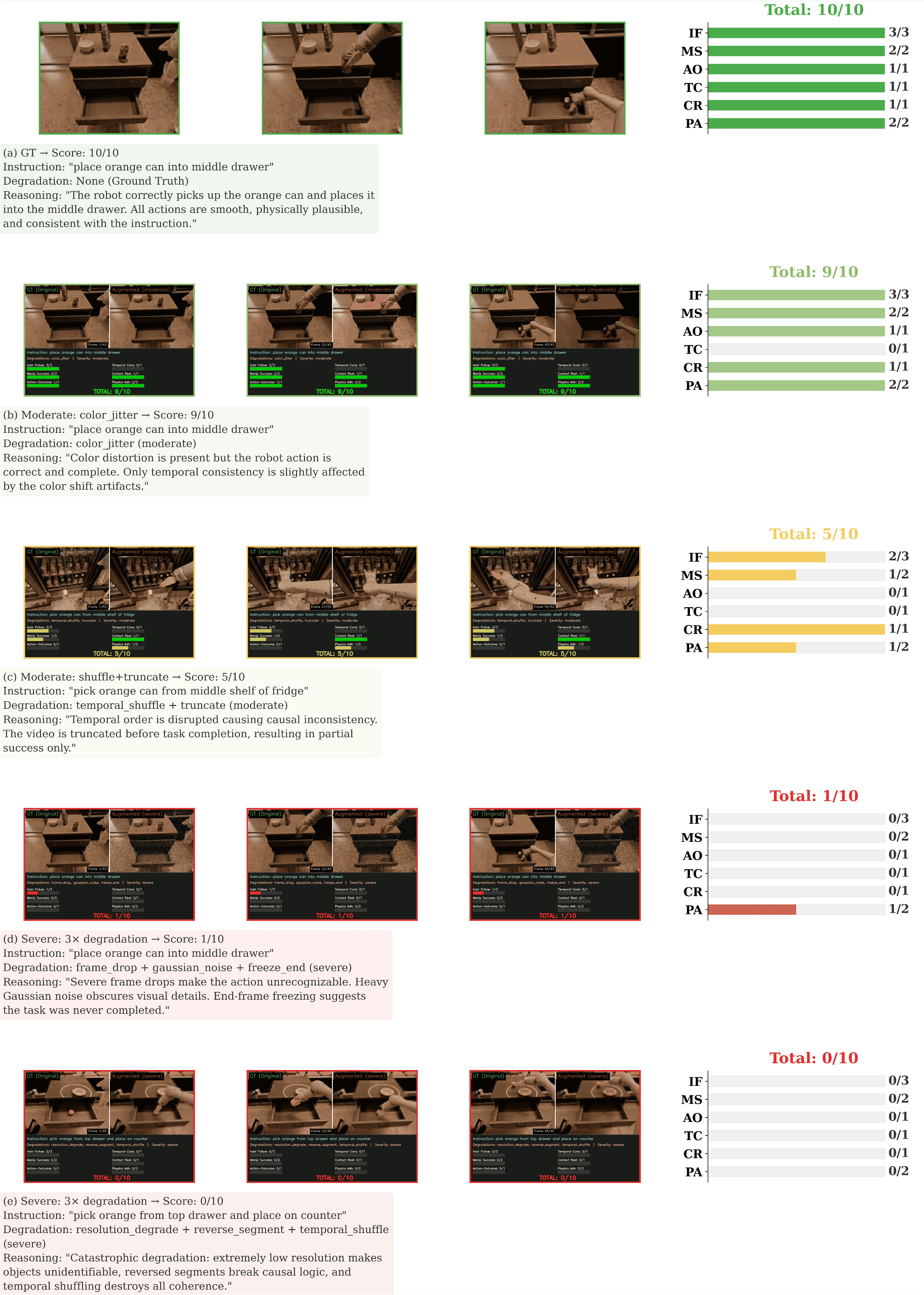}
\caption{Representative examples from Part~A synthetic degradation with rule-based annotation. Starting from a ground-truth robot manipulation video, we apply controlled temporal, visual, spatial, and semantic perturbations at different severity levels to generate training examples with reproducible score deductions.}
\label{fig:judge_synth_degradation_examples}
\end{figure}

\paragraph{Part B: Real I2V model outputs with human annotation.}
Synthetic degradation cannot fully capture the failure modes of real generative models---\emph{e.g.}, hallucinated objects, semantic drift, physically implausible contact dynamics, or mode collapse.
To address this gap, we collect generated videos from 10 diverse image-to-video (I2V) models evaluated on RT-1 prompts, spanning both general-purpose and robotics-specific architectures (Table~\ref{tab:judge_i2v_models}).
Each video is scored by trained human annotators following the identical six-dimension rubric used in Part~A.

\begin{table}[h]
\centering
\small
\setlength{\tabcolsep}{4pt}
\renewcommand{\arraystretch}{1.10}
\caption{I2V models used for collecting human-annotated training data in \textsc{RobotWorldBench}. Models span diffusion-based, transformer-based, and autoregressive architectures across general and robotics domains.}
\label{tab:judge_i2v_models}
\fitwidth{%
\begin{tabular}{@{}llll@{}}
\toprule
\rowcolor{tblheadbg}
\textbf{Model} & \textbf{Architecture} & \textbf{Domain} & \textbf{Key Characteristic} \\
\midrule
\rowcolor{tblstripe}
CogVideoX       & Diffusion-based I2V  & General  & 3D causal VAE + DiT \\
HunyuanVideo     & Diffusion-based I2V  & General  & Dual-stream DiT \\
\rowcolor{tblstripe}
I2VGen-XL        & Cascaded diffusion   & General  & Two-stage generation \\
LTX-Video        & Latent Transformer   & General  & Efficient latent space \\
\rowcolor{tblstripe}
SVD              & Diffusion-based I2V  & General  & Temporal layer fine-tuning \\
OpenSora         & DiT-based I2V        & General  & Open-source Sora replica \\
\rowcolor{tblstripe}
DynamiCrafter    & Diffusion-based I2V  & General  & Image-conditioned dynamics \\
\midrule
\rowcolor{tblstripe}
iVideoGPT       & Autoregressive GPT   & Robotics & Token-based world model \\
RoboDreamer      & Goal-conditioned diffusion & Robotics & Compositional planning \\
\rowcolor{tblstripe}
Vid2World        & Diffusion world model & Robotics & Interactive simulation \\
\bottomrule
\end{tabular}}
\end{table}

\paragraph{Data complementarity and balancing.}
The two sources serve complementary roles: Part~A provides controllable coverage across score ranges with deterministic labels, while Part~B introduces authentic generative artifacts that improve robustness to real-world evaluation scenarios. Beyond these RT-1-centered examples, the full \textsc{RobotWorldBench} training corpus used in the current study spans four robot datasets and contains 10{,}000 annotated video--instruction pairs for judge training and reward distillation. After score-bin balancing and corpus aggregation, the resulting training set maintains broad coverage over the $[0, 10]$ score range, which is important for preventing the judge from developing score-range biases that would propagate through reward distillation into RL post-training.

\begin{table}[htbp]
  \centering
  \small
  \setlength{\tabcolsep}{8pt}
  \renewcommand{\arraystretch}{1.1}
  \caption{\textsc{RobotWorldBench} data statistics.}
  \label{tab:judge_data_stats_template}
  \fitwidth{%
  \begin{tabular}{lcccc}
    \toprule
    \rowcolor{tblheadbg}
    Dataset & GT & Synthetic & Generated & Total Pairs \\
    \midrule
    \rowcolor{tblstripe}
    RT-1 & 150 & 450 & 600 & 1,200 \\
    BridgeData V2 & 450 & 1,050 & 1,500 & 3,000 \\
    \rowcolor{tblstripe}
    CALVIN & 300 & 700 & 1,000 & 2,000 \\
    LIBERO & 550 & 1,350 & 1,900 & 3,800 \\
    \midrule
    \rowcolor{mydarkblue!6}
    \textbf{Total} & \textbf{1,450} & \textbf{3,550} & \textbf{5,000} & \textbf{10,000} \\
    \bottomrule
  \end{tabular}}
\end{table}

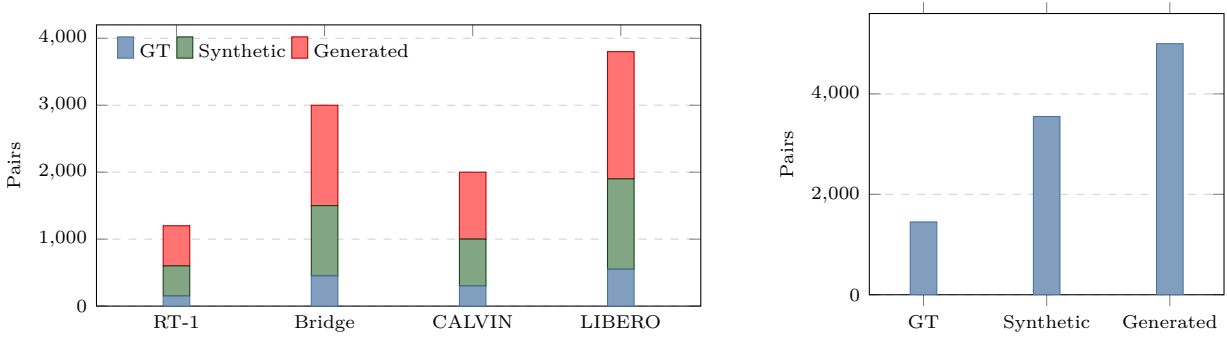
\begin{figure*}[t]
\centering
\begin{minipage}[t]{0.58\textwidth}
\centering
\begin{tikzpicture}
\begin{axis}[
  ybar stacked,
  width=\linewidth,
  height=5.3cm,
  ymin=0,
  ymax=4200,
  ylabel={Pairs},
  symbolic x coords={RT-1,Bridge,CALVIN,LIBERO},
  xtick=data,
  xticklabel style={font=\scriptsize},
  ylabel style={font=\scriptsize},
  yticklabel style={font=\scriptsize},
  legend style={
    font=\scriptsize,
    at={(0.02,0.98)},
    anchor=north west,
    draw=none,
    fill=none,
    legend columns=3
  },
  enlarge x limits=0.18,
  ymajorgrids=true,
  grid style={dashed,gray!35},
]
\addplot+[fill=mydarkblue!55, draw=mydarkblue!80] coordinates {(RT-1,150) (Bridge,450) (CALVIN,300) (LIBERO,550)};
\addplot+[fill=codegreen!55, draw=codegreen!70!black] coordinates {(RT-1,450) (Bridge,1050) (CALVIN,700) (LIBERO,1350)};
\addplot+[fill=red!55, draw=red!70!black] coordinates {(RT-1,600) (Bridge,1500) (CALVIN,1000) (LIBERO,1900)};
\legend{GT,Synthetic,Generated}
\end{axis}
\end{tikzpicture}
\end{minipage}\hfill
\begin{minipage}[t]{0.38\textwidth}
\centering
\begin{tikzpicture}
\begin{axis}[
  ybar,
  width=\linewidth,
  height=5.3cm,
  ymin=0,
  ymax=5600,
  ylabel={Pairs},
  symbolic x coords={GT,Synthetic,Generated},
  xtick=data,
  xticklabel style={font=\scriptsize},
  ylabel style={font=\scriptsize},
  yticklabel style={font=\scriptsize},
  enlarge x limits=0.22,
  ymajorgrids=true,
  grid style={dashed,gray!35},
]
\addplot+[fill=mydarkblue!55, draw=mydarkblue!80] coordinates {(GT,1450) (Synthetic,3550) (Generated,5000)};
\end{axis}
\end{tikzpicture}
\end{minipage}
\caption{\textbf{Corpus composition of \textsc{RobotWorldBench}.} \textbf{Left:} per-dataset breakdown by source type. \textbf{Right:} aggregate source composition across the full 10,000-pair annotation corpus. Generated samples form the largest component, complemented by rule-based synthetic degradations and a smaller set of ground-truth reference videos.}
\label{fig:judge_data_stats}
\end{figure*}

\FloatBarrier
\subsection{Six-Dimension Scoring Rubric}
\label{app:judge_rubric}

\textsc{RoboAlign-Judge} evaluates each video along six complementary dimensions, grouped into \emph{task alignment} and \emph{physical realism} categories (Table~\ref{tab:judge_rubric}).
The rubric is designed to capture the properties most relevant to robot world-model utility: whether the predicted future correctly reflects the intended manipulation, and whether the physical dynamics are plausible enough to support downstream planning.

\begin{table}[h]
\centering
\small
\setlength{\tabcolsep}{3pt}
\renewcommand{\arraystretch}{1.10}
\caption{Six-dimension scoring rubric for \textsc{RoboAlign-Judge}. Dimensions are grouped by task alignment (semantic correctness) and physical realism (dynamics plausibility).}
\label{tab:judge_rubric}
\begin{tabular}{@{}llcp{0.50\linewidth}@{}}
\toprule
\rowcolor{tblheadbg}
\textbf{Category} & \textbf{Dimension} & \textbf{Range} & \textbf{Scoring Criteria} \\
\midrule
\rowcolor{tblstripe}
& Instruction Following (IF) & 0--3 & 0: completely unrelated action; 1: vaguely related but wrong; 2: correct action but incomplete/imprecise; 3: perfectly follows instruction \\
& Manipulation Success (MS) & 0--2 & 0: task completely failed; 1: partial success (object moved but not to target); 2: full success \\
\rowcolor{tblstripe}
& Action-Outcome Consist. (AO) & 0--1 & 0: actions and outcomes are inconsistent; 1: logically consistent \\
\midrule
\rowcolor{tblstripe}
& Temporal Consistency (TC) & 0--1 & 0: severe temporal artifacts (flickering, jumps); 1: smooth and temporally coherent \\
& Contact Realism (CR) & 0--1 & 0: unrealistic contacts (penetration, floating); 1: natural and physically realistic \\
\rowcolor{tblstripe}
& Physics Adherence (PA) & 0--2 & 0: severe physics violations; 1: minor issues but mostly plausible; 2: fully physically plausible \\
\midrule
\rowcolor{mydarkblue!6}
& \textbf{Total} & \textbf{0--10} & Sum of all six dimensions \\
\bottomrule
\end{tabular}
\end{table}

\subsection{Model Architecture and LoRA Fine-Tuning}
\label{app:judge_training}

\paragraph{Architecture.}
\textsc{RoboAlign-Judge} is built on Qwen3-VL-8B-Thinking, a multimodal large language model with native multi-frame video understanding.
We apply Low-Rank Adaptation (LoRA) to all linear projection layers in both the attention blocks ($\mathbf{W}_q, \mathbf{W}_k, \mathbf{W}_v, \mathbf{W}_o$) and the MLP blocks ($\mathbf{W}_{\text{gate}}, \mathbf{W}_{\text{up}}, \mathbf{W}_{\text{down}}$), totaling 7 target modules per Transformer layer.
This broad LoRA coverage ensures that both the visual--language alignment and the reasoning capacity are adapted to the robot manipulation evaluation domain.
Table~\ref{tab:judge_training_config} summarizes the complete training configuration.

\begin{table}[h]
\centering
\small
\setlength{\tabcolsep}{4pt}
\renewcommand{\arraystretch}{1.10}
\caption{Training configuration for \textsc{RoboAlign-Judge} LoRA fine-tuning.}
\label{tab:judge_training_config}
\fitwidth{%
\begin{tabular}{@{}ll@{}}
\toprule
\rowcolor{tblheadbg}
\textbf{Hyperparameter} & \textbf{Value} \\
\midrule
\rowcolor{tblstripe}
Base Model & Qwen3-VL-8B-Thinking \\
Trainable Parameters & ${\sim}$160M (LoRA only) / 8.3B total \\
\rowcolor{tblstripe}
LoRA Rank ($r$) & 64 \\
LoRA Scaling ($\alpha$) & 128 \\
\rowcolor{tblstripe}
LoRA Dropout & 0.05 \\
Target Modules & \texttt{q\_proj, k\_proj, v\_proj, o\_proj, gate\_proj, up\_proj, down\_proj} \\
\midrule
\rowcolor{tblstripe}
Training Epochs & 5 \\
Optimizer & AdamW ($\beta_1{=}0.9$, $\beta_2{=}0.999$, $\epsilon{=}10^{-8}$) \\
\rowcolor{tblstripe}
Learning Rate & $2 \times 10^{-4}$ \\
LR Schedule & Cosine decay with linear warmup \\
\rowcolor{tblstripe}
Warmup Ratio & 0.1 \\
Effective Batch Size & 32 ($1 \times 4$ gradient accumulation $\times$ 8 GPUs) \\
\rowcolor{tblstripe}
Max Sequence Length & 4,096 tokens \\
Precision & BFloat16 (mixed precision) \\
\rowcolor{tblstripe}
Gradient Checkpointing & \checkmark \\
Hardware & 8$\times$ NVIDIA A100 40GB \\
\rowcolor{tblstripe}
Training Time & ${\sim}$31 minutes \\
\bottomrule
\end{tabular}}
\end{table}

\paragraph{Input format.}
Each training sample is structured as a three-turn conversation following the Qwen3-VL chat template:
\begin{enumerate}[leftmargin=*,nosep,label=(\arabic*)]
    \item \textbf{System prompt}: defines the evaluator role, the six scoring dimensions with detailed rubrics, and the required JSON output schema (see Appendix~\ref{app:judge_prompt} for the full prompt).
    \item \textbf{User message}: contains the language instruction $l$, the initial frame $x_0$ (as an image), and $N{=}8$ frames uniformly sampled from the generated video $v$.
    \item \textbf{Assistant response}: begins with \texttt{<think>}...\texttt{</think>} chain-of-thought reasoning that analyzes each dimension step by step, followed by a structured JSON object containing the six dimension scores and total.
\end{enumerate}
The training loss $\mathcal{L}_{\text{teacher}}$ (Eq.~\ref{eq:teacher}) is computed only on the assistant response tokens, ensuring that the model learns to produce both the reasoning trace and the structured scores while treating the system and user turns as context.

\paragraph{Training dynamics.}
The training loss drops rapidly in the first epoch (from 0.762 to 0.066), indicating fast adaptation of the LoRA parameters to the evaluation task.
Convergence is reached by epoch 3--4, with the final loss stabilizing at ${\sim}$0.016 by epoch 5.
The smooth convergence without loss spikes or oscillation confirms that the LoRA rank ($r{=}64$) provides sufficient capacity for this task, and that the learning rate schedule is well-calibrated.
Table~\ref{tab:judge_loss} reports the loss trajectory at key checkpoints, and Figure~\ref{fig:judge_training_loss} visualizes the full training curve.

\begin{table}[h]
\centering
\small
\setlength{\tabcolsep}{4pt}
\renewcommand{\arraystretch}{1.10}
\caption{Training loss trajectory for \textsc{RoboAlign-Judge}. The model converges within 3 epochs and stabilizes by epoch 5.}
\label{tab:judge_loss}
\fitwidth{%
\begin{tabular}{@{}cccccccccc@{}}
\toprule
\rowcolor{tblheadbg}
\textbf{Epoch} & 0.27 & 0.53 & 0.80 & 1.05 & 2.11 & 3.16 & 4.21 & 5.00 \\
\midrule
\rowcolor{tblstripe}
\textbf{Loss} & 0.762 & 0.209 & 0.066 & 0.036 & 0.025 & 0.021 & 0.018 & 0.016 \\
\bottomrule
\end{tabular}}
\end{table}

\begin{figure}[h]
\centering
\includegraphics[width=0.85\linewidth]{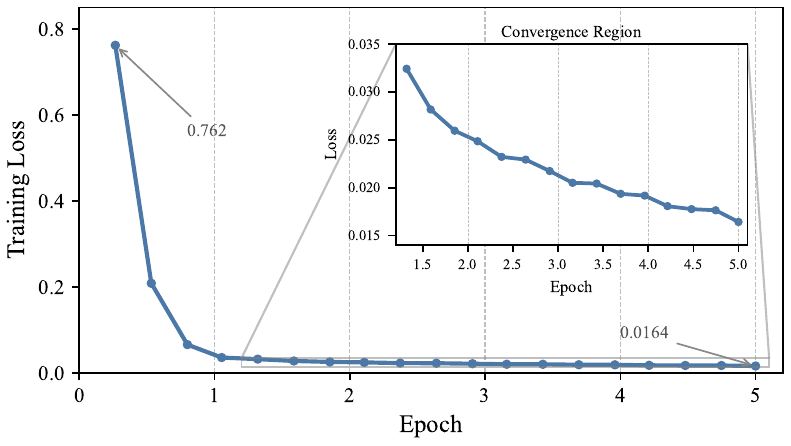}
\caption{Training loss curve for \textsc{RoboAlign-Judge} LoRA fine-tuning over 5 epochs. The inset shows the convergence region (epochs 1--5). Loss drops rapidly in the first epoch and stabilizes around 0.016.}
\label{fig:judge_training_loss}
\end{figure}

\FloatBarrier
\subsection{Chain-of-Thought Reasoning}
\label{app:judge_cot}

A distinctive feature of \textsc{RoboAlign-Judge} is its use of explicit chain-of-thought (CoT) reasoning before producing scores.
By leveraging the ``Thinking'' capability of Qwen3-VL, the judge first generates a detailed analysis within \texttt{<think>}...\texttt{</think>} tags, examining each evaluation dimension in sequence:
\begin{itemize}[leftmargin=*,nosep]
    \item Whether the robot's motion trajectory matches the instructed action;
    \item Whether the manipulation target is correctly identified and reached;
    \item Whether object contacts appear physically natural (no penetration or floating);
    \item Whether the video maintains temporal coherence without flickering or jumps;
    \item Whether the final state is consistent with the expected action outcome.
\end{itemize}
This reasoning trace serves two purposes: (1)~it improves scoring accuracy by forcing the model to attend to each dimension before committing to a score, and (2)~it provides interpretable justifications that can be inspected during reward debugging and used to diagnose failure modes in the world model's predictions.

\subsection{From Teacher Judge to Reward Signal}
\label{app:judge_to_reward}

\textsc{RoboAlign-Judge} serves as the teacher $f_\phi$ in the reward distillation pipeline (\S\ref{sec:reward}).
Given a language instruction $l$ and a generated video $v$, the judge produces a structured raw score vector $\hat{\mathbf{r}} = f_\phi(l, v)$ across the six dimensions, using the original rubric ranges defined in Table~\ref{tab:judge_rubric}. For student distillation, these raw teacher scores are normalized dimension-wise to $[0,1]$ before regression.
These teacher scores are used in two ways:

\begin{enumerate}[leftmargin=*,nosep,label=(\roman*)]
    \item \textbf{Offline labeling for student distillation.} The teacher scores a mixed corpus consisting of benchmark videos together with generated rollouts from baseline or current world models to create the distillation dataset $\{(l_i, v_i, f_\phi(l_i, v_i))\}$, which trains the lightweight student reward model $g_\psi$ via regression (Eq.~\ref{eq:distill}).
    \item \textbf{Online iterative calibration.} Every $K$ policy updates during GRPO post-training, fresh rollouts from the current world model are scored by the teacher and used to update the student, preventing reward hacking from distributional shift (Algorithm~\ref{alg:train}, Stage 4).
\end{enumerate}

The teacher's autoregressive decoding cost (${\sim}$2.8 seconds per video on a single A100) makes it impractical as a direct online reward.
The distilled student reduces this to a single forward pass (${\sim}$20\,ms), achieving $>$10$\times$ speedup while maintaining high correlation with teacher judgments.

\subsection{Comparison with Alternative VLM Judges}
\label{app:judge_vlm_comparison}

A central question is whether the main findings remain stable under independent model-based evaluation, rather than depending on a single in-domain judge. To answer this question, we conduct an external VLM-based validation by comparing \textsc{RoboAlign-Judge} against 8 alternative VLM judges---4 proprietary and 4 open-source---on a held-out test set of 50 human-annotated samples from \textsc{RobotWorldBench}.

\paragraph{Evaluation protocol.}
The 50 test samples are drawn from both data sources (25 from synthetic degradation, 25 from real I2V model outputs) and span the full $[0, 10]$ score range.
Each sample has human annotations across all six dimensions and is used solely for independent cross-checking rather than for training the main automatic evaluator.
All VLM judges receive the identical system prompt (Appendix~\ref{app:judge_prompt}), the same input format (instruction + initial frame + 8 sampled frames), and are asked to produce the same structured JSON output.
For proprietary models, we use the official API with default parameters; for open-source models, we use greedy decoding ($T{=}0$) to minimize variance; for \textsc{RoboAlign-Judge}, we report results with both greedy ($T{=}0$) and sampling ($T{=}0.6$, top-$p{=}0.95$) decoding.

\paragraph{Metrics.}
We evaluate judge quality along five complementary axes:
\begin{enumerate}[leftmargin=*,nosep,label=(\roman*)]
    \item \textbf{Pearson $\rho$}: linear correlation between predicted and human total scores, measuring absolute scoring accuracy.
    \item \textbf{Spearman $\rho_s$}: rank correlation between predicted and human total scores, measuring ordinal consistency---critical for reward-based ranking in GRPO.
    \item \textbf{Per-dimension MAE}: mean absolute error averaged across all six dimensions, measuring fine-grained scoring precision.
    \item \textbf{Pairwise Accuracy}: given all $\binom{50}{2} = 1{,}225$ video pairs, the fraction where the judge's relative ordering agrees with the human ranking---directly measuring the quality of the reward signal for preference-based optimization.
    \item \textbf{Inference Cost}: wall-clock time per video on a single NVIDIA A100 40GB GPU (or API latency for proprietary models), measuring practical feasibility as a teacher labeler.
\end{enumerate}

\paragraph{Results.}
Table~\ref{tab:judge_vlm_comparison} reports the full comparison. We view this table as an independent external cross-check of judge consistency on a held-out human-annotated subset, rather than as a replacement for human evaluation itself. Under this protocol, \textsc{RoboAlign-Judge} remains the most aligned with the held-out annotations while maintaining the same inference cost as its base model.

\begin{table}[h]
\centering
\small
\setlength{\tabcolsep}{3pt}
\renewcommand{\arraystretch}{1.10}
\caption{\textbf{Comparison of VLM judges on 50 human-annotated test samples from \textsc{RobotWorldBench}.} All zero-shot judges use the same prompt template. Best results per metric are \textbf{bolded}; second-best are \underline{underlined}. $\dagger$: API latency includes network overhead.}
\label{tab:judge_vlm_comparison}
\fitwidth{%
\begin{tabular}{@{}llccccc@{}}
\toprule
\rowcolor{tblheadbg}
\textbf{Judge} & \textbf{Type} & \textbf{Pearson $\rho$ $\uparrow$} & \textbf{Spearman $\rho_s$ $\uparrow$} & \textbf{MAE $\downarrow$} & \textbf{Pair Acc. $\uparrow$} & \textbf{Cost (s)} $\downarrow$ \\
\midrule
\multicolumn{7}{l}{\textit{Proprietary models (zero-shot)}} \\
\rowcolor{tblstripe}
GPT-4o              & API$^\dagger$ & 0.72 & 0.68 & 0.82 & 71.3\% & ${\sim}$8.5 \\
GPT-4.1             & API$^\dagger$ & \underline{0.75} & \underline{0.71} & \underline{0.76} & \underline{73.8\%} & ${\sim}$7.2 \\
\rowcolor{tblstripe}
Gemini 2.5 Pro      & API$^\dagger$ & 0.70 & 0.66 & 0.85 & 69.5\% & ${\sim}$6.8 \\
Claude 3.7 Sonnet   & API$^\dagger$ & 0.68 & 0.64 & 0.89 & 67.2\% & ${\sim}$9.1 \\
\midrule
\multicolumn{7}{l}{\textit{Open-source models (zero-shot)}} \\
\rowcolor{tblstripe}
Qwen3-VL-8B-Thinking & Local & 0.58 & 0.54 & 1.12 & 62.4\% & ${\sim}$2.8 \\
Qwen2.5-VL-7B       & Local & 0.52 & 0.48 & 1.28 & 58.1\% & ${\sim}$2.5 \\
\rowcolor{tblstripe}
InternVL2.5-8B       & Local & 0.49 & 0.45 & 1.35 & 56.7\% & ${\sim}$2.6 \\
LLaVA-OneVision-7B   & Local & 0.44 & 0.40 & 1.48 & 53.2\% & ${\sim}$2.4 \\
\midrule
\multicolumn{7}{l}{\textit{Fine-tuned (ours)}} \\
\rowcolor{mydarkblue!6}
\textbf{\textsc{RoboAlign-Judge}} & Local & \textbf{0.89} & \textbf{0.86} & \textbf{0.41} & \textbf{87.6\%} & ${\sim}$2.8 \\
\bottomrule
\end{tabular}}
\end{table}

\paragraph{Analysis.}
Several key observations emerge from Table~\ref{tab:judge_vlm_comparison} and Figure~\ref{fig:judge_vlm_scatter}:

\begin{enumerate}[leftmargin=*,nosep,label=(\arabic*)]
    \item \textbf{Fine-tuning dramatically outperforms zero-shot.}
    \textsc{RoboAlign-Judge} achieves Pearson $\rho = 0.89$, surpassing the best proprietary model GPT-4.1 ($\rho = 0.75$) by $+0.14$ and its own base model Qwen3-VL-8B-Thinking ($\rho = 0.58$) by $+0.31$.
    This $+53\%$ relative improvement over the base model demonstrates that domain-specific LoRA fine-tuning on robot manipulation data is essential---general VLMs lack the calibration needed for fine-grained physical and task-level assessment.

    \item \textbf{This serves as an independent model-based cross-check.}
    The Pairwise Accuracy of 87.6\% means that in ${\sim}$88\% of video pairs, \textsc{RoboAlign-Judge} correctly identifies which video is better---a prerequisite for effective GRPO optimization.
    In contrast, GPT-4.1 achieves only 73.8\%, meaning ${\sim}$26\% of pairwise comparisons would provide incorrect gradient signals during RL training.

    \item \textbf{Open-source zero-shot models are insufficient.}
    All open-source models score below $\rho = 0.60$ and Pairwise Accuracy below 63\%, with particularly poor performance on physics-related dimensions (TC, CR, PA).
    This is expected: these models are trained on general visual understanding tasks and lack exposure to the specific failure modes of robot manipulation videos.

    \item \textbf{Proprietary models provide complementary evidence but remain impractical at scale.}
    While GPT-4.1 and GPT-4o achieve reasonable correlation ($\rho \ge 0.72$), their API latency (${\sim}$7--9 seconds) and per-query cost make them infeasible for large-scale teacher labeling.
    Labeling 10,000 rollouts would cost ${\sim}\$500$--$\$1{,}000$ and take ${\sim}$20 hours sequentially, compared to ${\sim}\$0$ and ${\sim}$8 hours for \textsc{RoboAlign-Judge} on local GPUs with parallelism.

    \item \textbf{Per-dimension analysis reveals domain gaps.}
    Figure~\ref{fig:judge_vlm_perdim_mae} shows that the largest improvements from fine-tuning occur in \emph{Physics Adherence} (MAE: 1.42 $\to$ 0.38) and \emph{Temporal Consistency} (MAE: 1.18 $\to$ 0.29), precisely the dimensions that require understanding of robot-specific physical dynamics.
    General VLMs perform relatively better on \emph{Instruction Following} (a more semantic/linguistic dimension), but still lag behind the fine-tuned judge.
\end{enumerate}

\begin{figure}[h]
\centering
\begin{minipage}[t]{0.48\linewidth}
    \centering
    \includegraphics[width=\linewidth]{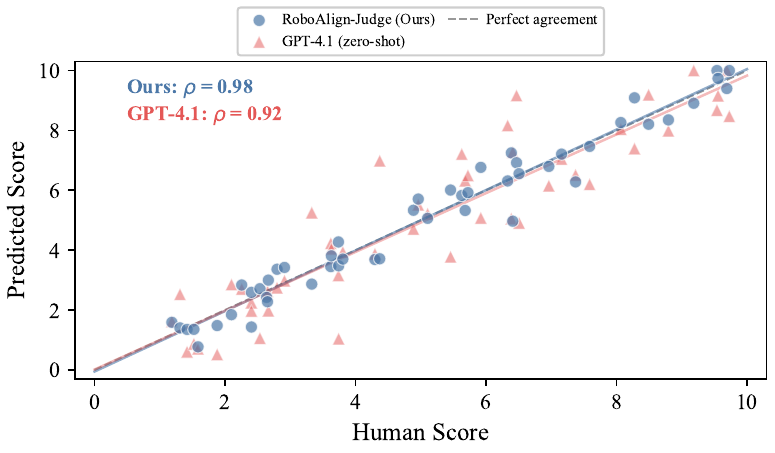}
    \caption{Predicted vs.\ human total scores for \textsc{RoboAlign-Judge} (blue) and GPT-4.1 (orange) on 50 test samples. The dashed line indicates perfect agreement. \textsc{RoboAlign-Judge} shows tighter clustering around the diagonal.}
    \label{fig:judge_vlm_scatter}
\end{minipage}
\hfill
\begin{minipage}[t]{0.48\linewidth}
    \centering
    \includegraphics[width=\linewidth]{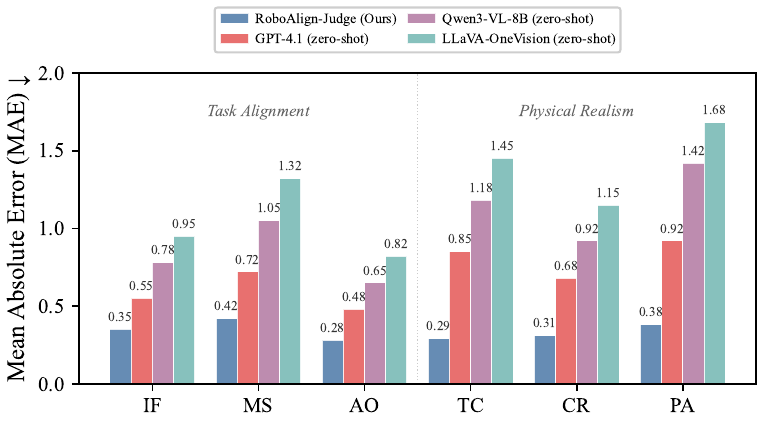}
    \caption{Per-dimension MAE comparison across judge types. Fine-tuning yields the largest improvements on physics-related dimensions (TC, CR, PA), where general VLMs lack domain knowledge.}
    \label{fig:judge_vlm_perdim_mae}
\end{minipage}
\end{figure}

\begin{figure}[h]
\centering
\includegraphics[width=0.85\linewidth]{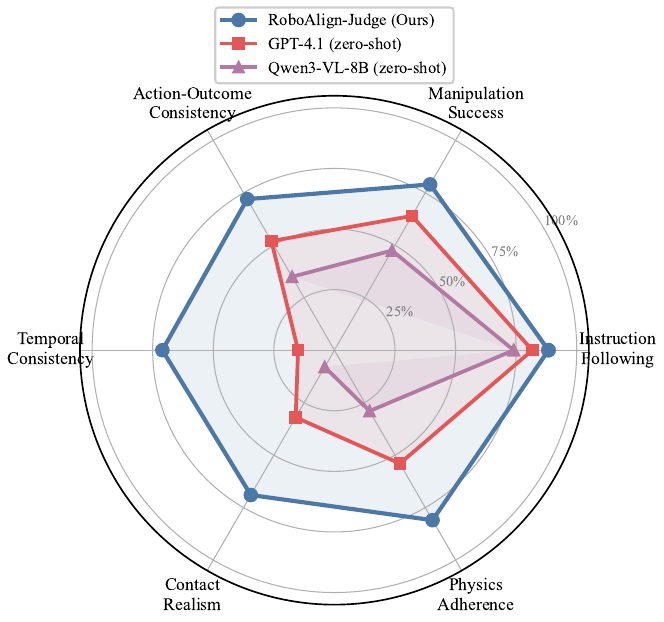}
\caption{Normalized scoring accuracy (1 $-$ MAE / max\_score) across six dimensions for representative judges. \textsc{RoboAlign-Judge} achieves near-uniform high accuracy across all dimensions, while zero-shot models show pronounced weaknesses in physical realism dimensions.}
\label{fig:judge_vlm_radar}
\end{figure}

\paragraph{Evaluation stability.}
Because \textsc{RoboAlign-Judge} uses sampling-based decoding at inference time, its output can vary slightly across runs. To measure this effect directly, we evaluate the same 50-sample test set 5 times with $T{=}0.6$ and top-$p{=}0.95$.
Figure~\ref{fig:judge_stability_box} shows that this variability is small: Pearson $\rho$ stays within 0.86--0.91 (mean 0.89, std 0.02), and Pairwise Accuracy stays within 85.2\%--89.4\% (mean 87.6\%, std 1.5\%).
Importantly, the relative ranking of all judges is unchanged across the 5 runs.
This indicates that stochastic decoding does not materially affect the judge's usefulness as a teacher model, since any residual per-sample noise is averaged out during large-batch distillation and RL training.
Figure~\ref{fig:judge_stability_box} visualizes the stability across runs.

\begin{figure}[h]
\centering
\includegraphics[width=0.9\linewidth]{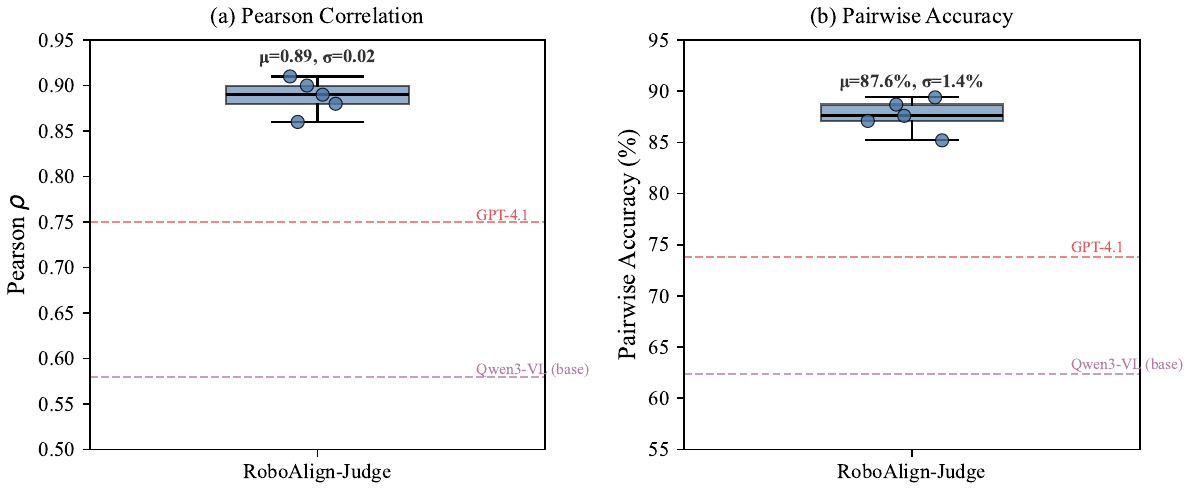}
\caption{Evaluation stability of \textsc{RoboAlign-Judge} over 5 independent runs (sampling decoding, $T{=}0.6$). Box plots show the distribution of Pearson $\rho$ and Pairwise Accuracy. The narrow interquartile ranges confirm reliable scoring despite stochastic decoding.}
\label{fig:judge_stability_box}
\end{figure}

\FloatBarrier
\subsection{Small-scale Blinded Human Evaluation}
\label{app:human_eval_template}

To complement the automatic judge-based evaluation, we conduct a small-scale blinded human study on a held-out subset. We sample instruction--initial-frame pairs that are disjoint from judge training and reward-distillation data and compare videos from \method{} against strong baselines under anonymized ordering. Each comparison is evaluated along four axes: \emph{overall preference}, \emph{task success}, \emph{physical plausibility}, and \emph{temporal coherence}. Each sample is labeled by multiple raters, and tie decisions are allowed when differences are not visually distinguishable.

Table~\ref{tab:human_eval_template} summarizes the results. Across all pairwise comparisons, \method{} is preferred over strong baselines on overall quality and consistently receives higher ratings on task success and physical plausibility. The human ranking is broadly consistent with the automatic \textsc{RoboAlign-Judge} results, providing complementary evidence that the improvements are not unique to a single in-domain evaluator.

\begin{table*}[t]
\centering
\small
\setlength{\tabcolsep}{5pt}
\renewcommand{\arraystretch}{1.10}
\caption{\textbf{Small-scale blinded human evaluation.} Pairwise human preference on a held-out subset comparing \method{} against strong baselines. Higher is better for win rates and mean preference scores.}
\label{tab:human_eval_template}
\fitwidth{%
\begin{tabular}{@{}lcccccc@{}}
\toprule
\rowcolor{tblheadbg}
\textbf{Comparison} & \textbf{Overall win rate $\uparrow$} & \textbf{Task success $\uparrow$} & \textbf{Physical plausibility $\uparrow$} & \textbf{Temporal consistency $\uparrow$} & \textbf{Tie rate} & \textbf{Mean pref. score $\uparrow$} \\
\midrule
\rowcolor{tblstripe}
\method{} vs.\ iVideoGPT & 66.7\% & 69.3\% & 71.3\% & 63.3\% & 10.0\% & 3.96 / 5 \\
\method{} vs.\ Wan2.2-TI2V-5B (LoRA) & 62.0\% & 64.7\% & 66.7\% & 59.3\% & 12.7\% & 3.82 / 5 \\
\rowcolor{tblstripe}
\method{} vs.\ RLVR-World & 59.3\% & 61.3\% & 63.3\% & 57.3\% & 14.0\% & 3.74 / 5 \\
\bottomrule
\end{tabular}}
\end{table*}

\FloatBarrier
\subsection{Comparison with Alternative Judge Backbones}
\label{app:judge_comparison}

Beyond the zero-shot comparison above, we also evaluate the effect of the \emph{backbone choice} for fine-tuning.
We select Qwen3-VL-8B-Thinking over alternative backbones based on preliminary experiments on a 50-sample validation set.
Table~\ref{tab:judge_backbone} summarizes the comparison.

\begin{table}[h]
\centering
\small
\setlength{\tabcolsep}{3pt}
\renewcommand{\arraystretch}{1.10}
\caption{Comparison of candidate judge backbones. Qwen3-VL-8B-Thinking achieves the best balance of scoring accuracy, reasoning quality, and inference efficiency.}
\label{tab:judge_backbone}
\fitwidth{%
\begin{tabular}{@{}lcccc@{}}
\toprule
\rowcolor{tblheadbg}
\textbf{Backbone} & \textbf{Params} & \textbf{CoT Reasoning} & \textbf{Multi-frame} & \textbf{Selected} \\
\midrule
\rowcolor{tblstripe}
GPT-4o (proprietary)        & ---   & \checkmark & \checkmark & \ding{55} (cost, API latency) \\
Gemini 2.5 (proprietary)    & ---   & \checkmark & \checkmark & \ding{55} (cost, reproducibility) \\
\rowcolor{tblstripe}
Qwen2.5-VL-7B               & 7B    & \ding{55}  & \checkmark & \ding{55} (no native CoT) \\
InternVL2.5-8B               & 8B    & \ding{55}  & \checkmark & \ding{55} (no native CoT) \\
\rowcolor{tblstripe}
LLaVA-OneVision-7B           & 7B    & \ding{55}  & \checkmark & \ding{55} (limited video) \\
\rowcolor{mydarkblue!6}
\textbf{Qwen3-VL-8B-Thinking} & \textbf{8B} & \checkmark & \checkmark & \checkmark \\
\bottomrule
\end{tabular}}
\end{table}

Key advantages of Qwen3-VL-8B-Thinking:
(1)~native \texttt{<think>} token support enables structured CoT without prompt engineering;
(2)~dynamic-resolution multi-frame input avoids the need for fixed-size frame preprocessing;
(3)~open-weight availability ensures full reproducibility and enables LoRA fine-tuning on custom data;
(4)~8B scale provides sufficient capacity for nuanced physical reasoning while remaining efficient for batch labeling.

\subsection{Limitations and Future Directions}
\label{app:judge_limitations}

We acknowledge several limitations of the current \textsc{RoboAlign-Judge}:
\begin{itemize}[leftmargin=*,nosep]
    \item \textbf{Dataset coverage.} Although the current training corpus contains 10{,}000 samples spanning four robot datasets, coverage remains uneven across embodiments, environments, and failure modes. Expanding annotation density and generative-model diversity would likely improve generalization further.
    \item \textbf{Sampling variance.} The CoT decoding introduces non-trivial variance (std $\approx$ 0.02 on Pearson $\rho$). Ensemble averaging or deterministic decoding could reduce this, at the cost of inference speed or diversity.
    \item \textbf{Domain transfer.} The current corpus is still dominated by tabletop manipulation and related interaction patterns. Extending to other embodiments (mobile manipulation, dexterous hands) requires additional domain-specific annotation and evaluation.
    \item \textbf{Temporal granularity.} The judge processes 8 uniformly sampled frames, which may miss brief but critical events (e.g., momentary contact). Adaptive frame sampling could improve sensitivity to such events.
\end{itemize}

\FloatBarrier
\subsection{Judge Prompt Templates}
\label{app:judge_prompt}

For completeness, we provide the prompt templates used for the robot manipulation judge. The prompting interface has two parts: a system prompt that specifies the evaluator role and scoring rubric, and a user message template that injects the task instruction, the initial frame, and uniformly sampled generated frames.

\paragraph{Prompt structure.}
\begin{center}
\small
\begin{tabular}{p{0.22\linewidth}p{0.7\linewidth}}
\toprule
Component & Role \\
\midrule
System prompt & Defines the judge identity, the six evaluation dimensions, and the required JSON output schema. \\
User message & Instantiates the instruction, initial frame, and $N$ sampled video frames for a concrete evaluation example. \\
\bottomrule
\end{tabular}
\end{center}

\paragraph{Scoring rubric.}
\begin{center}
\footnotesize
\setlength{\tabcolsep}{4pt}
\renewcommand{\arraystretch}{1.05}
\begin{tabular}{p{0.29\linewidth}cp{0.56\linewidth}}
\toprule
Dimension & Range & Criterion \\
\midrule
Instruction Following & 0--3 & Whether the robot attempts and correctly executes the instructed action. \\
Manipulation Success & 0--2 & Whether the manipulation objective is fully completed by the end of the video. \\
Action--Outcome Consistency & 0--1 & Whether the observed outcomes are logically consistent with the robot's actions. \\
Temporal Consistency & 0--1 & Whether the video remains smooth and free of major temporal artifacts. \\
Contact Realism & 0--1 & Whether object contacts appear physically natural. \\
Physics Adherence & 0--2 & Whether the generated motion obeys basic physical plausibility. \\
\midrule
Total & 0--10 & Sum of the six dimension scores. \\
\bottomrule
\end{tabular}
\end{center}

\begin{contribbox}
\textbf{System Prompt.}

{\ttfamily\footnotesize
You are an expert evaluator for robot manipulation video generation quality.\\
You will be given:\\
1. An instruction describing what the robot should do\\
2. An initial frame (the starting state)\\
3. A sequence of frames from a generated video\\[0.4em]
Your task is to evaluate the generated video across 6 dimensions. Score each dimension carefully.\\[0.4em]
Scoring Dimensions:\\
1. Instruction Following (0-3 points): Does the video show the robot attempting and executing the action described in the instruction?\\
\quad 0: Completely unrelated action\\
\quad 1: Vaguely related but wrong action\\
\quad 2: Correct action but incomplete or imprecise\\
\quad 3: Perfectly follows the instruction\\[0.3em]
2. Manipulation Success (0-2 points): Is the manipulation task successfully completed by the end of the video?\\
\quad 0: Task completely failed\\
\quad 1: Partial success (object moved but not to target)\\
\quad 2: Full success (task completed as instructed)\\[0.3em]
3. Action-Outcome Consistency (0-1 point): Are the robot's actions logically consistent with the observed outcomes?\\
\quad 0: Actions and outcomes are inconsistent\\
\quad 1: Actions and outcomes are consistent\\[0.3em]
4. Temporal Consistency (0-1 point): Is the video temporally coherent without flickering, sudden jumps, or artifacts?\\
\quad 0: Severe temporal artifacts\\
\quad 1: Smooth and temporally consistent\\[0.3em]
5. Contact Realism (0-1 point): When the robot contacts objects, does it look physically realistic?\\
\quad 0: Unrealistic contacts\\
\quad 1: Contacts look natural and realistic\\[0.3em]
6. Physics Adherence (0-2 points): Does the video obey basic physics?\\
\quad 0: Severe physics violations\\
\quad 1: Minor physics issues but mostly plausible\\
\quad 2: Fully physically plausible\\[0.4em]
Output Format:\\
First, think step by step about what you observe in the video. Then output your evaluation as a JSON object:\\
\{\\
\quad "reasoning": "Your detailed analysis",\\
\quad "instruction\_following": <0-3>,\\
\quad "manipulation\_success": <0-2>,\\
\quad "action\_outcome\_consistency": <0-1>,\\
\quad "temporal\_consistency": <0-1>,\\
\quad "contact\_realism": <0-1>,\\
\quad "physics\_adherence": <0-2>,\\
\quad "total": <0-10>\\
\}
}
\end{contribbox}

\begin{problembox}
\textbf{User Message Template.}

{\ttfamily\footnotesize
Instruction: \{instruction\}\\[0.35em]
Initial Frame (starting state):\\
\mbox{[Initial frame image]}\\[0.35em]
Generated Video Frames (\{N\} frames sampled uniformly):\\
\mbox{[Frame 1] [Frame 2] ... [Frame N]}\\[0.35em]
Please evaluate this generated video based on the 6 dimensions described.
}
\end{problembox}

\newpage
        3w3w\section*{NeurIPS Paper Checklist}

\begin{enumerate}

\item {\bf Claims}
    \item[] Question: Do the main claims made in the abstract and introduction accurately reflect the paper's contributions and scope?
    \item[] Answer: \answerYes{}
    \item[] Justification: The abstract and the Introduction state three concrete contributions---(i) a reward-aligned RL post-training framework with a distilled multimodal reward (\S\ref{sec:reward}), (ii) the \textsc{RobotWorldBench} benchmark and teacher--student reward distillation pipeline (\S\ref{sec:reward}, Appendix~\ref{app:roboalign_judge}), and (iii) sliding window re-encoding (SWR) for long-horizon inference (\S\ref{sec:swr}). All headline numbers ($+10.1\%$ aggregate judge score, $+2.8\%$ SSIM / $-9.8\%$ LPIPS from SWR, ${\sim}10\times$ reward-cost reduction) are supported by the experiments (see Experiments section) and the appendices; scope limitations (single-arm tabletop manipulation, no downstream policy validation) are stated explicitly in the Conclusion.

\item {\bf Limitations}
    \item[] Question: Does the paper discuss the limitations of the work performed by the authors?
    \item[] Answer: \answerYes{}
    \item[] Justification: The Conclusion explicitly states that our current evaluation is limited to single-arm tabletop and countertop manipulation, that downstream policy improvement is not yet validated, and that extending to diverse embodiments and control-level gains remains future work. Appendix~\ref{app:judge_limitations} discusses failure modes and limitations of the teacher judge, and Appendix~\ref{app:proof_swr} clearly marks the SWR stability result as a stylized local analysis under explicit assumptions (\textbf{A1}--\textbf{A2}) rather than a global guarantee.

\item {\bf Theory assumptions and proofs}
    \item[] Question: For each theoretical result, does the paper provide the full set of assumptions and a complete (and correct) proof?
    \item[] Answer: \answerYes{}
    \item[] Justification: Our theoretical results are Proposition~\ref{prop:swr} and Corollaries~\ref{cor:wstar},~\ref{cor:swr_vs_ar} (SWR error bound, asymptotic optimum $W^{*}$, and SWR-vs-AR gap). Their assumptions \textbf{A1} (bounded per-step error with local contraction) and \textbf{A2} (bounded decode--re-encode distortion) are stated in Appendix~\ref{app:proof_swr}, and the stylized derivations are given there in full detail with a proof sketch also included in the main text. The limitations of the stylized assumptions are discussed in the ``Discussion of assumptions'' paragraph of Appendix~\ref{app:proof_swr}.

    \item {\bf Experimental result reproducibility}
    \item[] Question: Does the paper fully disclose all the information needed to reproduce the main experimental results of the paper to the extent that it affects the main claims and/or conclusions of the paper (regardless of whether the code and data are provided or not)?
    \item[] Answer: \answerYes{}
    \item[] Justification: The full world-model training pipeline (dual-branch FSQ tokenizer, 12-layer LLaMA-style Transformer, token/action vocabulary, optimization infrastructure, and data splits) is documented in Appendix~\ref{app:vwm_training} with a dedicated Reproduction Protocol (Appendix~\ref{app:vwm_repro}); the reward-aligned RL post-training procedure is formalized in Algorithm~\ref{alg:train} and the student reward model architecture and hyperparameters in Appendix~\ref{app:student_reward}; SWR inference is given as both pseudocode (Algorithm~\ref{alg:swr}) and a PyTorch implementation (Listing~\ref{lst:swr}). The teacher judge training data, rubric, prompts, and LoRA configuration are in Appendix~\ref{app:roboalign_judge}. We also release code at the project repository linked in the abstract.

\item {\bf Open access to data and code}
    \item[] Question: Does the paper provide open access to the data and code, with sufficient instructions to faithfully reproduce the main experimental results, as described in supplemental material?
    \item[] Answer: \answerYes{}
    \item[] Justification: The project repository and project page are linked in the abstract, and will host the training/inference code for the world model, reward distillation, GRPO post-training, and SWR, together with the \textsc{RobotWorldBench} annotations and evaluation scripts. All upstream datasets (RT-1, BridgeData V2, CALVIN, LIBERO) are publicly available from their respective creators and are used under their released terms. Anonymized versions of code and data will be provided with the submission where required.

\item {\bf Experimental setting/details}
    \item[] Question: Does the paper specify all the training and test details (e.g., data splits, hyperparameters, how they were chosen, type of optimizer) necessary to understand the results?
    \item[] Answer: \answerYes{}
    \item[] Justification: The Experiments section summarizes the main experimental protocol, including datasets (RT-1, BridgeData V2), baselines, and evaluation metrics. Full training details for the world model (optimizer, learning-rate schedule, batch size, number of iterations, hardware) are provided in Appendix~\ref{app:vwm_training}; the student reward model's optimizer, differential learning rates, mixup, EMA, and early-stopping criterion are in Appendix~\ref{app:student_reward}; the teacher judge's LoRA hyperparameters and training data composition are in Appendix~\ref{app:roboalign_judge}; SWR evaluation uses the protocol in Appendix~\ref{app:swr_window_ablate} (12-layer LLaMA, hidden 768, 12 heads, vLLM inference, 300 episodes $\times$ 30 frames).

\item {\bf Experiment statistical significance}
    \item[] Question: Does the paper report error bars suitably and correctly defined or other appropriate information about the statistical significance of the experiments?
    \item[] Answer: \answerYes{}
    \item[] Justification: Tables~\ref{tab:robotworldbench_main}, \ref{tab:robotworldbench_lowlevel}, and~\ref{tab:swr_quality_efficiency} report mean $\pm$ standard deviation across repeated evaluations over the full generated rollouts. The reported error bars are $1\sigma$ standard deviations computed from the per-episode scores, as stated in the corresponding table captions. For the human study (Appendix~\ref{app:human_eval_template}) we report agreement statistics on the held-out subset.

\item {\bf Experiments compute resources}
    \item[] Question: For each experiment, does the paper provide sufficient information on the computer resources (type of compute workers, memory, time of execution) needed to reproduce the experiments?
    \item[] Answer: \answerYes{}
    \item[] Justification: All experiments use NVIDIA A100 GPUs; per-stage configurations (e.g., $8\times$ A100-40GB for world-model training, a single A100 for SWR inference measurements and for student/teacher latency profiling) are documented in Appendix~\ref{app:vwm_infra} and in the captions of Tables~\ref{tab:student_efficiency}, \ref{tab:student_rl_cost}, and~\ref{tab:swr_efficiency_full}. Per-video inference latency ($\sim$2800\,ms teacher vs.\ $\sim$20\,ms student), RL reward-cost comparisons, and SWR wall-clock / peak-memory numbers are reported in the corresponding tables. We additionally disclose that exploratory runs (tokenizer and RL ablations that did not make it into the paper) consumed additional compute beyond the reported runs.
    
\item {\bf Code of ethics}
    \item[] Question: Does the research conducted in the paper conform, in every respect, with the NeurIPS Code of Ethics \url{https://neurips.cc/public/EthicsGuidelines}?
    \item[] Answer: \answerYes{}
    \item[] Justification: The research uses publicly available robot manipulation datasets under their original licenses, does not involve personally identifiable information, and relies only on simulated and tabletop manipulation videos. The human evaluation (Appendix~\ref{app:human_eval_template}) was conducted on a small internal panel reviewing de-identified video clips of tabletop manipulation with fairly compensated volunteer evaluators and no sensitive attributes. We have reviewed the NeurIPS Code of Ethics and confirm that our work conforms to it.

\item {\bf Broader impacts}
    \item[] Question: Does the paper discuss both potential positive societal impacts and negative societal impacts of the work performed?
    \item[] Answer: \answerYes{}
    \item[] Justification: Positive impacts include more sample-efficient and safer robot learning---better world models can act as learnable simulators for planning and policy evaluation, reducing the need for real-robot trials and improving physical plausibility of predicted rollouts. Negative impacts are limited but nontrivial: improved robot video generation could be misused to synthesize misleading robot demonstrations, and reward-model distillation could inherit biases from the teacher VLM. We do not anticipate immediate deployment risks because the models are restricted to tabletop manipulation scenarios, and we recommend gated release and clear provenance metadata when distributing trained generators, in line with common practice for generative models.
    
\item {\bf Safeguards}
    \item[] Question: Does the paper describe safeguards that have been put in place for responsible release of data or models that have a high risk for misuse (e.g., pre-trained language models, image generators, or scraped datasets)?
    \item[] Answer: \answerNA{}
    \item[] Justification: Our released artefacts consist of a benchmark of annotated tabletop manipulation videos (sourced from existing, already-released robot datasets) and a world model and reward model narrowly trained on single-arm tabletop manipulation. The model cannot generate human faces, persons, or open-domain imagery, so we do not consider it to pose a high misuse risk. Should the project be extended to broader embodiments or open-domain video generation, we will revisit the need for gated release and additional safety filters.

\item {\bf Licenses for existing assets}
    \item[] Question: Are the creators or original owners of assets (e.g., code, data, models), used in the paper, properly credited and are the license and terms of use explicitly mentioned and properly respected?
    \item[] Answer: \answerYes{}
    \item[] Justification: All upstream datasets and models are cited at first use in the Experiments section and in the appendices: RT-1~\cite{brohan2023rt1}, BridgeData V2~\cite{walke2023bridgedata}, CALVIN~\cite{mees2022calvin}, LIBERO~\cite{liu2024libero}, Qwen3-VL-Thinking~\cite{bai2025qwen25vl}, LoRA~\cite{hu2022lora}, and the baseline world models and video generators compared in Table~\ref{tab:robotworldbench_main}. These assets are used in accordance with their respective licenses (mostly Apache-2.0 / MIT / research-use licenses as posted by their maintainers). The aggregated annotation files released with \textsc{RobotWorldBench} will be distributed under a CC-BY-4.0 license together with per-source license notices.

\item {\bf New assets}
    \item[] Question: Are new assets introduced in the paper well documented and is the documentation provided alongside the assets?
    \item[] Answer: \answerYes{}
    \item[] Justification: We release three new assets: (i) \textsc{RobotWorldBench}, a 10{,}000-pair annotated video--instruction benchmark, documented in \S\ref{sec:reward} and Appendix~\ref{app:judge_data} (data sources, degradation pipeline, rubric, annotation protocol); (ii) \textsc{RoboAlign-Judge}, a LoRA-fine-tuned Qwen3-VL teacher judge, documented in Appendix~\ref{app:roboalign_judge} together with its prompt templates (Appendix~\ref{app:judge_prompt}); and (iii) the distilled student reward model (\S\ref{sec:reward}, Appendix~\ref{app:student_reward}). Each asset will be accompanied by a model/dataset card describing training data, intended use, limitations, and license.

\item {\bf Crowdsourcing and research with human subjects}
    \item[] Question: For crowdsourcing experiments and research with human subjects, does the paper include the full text of instructions given to participants and screenshots, if applicable, as well as details about compensation (if any)? 
    \item[] Answer: \answerYes{}
    \item[] Justification: The small-scale blinded human evaluation (Appendix~\ref{app:human_eval_template}) uses internal volunteer raters, not crowd workers; Appendix~\ref{app:human_eval_template} provides the full instructions given to raters, the rubric and interface template, and the number of pairwise comparisons used. Raters were fairly compensated at or above the local minimum wage and were not exposed to personally identifiable or sensitive content. No identifying information about raters is collected or released.

\item {\bf Institutional review board (IRB) approvals or equivalent for research with human subjects}
    \item[] Question: Does the paper describe potential risks incurred by study participants, whether such risks were disclosed to the subjects, and whether Institutional Review Board (IRB) approvals (or an equivalent approval/review based on the requirements of your country or institution) were obtained?
    \item[] Answer: \answerNA{}
    \item[] Justification: Our human evaluation consists only of blinded rating of short tabletop manipulation video clips with no personally identifiable or sensitive content; under the relevant institutional guidelines such minimal-risk evaluation does not require formal IRB review. The task, expected effort, and voluntary nature of participation were disclosed to all raters before they consented to participate, as described in Appendix~\ref{app:human_eval_template}.

\item {\bf Declaration of LLM usage}
    \item[] Question: Does the paper describe the usage of LLMs if it is an important, original, or non-standard component of the core methods in this research? Note that if the LLM is used only for writing, editing, or formatting purposes and does \emph{not} impact the core methodology, scientific rigor, or originality of the research, declaration is not required.
    \item[] Answer: \answerYes{}
    \item[] Justification: A multimodal LLM (Qwen3-VL-8B-Thinking) is used as the backbone of the \textsc{RoboAlign-Judge} teacher model after LoRA fine-tuning on \textsc{RobotWorldBench} (\S\ref{sec:reward}, Appendix~\ref{app:roboalign_judge}). Its role, training data, and how its scores are distilled into the student reward model used for RL post-training are described in detail. Any LLM use for writing or editing is consistent with NeurIPS policy and is not part of the core methodology.

\end{enumerate}

\end{document}